\documentclass[12pt]{article}

\usepackage[utf8]{inputenc}
\usepackage[english]{babel}
\usepackage[OT1]{fontenc}
\usepackage{amsmath,amsthm,mathtools,amsfonts,amssymb,bm,amsbsy}
\usepackage{graphicx,psfrag,epsf}
\usepackage{booktabs,longtable,array,tabularx,multirow,adjustbox,ragged2e}
\usepackage{subcaption}
\usepackage{enumerate,enumitem}
\usepackage{natbib}
\usepackage{xcolor}
\usepackage{soul}
\usepackage{url}
\usepackage{etoolbox}
\usepackage{tipa}
\usepackage{float}
\usepackage{comment}
\usepackage{pifont}
\usepackage{placeins}
\usepackage[ruled,vlined,linesnumbered,resetcount,algosection]{algorithm2e}
\usepackage{algpseudocode}
\usepackage{titlesec}
\usepackage{parskip}
\usepackage{setspace}
\usepackage{hyperref}
\hypersetup{draft=false,colorlinks=true,linkcolor=blue,citecolor=blue,urlcolor=blue}
\usepackage{bookmark}

\newcommand\BibTeX{{\rmfamily B\kern-.05em \textsc{i\kern-.025em b}\kern-.08em T\kern-.1667em\lower.7ex\hbox{E}\kern-.125emX}}

\newcommand{\cmark}{\checkmark}
\newcommand{\xmark}{\ding{55}}

\newtheorem{theorem}{Theorem}[section]

\newtheorem{lemma}[theorem]{Lemma}

\titleformat{\paragraph}{\normalfont\normalsize\bfseries}{\theparagraph}{1em}{}
\titlespacing*{\paragraph}{0pt}{3.25ex plus 1ex minus .2ex}{1.5ex plus .2ex}
\DeclareMathAlphabet\mathbfcal{OMS}{cmsy}{b}{n}
\numberwithin{equation}{section}
\renewcommand{\thefigure}{\arabic{figure}}
\renewcommand{\thetable}{\arabic{table}}
\allowdisplaybreaks

\begin{document}

\title{Sparse Functional Singular Value Decomposition for Biclustering and Triclustering Longitudinal Data}

\author{Yue Zhao, Thierry Chekouo$^{\ast}$, and Sandra E. Safo$^{\ast}$\\
\textit{Division of Biostatistics and Health Data Science}\\
\textit{University of Minnesota}\\
\textit{2221 University Avenue SE}\\
\textit{Minneapolis, MN 55414, United States}\\
E-mail address for correspondence: tchekouo@umn.edu, ssafo@umn.edu}
\date{}

\maketitle

\begin{abstract}
Identifying subtypes of complex conditions, such as Inflammatory Bowel Disease (IBD), often requires capturing latent patterns in longitudinal omics data. However, these data are typically high-dimensional, sparsely sampled, and irregularly observed over time, posing substantial challenges for conventional (bi)clustering and functional data analysis methods. We propose Tri-SfSVD, a unified sparse functional Singular Value Decomposition framework for discovering biclusters and triclusters in longitudinal data. Unlike existing functional biclustering methods that rely on ad hoc imputation or enforce restrictive shape-homogeneity assumptions, Tri-SfSVD integrates continuous trajectory estimation with simultaneous subject, feature, and temporal selection within a single optimization framework. By imposing sparse penalties across subjects, variables, and temporal subregions, the proposed method works directly on observed data to uncover localized structures at the subject, subject-feature, and subject-feature-time levels. Extensive simulations demonstrate that Tri-SfSVD outperforms existing approaches in high-dimensional settings. Applied to IBD multi-omics data, the method identified three biclusters linking sample clusters with distinct IBD-related clinical characteristics to microbial pathway groups associated with specific bacterial taxa, providing interpretable subject-pathway associations for characterizing disease heterogeneity. Applied to multi-channel EEG data, the method identified three triclusters linking sample clusters with distinct alcohol-related phenotypes to localized brain activity patterns, including subgroup differences separated by temporal subregions within the same spatial region.
\end{abstract}


\section{Introduction}
In modern biomedical studies, high-throughput technologies frequently generate longitudinal measurements across a large number of features. In practice, these data are often sparsely and irregularly observed over time, where each subject is observed at only a limited, highly variable set of time points rather than a discrete, uniform grid.  A prominent example is the Integrative Human Microbiome Project (iHMP), which established the Inflammatory Bowel Disease Multi-omics Database (IBDMDB). This is a large-scale resource for longitudinal profiling of the gut microbiome in individuals with inflammatory bowel disease (IBD), including Crohn's disease (CD) and ulcerative colitis (UC), as well as non-IBD controls \citep{LloydPrice2019}.

Many efforts have been made to better understand the biological mechanisms underlying IBD \citep{chen2020coabundance,yu2023_32gene_ibd}. However, these studies typically treat the data cross-sectionally by analysing a single time point, thereby failing to leverage the temporal structure of longitudinal measurements. To overcome this limitation, functional data analysis (FDA) has emerged as a powerful framework, treating longitudinal observations as functions over time to capture continuous-time trajectories. 
While useful for supervised group discrimination, IBD is a highly heterogeneous disease. Existing supervised approaches cannot capture latent, localized sub-structures where a specific patient subgroup exhibits a unique co-expression pattern across only a subset of features.

To discover such localized patterns, biclustering methods have been extensively developed for matrix-valued data to simultaneously cluster samples and features \citep{biclustersandra,biclusterthierry,biclusterssvd,Sill2011}. Recently, this concept has been extended to continuous-time profiles via functional biclustering \citep{funcc, modelbasedfunctionalbi}. However, existing functional biclustering methods face four distinct methodological challenges in high-dimensional biomedical applications: i) \textit{Rectrictive Shape Homogeneity}. Most methods assume trajectories within a bicluster share a similar geometric shape over time. This assumption is overly restrictive for heterogeneous diseases like IBD, where different biological indicators may collectively characterize a patient subgroup but exhibit entirely different trajectory shapes. 
ii) \textit{Sensitivity to Preprocessing:} 
When functional data are extremely sparse, existing functional biclustering methods rely on a two-step procedure; first imputing missing values or smoothing trajectories (e.g., via principal analysis by conditional expectation [PACE]; \cite{PACE}) and then applying biclustering.  
This separate estimation makes the resulting biclusters highly sensitive to the choice of preprocessing.
iii) \textit{High-dimensionality setting.} Current frameworks offer limited discussion or empirical evaluation when the number of functional variables outnumbers subjects. This high-dimensional regime is increasingly common in modern biomedical studies, including the IBDMDB data.
iv) \textit{Lack of Temporal Subregion Selection.} Biological signals are often temporally localized, concentrating within narrow windows of the time domain. Because current functional biclustering approaches select variables globally across all time points, they cannot isolate these brief, localized effects, necessitating a unified framework that simultaneously handles subject, feature, and temporal selection (i.e., triclustering).

To address these challenges, we propose a sparse functional singular value decomposition (SVD) framework and develop a novel functional triclustering method, Tri-SfSVD. Functional SVD provides a natural extension of matrix factorization to functional data. For univariate functional data, \cite{huang2008functional} proposed a functional SVD framework where the left singular vector contains subjects' principal component (PC) scores and the right singular function represents a smooth functional loading over time. 
\cite{zhao2024} extended this framework to multivariate functional data by introducing a multivariate right singular function across multiple time domains. Building upon this formulation, Tri-SfSVD introduces three sparsity-inducing penalties to capture structure at multiple levels: a sample-level sparsity penalty on the left singular vector, a variable-level sparsity penalty on the multivariate right singular function, and a within-variable subregion-level sparsity penalty. Together, these penalties enable simultaneous selection of subjects, variables, and informative time-localized subregions, thereby allowing recovery of subject-variable-time structure. Each sparse rank-one approximation defines a sparse layer corresponding to a tricluster. 

Tri-SfSVD offers several key advantages over existing frameworks. First, unlike traditional biclustering methods, Tri-SfSVD does not force trajectories within a cluster to share a common shape between features. Instead, it extracts sparse latent layers that capture dominant, coordinated variation, utilizing data-driven sparsity to identify the participating subjects, variables, and temporal subregions. Second, it jointly estimates trajectories and identifies clusters, operating directly on sparse, irregularly observed data using penalized least squares without ad hoc imputation. Third, the proposed framework applies to high-dimensional functional data and is designed to maintain computational feasibility as the number of functional variables increases. 
Finally, the proposed framework allows the data to determine whether biclustering is sufficient or whether additional time-localized triclustering structure is supported.

We demonstrate the utility of Tri-SfSVD through extensive simulation studies and two distinct real-world applications. In the motivating application to the IBDMDB data, our method partitions microbial features into three interpretable groups. The three feature groups exhibit clear taxonomic patterns: two are predominantly associated with \textit{Alistipes putredinis} and \textit{Ruminococcus torques}, respectively, while the third is primarily composed of pathways labeled as unclassified in the dataset. Each microbial feature group is mapped to a distinct clinically meaningful subject subgroup, yielding interpretable subject–feature associations. In this application, full-domain biclusters are supported without any within-variable subregion shrinkage to zero. 
To illustrate the method's unique ability to recover time-localized triclustering structure, we present an additional analysis of a multi-channel electroencephalography (EEG) dataset \citep{Zhang1997_EEG, eeg_database_121} in section~\ref{real_data_eeg} of the Supplementary Material. Conversely, the EEG analysis reveals clear triclustering structures, partitioning channels into three interpretable groups: a frontal scalp pattern spanning the entire time domain, and two posterior scalp patterns localized within distinct temporal windows. These isolated spatio-temporal signatures are strongly associated with specific alcohol-related phenotypes, demonstrating the method's ability to uncover multi-level, localized structures.

The remainder of the paper is organized as follows. Section~\ref{Methodology} introduces the theoretical foundations and presents the proposed Tri-SfSVD framework. Section~\ref{Simulation Study} evaluates the performance of Tri-SfSVD through simulation studies and compares it with existing approaches. Section~\ref{Real data analysis} illustrates the practical utility of the proposed method through real data applications. A concise summary of the methods considered in these analyses is provided in Table~\ref{tab:task_yesno} of Section~\ref{Methods Comparison} in the Supplementary Material. Section~\ref{Conclusion} concludes the paper with a discussion and future directions.

\section{Methodology}
\label{Methodology}

\subsection{Theoretical Foundation and Rank-One Approximation}
We consider multivariate functional data where each observation consists of \( p \geq 2 \) functional variables \( f_{1}, \dots, f_{p} \), possibly defined on different domains \( \mathcal{T}_1, \dots, \mathcal{T}_p \). For each \( j = 1, \dots, p \), the domain \( \mathcal{T}_j \) is a compact subset of \( \mathbb{R} \) with finite Lebesgue measure. We assume each functional variable \( f_{j}: \mathcal{T}_j \to \mathbb{R} \) belongs to the  Hilbert space \( H_j := L^2(\mathcal{T}_j) \) of square-integrable functions on \( \mathcal{T}_j \). 

Following the framework of  \cite{ReMFPCA} and \cite{happ2018multivariate}, we define the Cartesian product Hilbert space \( \mathbb{H} = H_1 \times \cdots \times H_p \). We equip the Hilbert space \( \mathbb{H}\) with the inner product
\[
\langle \pmb{f},\pmb{g}\rangle_{\mathbb{H}}
= \sum_{j=1}^p \langle f_j,g_j\rangle_{H_j} = \sum_{j=1}^p \int_{\mathcal{T}_j} f_j(t_j) g_j(t_j) \, \text{d}t_j.
\] where \( \pmb{f} = (f_1, \cdots, f_p) \) and \( \pmb{g} = (g_1, \cdots, g_p) \in \mathbb{H} \) with \( f_j, g_j \in H_j \) are two multivariate functions.  

Given $n$ multivariate functional samples $\pmb{x}_i = (x_{i,1},\cdots,x_{i,p})\in\mathbb{H}$, $i=1,\ldots,n$, we define the centralized bounded linear data operator $\mathbfcal{X}: \mathbb{R}^n\to\mathbb{H} $  as $\mathbfcal{X}a=\sum_{i=1}^n a_i\,\pmb{x}_i$ for $\pmb{a}=(a_1,\ldots,a_n)^\top\in\mathbb{R}^n$. The corresponding adjoint operator $\mathbfcal{X}^*:\mathbb{H} \rightarrow \mathbb{R}^n$ is given by $\mathbfcal{X}^*{\pmb z}=
	(\langle {\pmb x}_1,{\pmb z}\rangle_{\mathbb{H}}, \langle {\pmb x}_2,{\pmb z}\rangle_{\mathbb{H}}, \ldots, \langle {\pmb x}_n,{\pmb z}\rangle_{\mathbb{H}})^\top,
	\ {\pmb z}\in\mathbb{H}.$   We denote 
$\mathbfcal{X}:=[\pmb{x}_i]_{i=1}^n\in\mathbb{F}^{p\times n}$, where $\mathbb{F}^{p\times n}$ denotes the space of all linear operators
$\mathbfcal{Z}:\mathbb{R}^n\to\mathbb{H}$. We equip $\mathbb{F}^{p\times n}$ with the inner product $ \langle \boldsymbol{\mathcal{Z}}_1, \boldsymbol{\mathcal{Z}}_2 \rangle_{\mathbb{F}} := \sum_{i=1}^{n} \langle \boldsymbol{z}_i^{(1)}, \boldsymbol{z}_i^{(2)} \rangle_{\mathbb{H}}$ where    $\boldsymbol{\mathcal{Z}}_1 = [\boldsymbol{z}_i^{(1)}]_{i=1}^n$ and $\boldsymbol{\mathcal{Z}}_2 = [\boldsymbol{z}_i^{(2)}]_{i=1}^n$ are in $\mathbb{F}^{p \times n}$. The associated norm is defined as   $\Vert\bm{\mathcal{Z}}\Vert_{\mathbb{F}} = \sqrt{\langle \bm{\mathcal{Z}}, \bm{\mathcal{Z}} \rangle_{\mathbb{F}}}.$


Under this operator formulation, we adopt the multivariate functional singular value decomposition (SVD) framework through the rank-one approximation problem:
\begin{equation}\label{multivariate functional SVD}
\min_{s\ge 0,\,
      \pmb{u}:\Vert\pmb{u}\Vert_2=1,\,
      \pmb{\phi}:\Vert\pmb{\phi}\Vert_{\mathbb{H}}=1}
\left\Vert
\mathbfcal{X} - s\,(\pmb{u}\otimes \pmb{\phi})
\right\Vert_\mathbb{F}^2 .
\end{equation}
where \(s\ge0\) is the scale of the rank-one component,
\(\pmb{u}\in\mathbb{R}^n\) is the unit left singular vector, and
\(\pmb{\phi}=(\phi_1,\ldots,\phi_p)\in\mathbb{H}\) is the unit multivariate
functional right singular function. The outer product \(\pmb{u}\otimes\pmb{\phi}\) is a rank-one linear operator
from \(\mathbb{R}^n\) to \(\mathbb{H}\) defined by $ (\pmb{u}\otimes \pmb{\phi})(\pmb{a})
=
\langle \pmb{a},\pmb{u}\rangle_{\mathbb{R}^n}\,\pmb{\phi}, \text{for every }
\pmb a\in\mathbb R^n.$
The key conceptual shift is that while each variable \(j\) has its own
temporal shape \(\phi_j(t)\), all \(p\) variables are linked by a common
subject direction \(\pmb u\). This shared direction forces the model to identify subject clusters that are consistent across all observed features.

While \eqref{multivariate functional SVD} is defined in infinite-dimensional space, trajectories in practice are observed on discrete grids. Let $\bm{Y}_j=\big(x_{ij}(t_{j\ell})\big)_{1\le i\le n,\,1\le \ell\le d_j} \in \mathbb{R}^{n \times d_j}$ denote the discretized observation matrix for the $j^{th}$ variable, where the $i^{th}$ row contains measurements for subject $i$ at time points $\{t_{j1}, \dots, t_{jd_j}\}$. By expanding the squared norm in \eqref{multivariate functional SVD} and incorporating roughness penalties to ensure the smoothness of each $\phi_j$, we derive the multivariate functional SVD objective function for discretized data:


\begin{lemma}
\label{lem:impl_bi_explicit}
The discretized implementation of the scaled rank-one operator approximation
problem \eqref{multivariate functional SVD} is given by the following
finite-dimensional optimization problem:
\begin{equation}
\label{eq:bi_impl_explicit}
\begin{aligned}
\min_{s,\,\pmb{u},\,\{\pmb{\varphi}_j\}_{j=1}^{p}} \quad
& \sum_{j=1}^{p}
\underbrace{\left\lVert \bm{Y}_j - s\,\pmb{u} \pmb{\varphi}_j^\top \right\rVert_F^2}_{\text{Reconstruction Error}}
+
\underbrace{\sum_{j=1}^p \alpha_j
(s\pmb{\varphi}_j)^\top \bm{\Omega}_j (s\pmb{\varphi}_j)}_{\text{Roughness Penalty}} \\
\text{subject to} \quad
& s\ge 0,\qquad
\left\lVert \pmb{u} \right\rVert_2 = 1,\qquad
\left\lVert \pmb{\varphi} \right\rVert_2 = 1,
\text{where }
\pmb{\varphi} =
\left(\pmb{\varphi}_1^\top,\ldots,\pmb{\varphi}_p^\top\right)^\top .
\end{aligned}
\end{equation}

Here $\pmb{\varphi}_j\in\mathbb{R}^{d_j}$ denotes the discretized vector for the
$j$th functional variable, obtained by evaluating the right singular functions $\phi_j(\cdot)$ on
$\{t_{j1},\ldots,t_{j d_j}\}$. The matrix $\bm{\Omega}_j$ is the corresponding
roughness penalty matrix for the $j$th functional variable. The norm $\Vert\cdot\Vert_F$ denotes
the Frobenius norm. The proof is provided in Section~\ref{supp_proofs} of the Supplementary Material.
\end{lemma}


As similarly introduced in \citep{huang2008functional} for the univariate functional SVD, the matrix $\pmb{\Omega}_j$ is a non-negative definite roughness penalty matrix introduced such that $\pmb{\varphi}_j^{\top}\pmb{\Omega}_j\,\pmb{\varphi}_j$ approximates the squared $L^2$-norm of the second derivative $\int \{\phi_j''(t)\}^2\,dt$, thereby enforcing smoothness of the loading function. Here \(\alpha_j \ge 0\) controls the strength of the roughness penalty. In \citep{huang2008functional}, the roughness penalty is multiplied by the squared norm of the non-normalized left singular vector, \( \pmb{u}^{\top}\pmb{u} \), to ensure scale invariance. In our formulation, the non-normalized left vector is written as \(s\pmb{u}\), with \(s\ge0\) and \(\Vert\pmb{u}\Vert_2=1\). Thus, the roughness penalty is equivalently applied to the scaled loading vector \(s\pmb{\varphi}_j\): $(s\pmb{\varphi}_j)^\top\pmb{\Omega}_j(s\pmb{\varphi}_j)
=
s^2\pmb{\varphi}_j^\top\pmb{\Omega}_j\pmb{\varphi}_j.$

In formulation \ref{eq:bi_impl_explicit}, the first term aggregates the univariate reconstruction errors across all $p$ variables, and the second term imposes individual smoothness constraints on each scaled loading function \(s\pmb{\varphi}_j\). This multivariate objective serves as the foundation for the Tri-SfSVD framework, enabling us to subsequently introduce sparsity penalties for simultaneous subject, feature, and temporal selection.

\subsection{Biclustering via Sparse Functional SVD} \label{Section bicluster fsvd}
\subsubsection{Integrated Selection of Subjects, Variables, and Temporal Subregions}
To move beyond global pattern recognition and identify localized structures within high-dimensional longitudinal data, we propose a unified framework that incorporates three layers of sparsity. By augmenting the multivariate functional SVD objective \ref{eq:bi_impl_explicit} with specific penalties, we formulate the Tri-SfSVD optimization problem:

\begin{equation}\label{bicluster fsvd general} 
\begin{aligned}
\min_{s,\,\pmb{u},\,\{\pmb{\varphi}_j\}_{j=1}^{p}} \quad
& \sum_{j=1}^{p} 
\left\lVert \pmb{Y}_j - s\,\pmb{u}\,\pmb{\varphi}_j^{\top} \right\rVert_{F}^{2}
+ \mathcal{P}_{{\gamma}}(s\pmb{u})
+ \mathcal{P}_{{\theta}}(s\pmb{\varphi})
+ \mathcal{P}_{\pmb{\lambda}}(s\pmb{\varphi}) 
+ \sum_{j=1}^{p}\alpha_j
(s\pmb{\varphi}_j)^{\top}\pmb{\Omega}_j(s\pmb{\varphi}_j)  \\
\text{subject to} \quad
& s\ge 0,\qquad
\left\lVert \pmb{u}\right\rVert_2=1,\qquad
\left\lVert \pmb{\varphi}\right\rVert_2=1,
\quad
\pmb{\varphi}=(\pmb{\varphi}_1^\top,\ldots,\pmb{\varphi}_p^\top)^\top .
\end{aligned}
\end{equation}

Here $\mathcal{P}_{{\gamma}}(s\pmb{u})$, $\mathcal{P}_{{\theta}}(s\pmb{\varphi})$, and $\mathcal{P}_{\pmb{\lambda}}(s\pmb{\varphi})$ are sparsity-inducing penalty terms, and $\gamma$, $\theta$, and $\pmb\lambda := \{\lambda_j\}_{j=1}^{p}$ are tuning parameters that control the degree of sparsity at the subject, variable, and within-variable subregion levels, respectively. We specifically define the following penalty structures:~\\
\textit{1. Subject Selection for Subgroup Discovery: $\mathcal{P}_{{\gamma}}(s\pmb{u})$}. To identify clusters of subjects that exhibit a specific latent trajectory, we apply an adaptive lasso penalty to the subject score vector: $\mathcal{P}_{\boldsymbol{\gamma}}(s\pmb{u}) = \gamma \sum_{i=1}^{n} w_{1,i} \lvert s u_i \rvert$. This encourages sparsity at the subject level, effectively grouping individuals with $u_i = 0$ as "inactive" for a particular component, thereby isolating a specific disease subtype or phenotype. 
Here \(w_{1,i}\) are data-driven weights defined by \(w_{1,i}=|\hat{s}\hat{u}_i|^{-\kappa}\), where \(\hat{s}\hat{u}_i\) denotes the unpenalized least-squares estimate of \(s u_i\), obtained by minimizing the reconstruction loss in \eqref{bicluster fsvd general} with \(\{\pmb{\varphi}_j\}_{j=1}^{p}\) fixed. The exponent $\kappa$ is a fixed nonnegative parameter, with typical values chosen from $\kappa \in \{0.5, 1, 2\}$ as suggested by \cite{Zou} and \cite{biclusterssvd}. 
These adaptive weights allow different coefficients to be penalized unequally, so that coefficients with larger initial estimates receive less shrinkage, whereas those with smaller initial estimates receive stronger shrinkage. In the special case where $w_{1,i} = 1$ for all $i$, the penalty reduces to the standard lasso.~\\
\textit{2. Variable Selection for Feature Relevance:} $\mathcal{P}_{\boldsymbol{\theta}}(s\pmb{\varphi})$. In multivariate settings, only a subset of variables (e.g., specific metabolites) may contribute to a latent pattern. We employ a group-lasso-style penalty on the multivariate loading functions: $\mathcal{P}_{\boldsymbol{\theta}}(s\pmb{\varphi})= \theta \sum_{j=1}^{p} w_{2,j}\,\lVert s \pmb{\varphi}_j \rVert_{2}$. This penalty shrinks the entire loading vector $\bm{\varphi}_j$ to zero for irrelevant variables, ensuring that the identified biclusters are defined only by a parsimonious set of features.
The $w_{2,j}$'s are data-driven weights defined as $w_{2,j} = \Vert \hat{s} \hat{\pmb{\varphi}}_j \Vert_2^{-\kappa}$ with the same $\kappa$, and $\hat{s} \hat{\pmb{\varphi}}_j$ denotes the unpenalized least squares estimate of $s \pmb{\varphi}_j$, defined as the minimizer of the reconstruction loss in \eqref{bicluster fsvd general} with $\pmb{u}$ fixed.~\\
\textit{3. Temporal Selection for Triclustering:} $\mathcal{P}_{\bm{\lambda}}(s\pmb{\varphi})$. A key innovation of Tri-SfSVD is the ability to identify ``triclusters''- patterns that exist only within specific temporal subregions. We introduce an element-wise adaptive lasso penalty on the loading functions: $\mathcal{P}_{\bm{\lambda}}(s\pmb{\varphi}) = \sum_{j=1}^{p} \lambda_j \Vert \pmb{w}_{3,j} \odot s \pmb{\varphi}_j \Vert_1,$ where \(\pmb{w}_{3,j} = (w_{3,j1}, \ldots, w_{3,jd_j})^\top\) is a vector of data-driven adaptive weights, and \(\odot\) denotes the Schur (element-wise) product. The weights are defined by $w_{3,j\ell} := |\hat{s} \hat{\varphi}_{j\ell}|^{-\kappa}$. By inducing sparsity within the loading functions $\bm{\varphi}_j$, the model can ``zero out'' segments of the timeline. This allows the framework to distinguish between: i) Biclusters, where a subject-variable association persists across the entire observation period; and ii) Triclusters, where a subject-variable association is localized to a specific time window (e.g., an early-stage vs. late-stage disease signature).

Specifically, the estimated \(s\), sparse \(\pmb{u}\), and \(\{\pmb{\varphi}_j\}_{j=1}^p\) together form a sparse rank-1 approximation to the original data. In this representation, samples corresponding to nonzero entries of \(\pmb{u}\), variables corresponding to nonzero vectors \(\pmb{\varphi}_j\), and informative subregions within the selected variables corresponding to nonzero elements of \(\pmb{\varphi}_j\) are jointly selected. Hence, each layer simultaneously links a subset of samples, a subset of functional variables, and localized subregions within those variables, thereby revealing a tricluster structure in the data. When \(\lambda_j = 0\) for all \(j=1,\ldots,p\), \eqref{bicluster fsvd general} reduces to a biclustering formulation that simultaneously identifies sample and variable clusters. Thus, the biclustering method is nested within the proposed Tri-SfSVD framework as a special case.

\subsubsection{Extension to Sparse and Irregular Sampling}

The formulation in \eqref{bicluster fsvd general} assumes a regularly sampled design, where each subject is observed on a common grid of time points for every functional variable $j$. In many longitudinal and biomedical studies, however, functional trajectories are observed sparsely and at irregular time points. For instance, in the IBD dataset analyzed in Section~\ref{Real data analysis}, each subject is observed at only a subset of time points, with observation schedules varying across subjects and variables. Under such designs, the matrix can not be represented as a complete matrix, rendering the formulation in \eqref{bicluster fsvd general} inapplicable in its original form. 

To accommodate sparse and irregular sampling, we reformulate the Tri-SfSVD objective to operate directly on the observed evaluation points, thereby avoiding the need for pre-smoothing or imputation. Let \(\pmb{t}_{i,j} := [t_{i,j,k}]_{k = 1}^{m_{i,j}}\) denote the observation time points for subject \(i\) on variable \(j\), where \(m_{i,j}\) is the number of observed measurements for subject-variable pair, and \(k\) indexes the observed time points. Let \(\pmb{y}_{i,j}:=(y_{i,j,1},\ldots,y_{i,j,m_{i,j}})^\top\) collect the corresponding observed measurements. The sparse-data extension of \eqref{bicluster fsvd general} is then defined as 
\begin{multline}\label{best approximation tri}
\min_{s,\,\pmb{u},\,\{\pmb{\varphi}_j\}_{j=1}^{p}}
\Bigg\{
\sum_{i=1}^{n}\sum_{j=1}^{p}\sum_{k=1}^{m_{i,j}}
\big[y_{i,j,k} - s u_i \varphi_{j}(t_{i,j,k})\big]^2
+ \gamma \sum_{i=1}^{n} w_{1,i} |s u_i| \\
+ \theta \sum_{j=1}^{p} w_{2,j} \Vert s\pmb{\varphi}_j \Vert_2
+ \sum_{j=1}^{p} \lambda_j
\Vert \pmb{w}_{3,j} \odot (s\pmb{\varphi}_j) \Vert_1
+ \sum_{j=1}^{p} \alpha_j
(s\pmb{\varphi}_j)^{\top}\pmb{\Omega}_j(s\pmb{\varphi}_j)
\Bigg\},
\end{multline}
subject to \(s\ge 0\), \(\Vert\pmb{u}\Vert_2=1\), and \(\Vert\pmb{\varphi}\Vert_2=1\).
Here \(\pmb{\varphi}_j=\big(\varphi_j(t^*_{j1}), ...,\varphi_j(t^*_{jd_j})\big)^{\top}\in \mathbb{R}^{d_j}\) represents the loading function evaluated on a common grid for variable $j$, $\{t^*_{j1},...,t^*_{jd_j}\}$, obtained by different time points of the set $\{t_{i,j,k}; i=1,...,n;k=1,...,m_{i,j}\}$. This point-wise reconstruction loss can be interpreted as a masked matrix optimization: $\sum_{j=1}^{p} \left\lVert \pmb{I}_j \odot (\pmb{X}_j - s\pmb{u}\,\pmb{\varphi}_j^{\top}) \right\rVert_{F}^{2}$, where $\pmb{I}_j\in\{0,1\}^{n\times d_j}$ is the indicator matrix such that  $I_j(i,\ell)=1$ if the data point for subject $i$ is observed at time $t^*_{j\ell}$ and 0 otherwise, and $\bm{X}_j=\big(x_{ij}(t^*_{j\ell})\big)_{i,\ell} \in \mathbb{R}^{n \times d_j}$ is the hypothetical complete matrix with missing entries masked.  The formulation remains applicable even when some trajectories are observed at only a few time points. Moreover, because the objective is defined directly on the observed evaluation points, it does not require explicit modeling of the missingness mechanism.
\subsubsection{Alternating Iterative Estimation Procedure}
Following the scaled-variable sparse SVD strategy of Lee et al. (2010), we estimate the sparse rank-one layer by alternating penalized regressions on the scaled working variables \(\widetilde{\pmb{u}}=s\pmb{u}\) and \(\widetilde{\pmb{\varphi}}=s\pmb{\varphi}\), followed by normalization. Specifically, we fix the normalized loading direction and update \(\widetilde{\pmb{u}}\), then fix the normalized subject direction and update \(\widetilde{\pmb{\varphi}}\).
~\\
\textit{First stage: update \(\widetilde{\pmb{u}}\).} The normalized loading functions \(\{\pmb{\varphi}_j\}_{j=1}^{p}\) are held fixed, and the scaled subject-specific scores  \(\widetilde{\pmb{u}}\) are updated by solving
\begin{equation}\label{best approximation step 1} 
\min_{\widetilde{\pmb{u}}}
\Bigg\{
\sum_{i=1}^{n}\sum_{j=1}^{p}\sum_{k=1}^{m_{i,j}}
\big[y_{i,j,k} - \widetilde u_i \varphi_j(t_{i,j,k})\big]^2
+ \gamma \sum_{i=1}^{n} w_{1,i} |\widetilde u_i|
\Bigg\},
\end{equation}
For notational convenience, define $\pmb{y}_i := 
(\pmb{y}_{i,1}^{\top}, \ldots, \pmb{y}_{i,p}^{\top})^{\top}$ and $\pmb{\varphi}_{i\cdot}^{*} := \big[ {\pmb{\varphi}_{i,1}^{*}}^{\top}, \ldots, {\pmb{\varphi}_{i,p}^{*}}^{\top} \big]^{\top},$ where $\pmb{\varphi}_{ij}^{*} = \big(\varphi_{j}(t_{i,j,1}),\cdots,\varphi_{j}(t_{i,j,m_{ij}})\big)$ collects the evaluations of the $j$th loading function at the observation times for subject $i$.
Since \eqref{best approximation step 1} can be rewritten by grouping terms according to each subject $i$, the minimization over \(\widetilde{\pmb{u}}\) decouples into $n$ independent scalar problems,
\begin{equation}\label{best approximation step 1 single}
\min_{\widetilde u_i}
\Big\{
\Vert \pmb{y}_i - \pmb{\varphi}^{*}_{i\cdot} \widetilde u_i \Vert_2^{2}
+ \gamma w_{1,i} |\widetilde u_i|
\Big\},
\qquad i=1,\ldots,n.
\end{equation}
Each subproblem in \eqref{best approximation step 1 single} is a one-dimensional quadratic loss with an $\ell_1$ penalty, and therefore admits a closed-form solution via the soft-thresholding operator:
\begin{equation}\label{u solution}
\widetilde u_i
=
\frac{\mathrm{sign}\big({\pmb{\varphi}^{*}_{i\cdot}}^\top \pmb{y}_i\big)}
{{\pmb{\varphi}^{*}_{i\cdot}}^{\top} \pmb{\varphi}^{*}_{i\cdot}}
\left(
\big|{\pmb{\varphi}^{*}_{i\cdot}}^\top \pmb{y}_i\big|
- \frac{\gamma w_{1,i}}{2}
\right)_+,
\qquad i=1,\ldots,n.
\end{equation}
where \((a)_+ = \max\{a,0\}\) denotes the positive-part operator for any real number \(a\in\mathbb{R}\). The resulting vector \(\widetilde{\pmb{u}}\) is a scaled subject-score vector. We normalized it by setting \(\pmb{u}=\widetilde{\pmb{u}}/\Vert\widetilde{\pmb{u}}\Vert_2\). This thresholding step effectively induces sparsity at the subject level, allowing the model to identify specific subjects belonging to a latent cluster.~\\
\textit{Second stage: Update \(\widetilde{\pmb{\varphi}}_j\).} We fix the normalized \(\pmb{u}\) and update \(\{\widetilde{\pmb{\varphi}}_j\}_{j=1}^{p}\)  by solving
\begin{equation}\label{best approximation step 2}
\begin{aligned}
\min_{\{\widetilde{\pmb{\varphi}}_j\}_{j=1}^{p}} \quad
&\Bigg\{
\sum_{i=1}^{n}\sum_{j=1}^{p}\sum_{k=1}^{m_{i,j}}
\big[y_{i,j,k} - u_i \widetilde{\varphi}_{j}(t_{i,j,k})\big]^2
+ \theta \sum_{j=1}^{p} w_{2,j} \Vert \widetilde{\pmb{\varphi}}_j \Vert_2 + \sum_{j=1}^{p} \lambda_j
\Vert \pmb{w}_{3,j} \odot \widetilde{\pmb{\varphi}}_j \Vert_1
+ \sum_{j=1}^{p} \alpha_j
\widetilde{\pmb{\varphi}}_j^{\top}\pmb{\Omega}_j\widetilde{\pmb{\varphi}}_j
 \Bigg\}.
\end{aligned}
\end{equation}
This objective decouples into \(p\) independent subproblems:
\begin{equation}\label{best approximation step 2 single}
\min_{\widetilde{\pmb{\varphi}}_j}
\left\{
\sum_{i=1}^{n}\sum_{k=1}^{m_{i,j}}
\big[y_{i,j,k} - u_i \widetilde{\varphi}_{j}(t_{i,j,k})\big]^2
+ \theta w_{2,j} \Vert \widetilde{\pmb{\varphi}}_j \Vert_2
+ \lambda_j \Vert \pmb{w}_{3,j} \odot \widetilde{\pmb{\varphi}}_j \Vert_1
+ \alpha_j \widetilde{\pmb{\varphi}}_j^{\top}\pmb{\Omega}_j\widetilde{\pmb{\varphi}}_j
\right\}.
\end{equation}
This decomposition provides an important computational advantage in high-dimensional settings: instead of solving one large coupled optimization problem, we only need to solve $p$ smaller independent subproblems. As a result, the optimization is more scalable and can be efficiently parallelized across variables. 

To facilitate an efficient estimation of the functional components, we first introduce a more compact notation. Let \(\tilde{\pmb{y}}_j := \big[ \pmb{y}_{1,j}^{\top}, \ldots, \pmb{y}_{n,j}^{\top} \big]^{\top} \in \mathbb{R}^{M_j}\) represent the concatenated vector of measurements for variable \(j\) across all subjects, where \(M_j = \sum_{i=1}^n m_{i,j}\).  We define the mapping matrix \(\pmb{B}_{i,j} \in \mathbb{R}^{m_{i,j} \times d_j}\) which maps  \(\pmb{\varphi}_j\) to \(\pmb{\varphi}^{*}_{i,j}\) (that is \(\pmb{B}_{i,j} \pmb{\varphi}_j = \pmb{\varphi}^{*}_{i,j}\)) such that  $B_{i,j}(k,\ell)=1$ if $t_{i,j,k}=t^*_{j\ell}$ and 0 otherwise for $k=1,...,m_{i,j}$ and $\ell=1,...,d_j$. 
By constructing the design matrix \(\pmb{U}_j = [\,u_1\pmb{B}_{1,j}^\top,\ldots,u_n\pmb{B}_{n,j}^\top\,]^\top \in \mathbb{R}^{M_j\times d_j}
\), the subproblem  \eqref{best approximation step 2 single} can be expressed as
\begin{equation}\label{best approximation step 2 matrix}
\min_{\widetilde{\pmb{\varphi}}_j}
\left\{
\Vert \tilde{\pmb{y}}_j - \pmb{U}_j \widetilde{\pmb{\varphi}}_j \Vert_2^{2}
+ \theta w_{2,j} \Vert \widetilde{\pmb{\varphi}}_j \Vert_2
+ \lambda_j \Vert \pmb{w}_{3,j} \odot \widetilde{\pmb{\varphi}}_j \Vert_1
+ \alpha_j \widetilde{\pmb{\varphi}}_j^{\top}\pmb{\Omega}_j\widetilde{\pmb{\varphi}}_j
\right\}.
\end{equation}


Motivated by the standard Fast Iterative Shrinkage-Thresholding Algorithm (FISTA) framework \citep{fista}, we develop a modified FISTA to solve \eqref{best approximation step 2 matrix}. When \(\lambda_j=0\), the triclustering penalty reduces to the biclustering penalty, and the details of the modified FISTA algorithm, including this biclustering special case, are provided in the Supplementary Material (Section~\ref{algorithm} and Algorithm~\ref{alg:fista_sga_lasso}).  After obtaining ${\{\widetilde{\pmb{\varphi}}_j\}_{j=1}^{p}}$,  normalize \(\widetilde{\pmb{\varphi}}\) where \(\widetilde{\pmb{\varphi}}=(\widetilde{\pmb{\varphi}}_1^\top,\ldots,\widetilde{\pmb{\varphi}}_p^\top)^\top\).



For model selection, we adopt a computationally efficient tuning strategy for the hyperparameters. While the subject-level ($\gamma$) and feature-level ($\theta$) sparsity parameters are treated as global, we select the variable-specific roughness ($\alpha_j$) and temporal sparsity ($\lambda_j$) parameters independently for each scaled-loading subproblem in \eqref{best approximation step 2 single}. This localized tuning significantly reduces the high-dimensional search space without sacrificing the ability to capture variable-specific temporal dynamics. The selection of all tuning parameters
is guided by the Extended Bayesian Information Criterion (EBIC; \citealp{chen2008ebic}), which provides a principled balance between model complexity and data fidelity. Comprehensive details of the tuning procedure and the complete Tri-SfSVD algorithm are provided in the Supplementary Material (Sections~\ref{Tuning Selection} and Algorithm~\ref{alg:complete_supp}). In addition, strategies for selecting the number of tri/biclusters, $K$, are discussed in Section~\ref{Selection of the Number of Clusters} in the Supplementary Material. In our simulation and real data analyses, we adopt the cumulative explained variance (CEV) strategy.


\subsection{Flexibility for Non-Overlapping and Overlapping Estimation}
In real data analysis, users may have varying preferences regarding the structure of the estimated bi- and triclusters. For instance, some may prefer non-overlapping results, where the clusters are disjoint, ensuring that samples, features, or subregions are assigned to only one cluster. Conversely, others may desire overlapping clusters to capture more complex relationships where elements can belong to multiple groups. To accommodate these diverse needs and enhance the flexibility of our method, we incorporate a weighting mechanism for non-overlapping scenarios, following the idea of \cite{biclustersandra}: after identifying a bicluster or tricluster, we assign zero weights to the involved samples or features, thereby enforcing zero coefficients for them in subsequent estimations. This approach allows users to seamlessly switch between non-overlapping and overlapping modes based on their specific requirements.

\section{Simulation Study} 
\label{Simulation Study}
\subsection{Simulation Setup}
We evaluate the empirical performance of the Tri-SfSVD framework through a series of numerical experiments designed to mimic the complexities of longitudinal high-dimensional data. We generate a \(p\)-dimensional functional object
\(\pmb{X}(\pmb{t})=(X_1(t),\ldots,X_p(t))\)  defined on the unit interval \([0,1]\).  To evaluate the model's ability to recover multi-component signals, data are generated from a rank-$K$ functional model:
$$\bm{X}_i(t) = \sum_{k=1}^K \lambda^{(k)} u_i^{(k)} \boldsymbol{\varphi}^{(k)}(t) + \boldsymbol{\epsilon}_i(t)$$

where $\lambda^{(k)}$ denotes the $k$-th singular value, $\mathbf{u}^{(k)} \in \mathbb{R}^n$ represents the $k$-th left singular vector (subject scores), and $\boldsymbol{\varphi}^{(k)}(t) = (\varphi_1^{(k)}(t), \dots, \varphi_p^{(k)}(t))^\top$ are the $p$-dimensional right singular functions. The noise term $\boldsymbol{\epsilon}_i(t)$ is assumed to follow a multivariate normal distribution $\mathcal{N}_p(\mathbf{0}, 0.5^2 \mathbf{I}_p)$. Following standard functional SVD conventions, we enforce identifiability by setting $\{\mathbf{u}^{(k)}\}_{k=1}^K$ to be orthonormal in $\mathbb{R}^n$ and $\{\boldsymbol{\varphi}^{(k)}\}_{k=1}^K$ to be orthonormal in the product Hilbert space $\mathbb{H}$, such that $\Vert\mathbf{u}^{(k)}\Vert = 1$ and $\Vert\boldsymbol{\varphi}^{(k)}\Vert_{\mathbb{H}} = 1$ for all $k$. Each variable is initially observed on a dense grid of $d=40$ equally spaced time points before we introduce varying levels of artificial missingness to test the framework's robustness to the irregular sampling common in clinical studies.


We fix the rank \(K=4\) and the sample size \(n=100\). 
We vary the number of functional variables to three settings \(p\in\{60,200,1000\}\). For each setting, the proportion of active variables, those possessing a non-zero loading function in at least one latent component, is fixed at approximately 70\% (corresponding to 40, 140, and 700 variables, respectively). The remaining 30\% of variables are inactive, with trajectories identically zero across the entire domain. 
The $p$-dimensional right singular functions $\{\boldsymbol{\varphi}^{(k)}\}_{k=1}^K$ are constructed to test the model's ability to distinguish between biclusters and triclusters. For each active variable, we assign a functional shape by sampling from a dictionary of ten candidate curves (see Supplementary Figure~\ref{fig:sim_candidate_curves}). These curves are categorized into two structural types: i) Full-domain continuous functions, representing global patterns defined over the entire interval $[0, 1]$, and ii) Subregion-localized functions, which are restricted to either the left or right half of the discrete grid with the remaining domain set to zero. This dual-type design specifically evaluates whether Tri-SfSVD can distinguish between features with global support and those with time-domain sparsity in a data-driven manner.

Regarding the left singular vectors $\{\bm{u}^{(k)}\}_{k=1}^K$, which represent subject-specific scores, we select a subset of indices $S_k \subset \{1, \dots, n\}$ for each component. Entries for indices in $S_k$ are drawn from $\mathcal{N}(1, 0.3^2)$ and set to zero otherwise, thereby inducing subject-level sparsity. The singular values are fixed at $(\lambda^{(1)}, \lambda^{(2)}, \lambda^{(3)}, \lambda^{(4)}) = (10, 8, 6, 4)$ to represent a decreasing signal-to-noise ratio across components. Finally, we simulate the challenges of irregular longitudinal sampling by imposing missing rate levels of 40\% and 60\% on the observed grid.


For each setting, we consider both nonoverlapping and overlapping configurations, in which the nonzero supports of \(\{\pmb{u}^{(k)}\}_{k=1}^{K}\) and \(\{\pmb{\varphi}^{(k)}\}_{k=1}^{K}\) are either disjoint or overlapping, respectively. Details on how orthogonality is handled under the overlapping setting are given in Section~\ref{sec:overlap_orthogonality} in the Supplementary Material. We performed 100 independent simulation replicates for each configuration. Figure~\ref{clustercomparisons} in the supplementary material illustrates nonoverlapping and overlapping bicluster structures, along with the estimated results from our approach, for the setting with \(p=60\) and a missing rate of \(60\%\). In nonoverlapping cases, biclusters form disjoint rectangular blocks that do not share samples or features. In contrast, overlapping cases allow biclusters to share a subset of samples or features, leading to intersecting block structures. We therefore evaluate performance under both configurations. This design is motivated by real data applications, where the underlying structure is typically unknown, and assessing both configurations provides a more comprehensive evaluation of robustness. Figure~\ref{fig:sim_mean_curve_example} in the Supplementary Material shows the detailed tricluster results from a randomly selected nonoverlapping dataset simulated with \(p=60\) and a missing rate of \(40\%\).

To quantify clustering performance, we adopt the similarity framework from~\cite{Sill2011} to compute $  F  $-scores at multiple levels, including subject, variable, subregion, bicluster, and tricluster; see Appendix~\ref{Performance Evaluation} for the detailed definitions and formulas.

\subsection{Simulation Results}
We applied our proposed method alongside five competing approaches. The first is a two-step procedure based on multivariate functional principal component analysis (MFPCA) \citep{happ2018multivariate}. Specifically, we first apply MFPCA to the multivariate functional data to obtain low-dimensional subject-level score vectors, and then apply \(k\)-means clustering to these scores to obtain sample clusters. We refer to this method as MFPCA-Kmean. The second is a two-stage baseline approach, denoted by FPCA-BC. Specifically, for each feature, we first applied functional principal component analysis (FPCA) \citep{PACE} and extracted its first left singular vector. We then concatenated these vectors across features to form an \(n\times p\) matrix, to which a standard biclustering algorithm \citep{biclusterthierry} was applied to obtain biclusters. The third competing approach is a well-known triclustering method \citep{tri_zhao}, referred to as TriCluster, which is designed to identify joint patterns across the sample, feature, and time dimensions in three-dimensional datasets. The remaining two competing approaches are existing functional biclustering methods: funLBM \citep{modelbasedfunctionalbi} and funCC \citep{funcc}. Since TriCluster, funLBM and funCC require complete input data, we first applied PACE, which is well suited for recovering the functional trajectories and then applied these methods to the resulting completed data. The simulation results are presented in Figure~\ref{triclusternonoverlap} for the nonoverlapping and overlapping settings. Each figure contains three boxplots summarizing \(F\)-scores for recovery on the sample, sample-feature, and sample-feature-time dimensions. Here, the sample-level \(F\)-score is computed from the sample clustering induced by each method's estimated tricluster structure, rather than from a separately performed sample clustering analysis. The second boxplot summarizes recovery of sample-feature bicluster assignments, and the third summarizes recovery of sample-feature-time tricluster assignments. For the sample-level recovery, we compare all methods. For bicluster recovery, we compare UFPCA, TriCluster, funLBM, funCC and Tri-SfSVD, since MFPCA produces sample-level clustering but does not output feature clusters or bicluster assignments. For tricluster recovery, we also compare only Tri-SfSVD and TriCluster, since other approaches do not produce tricluster output. 

In general, the results demonstrate the superior performance of Tri-SfSVD across all evaluated cases, achieving higher \(F\)-scores and showing strong robustness, particularly in highly sparse and high-dimensional settings. In particular, TriCluster exhibits consistently low \(F\)-scores under this simulation design. TriCluster is designed to group features primarily based on the similarity of functional shapes. In our simulations, however, a tricluster may contain features with heterogeneous shapes, provided that they jointly capture the dominant variation within the selected samples. Under such heterogeneous shape settings, shape-driven methods, such as TriCluster, tend to underperform because they do not prioritize contributory patterns when within-cluster shapes differ. In contrast, Tri-SfSVD directly targets the contributory structure under subregion sparsity, enabling it to recover coherent triclusters even when within-cluster shapes vary substantially.

\begin{figure}[!htbp]
  \centering

  \begin{subfigure}[b]{0.48\textwidth}
    \centering
    \includegraphics[width=\linewidth]{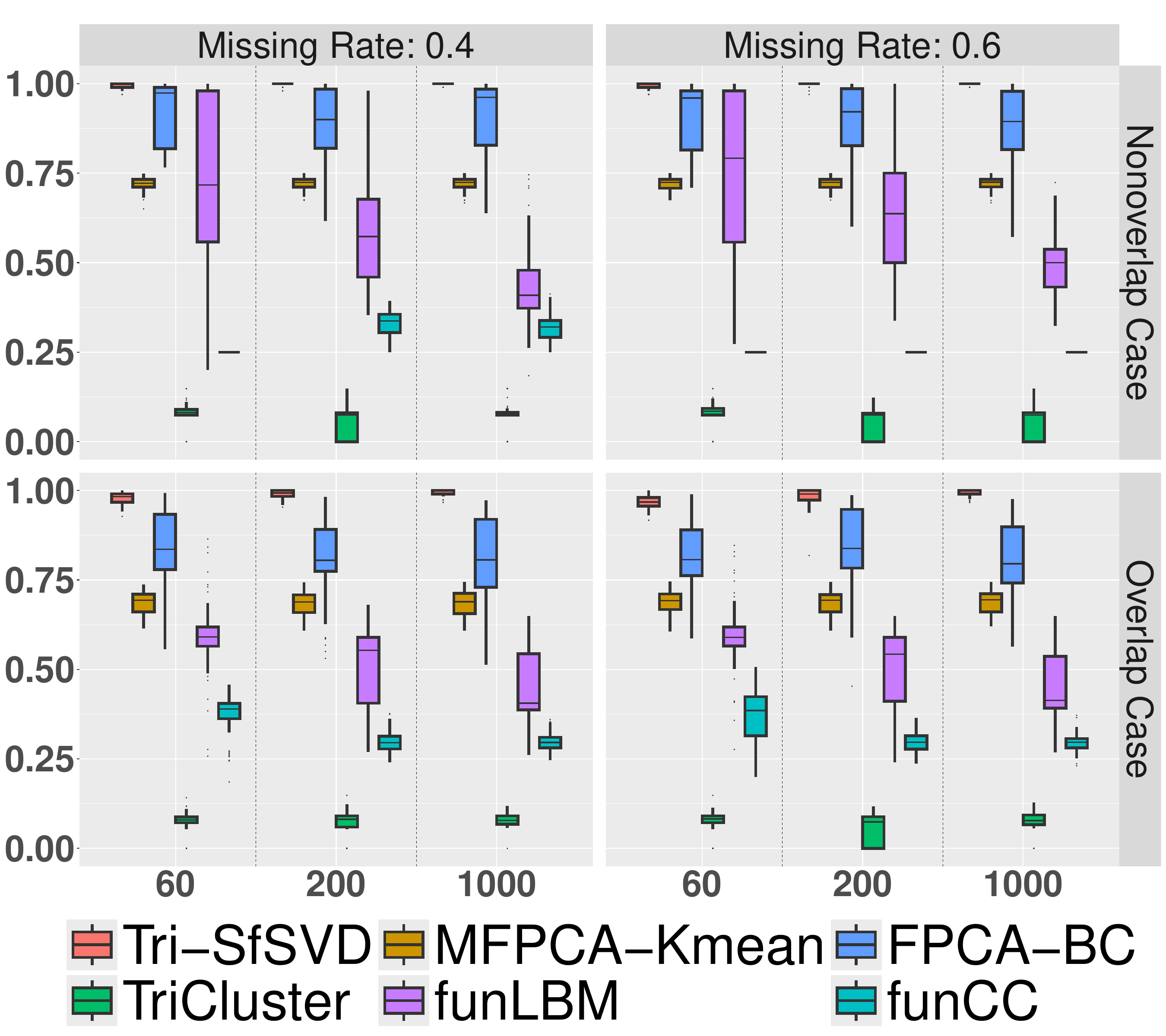}
    \caption{Sample cluster}
  \end{subfigure}
  \hfill
  \begin{subfigure}[b]{0.48\textwidth}
    \centering
    \includegraphics[width=\linewidth]{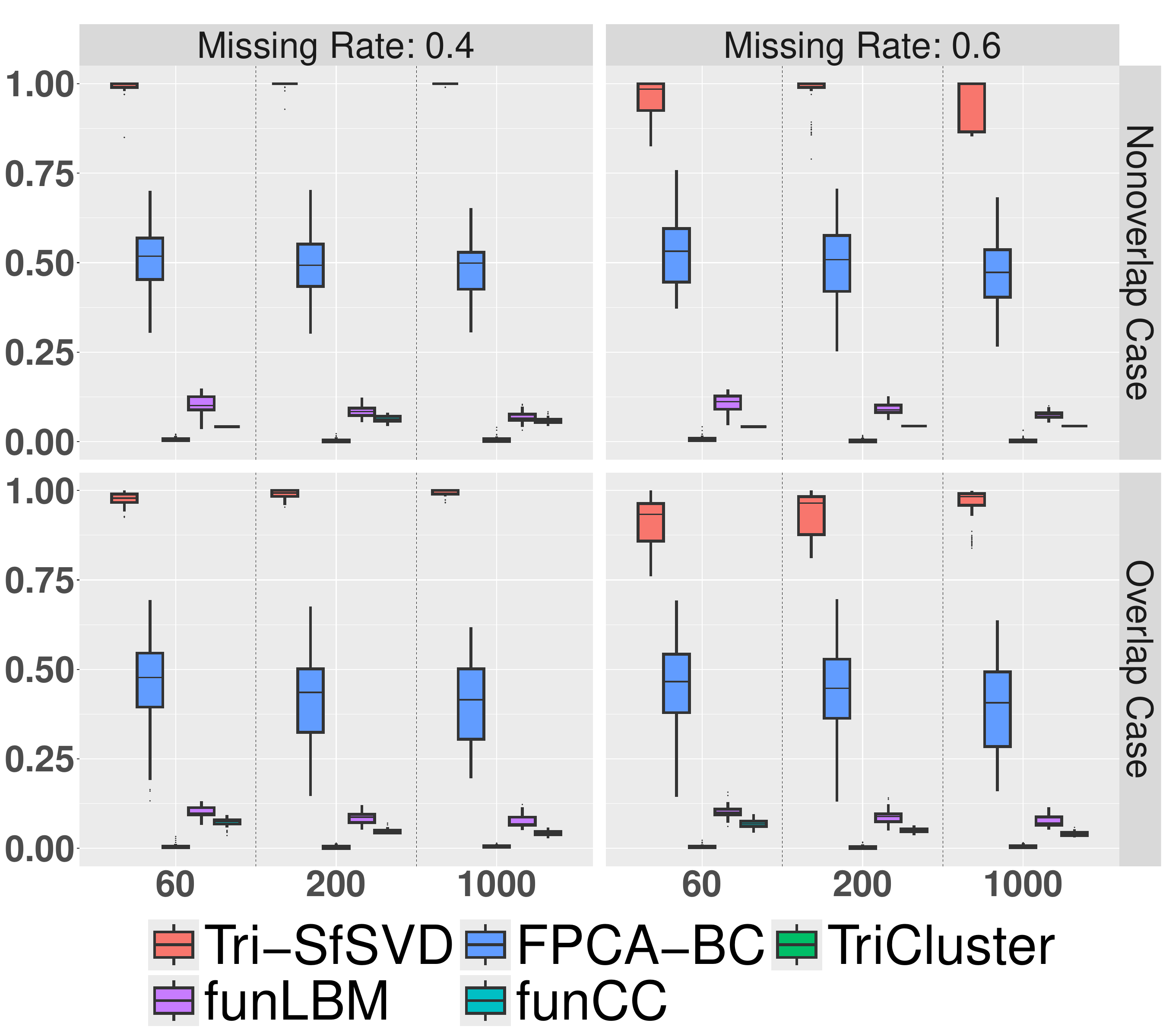}
    \caption{Bicluster}
  \end{subfigure}

  \vspace{0.8em}

  \begin{subfigure}[b]{0.48\textwidth}
    \centering
    \includegraphics[width=\linewidth]{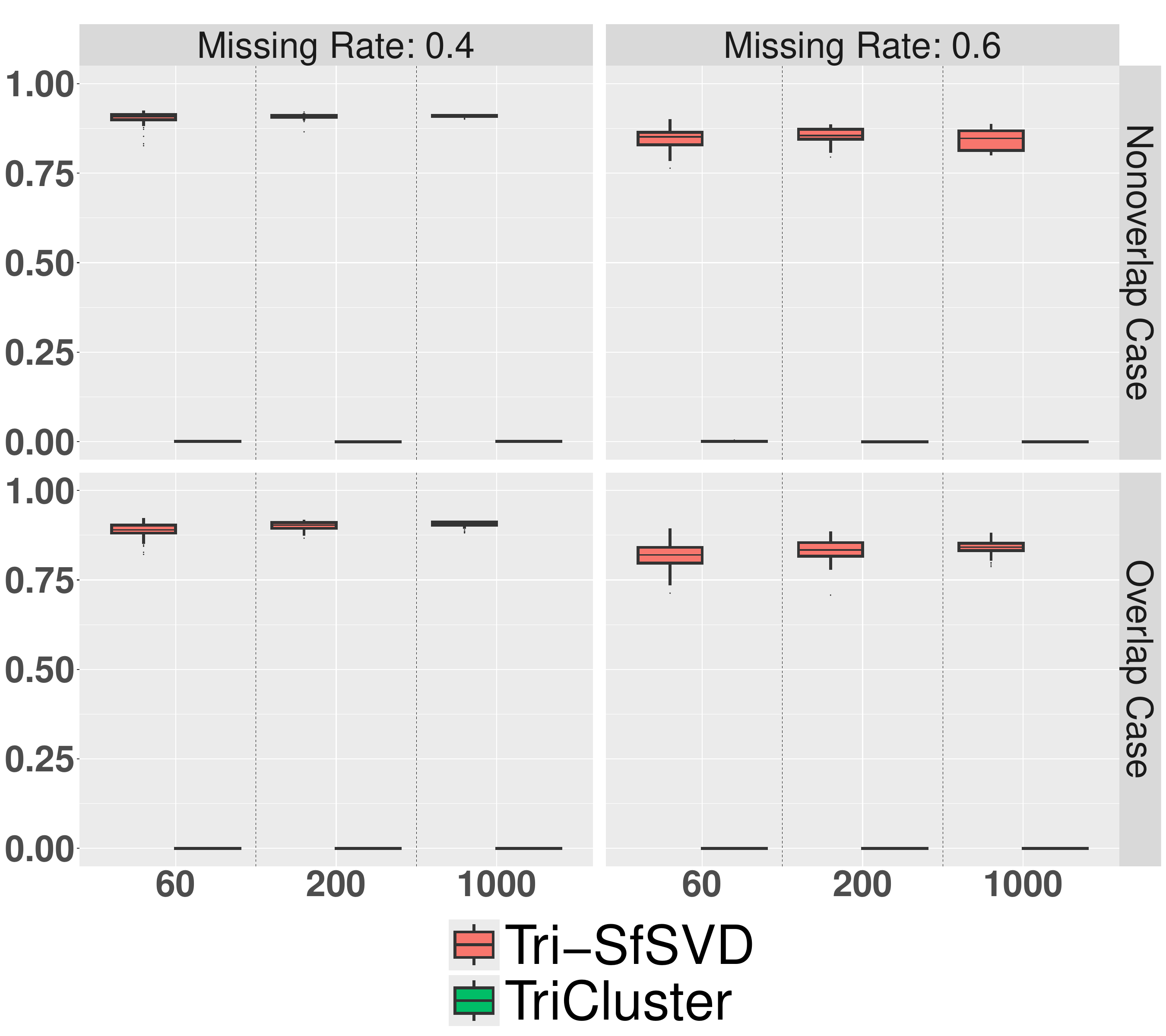}
    \caption{Tricluster}
  \end{subfigure}

  \caption{\(F\)-scores for simulation studies under non-overlapping and overlapping tricluster structures. Panels (a)--(c) summarize recovery performance at the sample, sample--feature bicluster, and sample--feature--time tricluster levels, respectively, over 100 simulation replications. Results are reported across different numbers of variables and missing rates. Larger \(F\)-scores indicate better agreement between the estimated and true clustering structures.}

  \label{triclusternonoverlap}
\end{figure}

\FloatBarrier

\section{Real Data Analysis} 
\label{Real data analysis}
In this section, we analyze the IBDMDB dataset, which tracks longitudinal microbial functional profiles from 130 participants over a 50-week period. The primary scientific objective is to identify clusters of subjects and microbial pathways that exhibit different longitudinal patterns in IBD. We use the preprocessed metagenomic data from \cite{msflda}. Specifically, we focus on the top 200 species-stratified metagenomic pathways identified in that study from an initial set of 1,622 candidates. Details of the preprocessing and feature selection procedure can be found in the same study. Each feature represents a microbial functional pathway associated with a particular bacterial species. Therefore, these features jointly encode pathway-level functional information and the bacterial species in which that function is observed. More details on the retained pathway functions and their bacterial-species distribution are provided in Table~\ref{tab:ibd_species_distribution} in the Supplementary Material. Before applying our method, each retained feature was standardized by subtracting its mean and dividing by its standard deviation across all observed values.
The number of sparse layers, denoted by \(K\), was determined using a cumulative explained variance criterion (see Section~\ref{Selection of the Number of Clusters} of the Supplementary Material. We selected  \(K=3\), which together account for over 90\% of the total variance. These three extracted layers effectively encompass all 130 participants and identify a subset of 151 metabolic pathways, providing an interpretable representation of the longitudinal disease signatures. 

\subsection{Clinical Interpretation of the Identified Feature Clusters} \label{Feature Sparsity via WGCNA}
Each sparse layer selected a subset of pathways through nonzero functional loadings. We therefore defined the pathway group \(\pmb{F}_k\) \((k=1,2,3)\) as the set of pathways with nonzero functional loadings in the \(k\)th sparse layer.
Notably, the pathways within each group were mainly associated with the same bacterial label, suggesting that each \(\pmb{F}_k\) captures a functionally coherent set of microbial pathways associated with a dominant bacterium. Specifically, \(\pmb{F}_1\) contains 65 pathways primarily associated with \textit{Alistipes putredinis}, \(\pmb{F}_2\) contains 49 pathways primarily associated with \textit{Ruminococcus torques}, and \(\pmb{F}_3\) contains 26 pathways primarily labeled as unclassified in this dataset. Detailed information on the pathways included in each \(\pmb{F}_k\) is provided in Section~\ref{real_data_appendix} of the Supplementary Material. In this application, the data-driven tuning procedure did not produce meaningful subregion-level sparsity in the estimated functional loadings across pathways, suggesting that the dominant pattern was better captured by subject--feature clustering over the full time domain. Figure~\ref{fig:ibd_feature_curves} in the Supplementary Material displays the estimated functional loadings for all pathways across the three sparse layers.

\subsection{Sample Cluster Refinement and Bicluster Association Analysis}

In our initial analysis, the EBIC/BIC-based selection tended toward conservative sparsity estimates (i.e., low sparsity) for the subject scores, a phenomen previously observed in \cite{biclusterssvd}. This resulted in layers that retained a large proportion of subjects, leading to substantial overlap across components and obscuring the boundaries between distinct biological subtypes. Similar challenges in achieving clear sample separability through sparse matrix decompositions alone have been noted by \cite{Sill2011}.

To enhance the separability and biological interpretability of the sample clusters, we adopted a two-step refinement strategy to delineate clear bicluster structures. First, we refined the sample clusters by applying \(k\)-means clustering to the left singular matrix estimated by our method. In this matrix, each column corresponds to the left singular vector \(\pmb{u}\) from a specific layer, while each row contains the latent scores for a given sample across all identified layers. 

By applying \(k\)-means to these row-vectors, we partitioned the \(n=130\) samples into three disjoint sample clusters, denoted by \(\{\pmb{S}_j\}_{j=1}^{3}\), where the choice of $k=3$ was aligned with the number of detected feature clusters.  

Let \(\pmb{C}\in\{0,1\}^{130\times3}\) denote the resulting sample cluster indicator matrix, where \(C_{ij}=1\) if sample \(i\in\pmb{S}_j\) and \(C_{ij}=0\) otherwise. Second, to characterize the functional behavior within these groups, we derived a low-dimensional representation of each pathway group by adapting the module eigengene farmework from Weighted Gene Co-expression Network Analysis \citep{WGCNA}. Specifically, for each pathway group \(\pmb{F}_k\), we defined its module eigengene as the first principal component of the data matrix restricted to the pathways in that group, denoted by \(\pmb{m}_k \in \mathbb{R}^{130}\). Thus, \(\pmb{m}_k\) captures the dominant variation pattern shared by the pathways in \(\pmb{F}_k\) across the 130 samples.

We then quantified the association between sample cluster \(\pmb{S}_j\) and pathway group \(\pmb{F}_k\) using distance correlation, $r_{jk}=\mathrm{dCor}(\pmb{C}_{\cdot j},\,\pmb{m}_k)$. Distance correlation ranges from 0 to 1, with larger values indicating stronger dependence. Unlike ordinary correlation, which mainly captures linear association, distance correlation can detect both linear and nonlinear dependence, making it useful for measuring association between a binary sample-cluster indicator and a continuous module eigengene. In our setting, a relatively large value of \(r_{jk}\) indicates a strong association between membership in sample cluster \(\pmb{S}_j\) and the pathway pattern summarized by \(\pmb{m}_k\). Accordingly, pairs \((\pmb{S}_j,\pmb{F}_k)\) with relatively large values of \(r_{jk}\) were interpreted as representative biclusters, reflecting the core associations between sample clusters and pathway groups.

\subsection{Clinical Characterization of the Identified Sample Clusters}
\label{Sample Clusters Identified}
Following the procedure described above, we obtained three sample clusters \(\{\pmb{S}_j\}_{j=1}^{3}\), which were compared across clinical phenotypes using Fisher's exact test for binary variables and the Kruskal--Wallis test for continuous variables; the results are summarized in Table~\ref{tab:phenotype}. A significant association was observed between the sample clusters and IBD type (\(P=0.0019\)), whereas the updated sex association did not reach the 0.05 significance level (\(P=0.0845\)). Specifically, \(\pmb{S}_1\) was enriched for CD (57.6\%) and females (63.6\%), with lower proportions of UC (18.2\%) and non-IBD individuals (24.2\%). In contrast, \(\pmb{S}_2\) contained the largest proportion of non-IBD subjects among the three clusters (35.6\%) and was predominantly male (62.2\%). Finally, \(\pmb{S}_3\) was dominated by IBD subjects (94.2\%; CD 57.7\%, UC 36.5\%) and showed a balanced sex distribution (50.0\% female, 50.0\% male), suggesting an IBD-associated profile spanning both CD and UC.

\begin{table}[h]
\centering
\small
\setlength{\tabcolsep}{6pt}
\begin{tabular}{
  >{\raggedright\arraybackslash}p{2.2cm}
  >{\raggedright\arraybackslash}p{3.6cm}
  >{\centering\arraybackslash}p{1.6cm}
  >{\centering\arraybackslash}p{1.6cm}
  >{\centering\arraybackslash}p{1.6cm}
  >{\centering\arraybackslash}p{2.2cm}}
\hline
 &  & \multicolumn{3}{c}{Sample clusters (N)} & \multirow{2}{*}{P-value} \\
\cline{3-5}
 &  & $\pmb{S}_1$ (33) & $\pmb{S}_2$ (45) & $\pmb{S}_3$ (52) & \\
\hline
\multirow{3}{*}{IBD type}
  & Non-IBD & 8 (24.2\%)  & 16 (35.6\%) & 3 (5.8\%)  & \multirow{3}{*}{0.0019} \\
  & CD      & 19 (57.6\%) & 16 (35.6\%) & 30 (57.7\%) & \\
  & UC      & 6 (18.2\%)  & 13 (28.9\%) & 19 (36.5\%) & \\
\hline
\multirow{2}{*}{Sex}
  & Female  & 21 (63.6\%) & 17 (37.8\%) & 26 (50.0\%) & \multirow{2}{*}{0.0845} \\
  & Male    & 12 (36.4\%) & 28 (62.2\%) & 26 (50.0\%) & \\
\hline
\multicolumn{2}{>{\raggedright\arraybackslash}p{5.8cm}}{\small Age at diagnosis, mean (SD)}
  & 21.12 (12.08) & 20.96 (14.92) & 21.63 (10.65) & 0.5637 \\
\multicolumn{2}{>{\raggedright\arraybackslash}p{5.8cm}}{\small CRP (C-Reactive Protein), mean (SD)}
  & 60.33 (241.90) & 63.18 (234.15) & 49.68 (203.66) & 0.3789 \\
\multicolumn{2}{>{\raggedright\arraybackslash}p{5.8cm}}{\small ESR (Erythrocyte Sedimentation Rate), mean (SD)}
  & 18.65 (19.81) & 21.61 (17.51) & 23.96 (20.82) & 0.3588 \\
\hline
\multirow{5}{*}{Race}
  & \parbox[t]{3.6cm}{\scriptsize American Indian or Alaska Native} & 0 & 0 & 1 & \multirow{5}{*}{0.8205} \\
  & \parbox[t]{3.6cm}{\scriptsize Black or African American}        & 1 & 5 & 4 & \\
  & \parbox[t]{3.6cm}{\scriptsize More than one race}               & 1 & 3 & 1 & \\
  & \parbox[t]{3.6cm}{\scriptsize Other}                            & 1 & 1 & 2 & \\
  & \parbox[t]{3.6cm}{\scriptsize White}                            & 30 & 36 & 44 & \\
\hline
\end{tabular}
\caption{Associations between sample clusters and clinical phenotypes in the IBD dataset. Entries are counts (percentages) for categorical variables and mean (SD) for continuous variables. \(P\)-values are from Fisher's exact tests for categorical variables and Kruskal--Wallis tests for continuous variables. The table shows that the identified sample clusters differ significantly in IBD type, whereas no significant differences are observed for sex, age at diagnosis, CRP, ESR, or race.}
\label{tab:phenotype}
\end{table}

\subsection{Interpretation of Identified Biclusters}
\label{Interpretation of Biclusters Identified}

Table~\ref{tab:assoc-bootmean} summarizes the associations between the identified subject clusters \(\{\pmb{S}_j\}_{j=1}^{3}\) and the pathway groups \(\{\pmb{F}_k\}_{k=1}^{3}\). Based on the largest distance correlation values, we identified five primary biclusters: \((\pmb{S}_1,\pmb{F}_1)\), \((\pmb{S}_3,\pmb{F}_1)\), \((\pmb{S}_2,\pmb{F}_2)\), \((\pmb{S}_3,\pmb{F}_2)\), and \((\pmb{S}_3,\pmb{F}_3)\).  Figure~\ref{fig:ibd_overlap_bicluster_mean} provides a visual summary of these identified biclusters by displaying three randomly selected pathways from each pathway group, with full longitudinal profiles for all pathways provided in the Supplementary Figures~\ref{fig:supp_S1F1}--\ref{fig:supp_S3F3}.

Pathway group \(\pmb{F}_1\) is dominated by pathway features stratified to \textit{Alistipes putredinis}. Based on the distance-correlation criterion in Table~\ref{tab:assoc-bootmean}, \(\pmb{F}_1\) contributes to two selected biclusters, \((\pmb{S}_1,\pmb{F}_1)\) and \((\pmb{S}_3,\pmb{F}_1)\). Figure~\ref{fig:ibd_overlap_bicluster_mean} visualizes these biclusters using representative pathways from \(\pmb{F}_1\), showing diverging temporal trajectories between the two subject clusters. The mean curves for these \textit{Alistipes}-stratified pathways, representing functional activity levels rather than simple bacterial abundance, show a general decline over time in $\bm{S}_1$, while exhibiting an increasing trend in $\bm{S}_3$. This divergence indicates that $(\bm{S}_1, \bm{F}_1)$ and $(\bm{S}_3, \bm{F}_1)$ represent two distinct functional signatures involving the same microbial species. This longitudinal contrast aligns with the clinical phenotypic composition of the two subject groups. While both $S_1$ and $S_3$ contain similar proportions of subjects with Crohn’s Disease (CD), $S_1$ is enriched with non-IBD (control) individuals, whereas $S_3$ includes a significantly larger proportion of subjects with Ulcerative Colitis (UC). Our findings are consistent with existing literature reporting that \textit{Alistipes putredinis} is a hallmark of a healthy gut microbiome and is typically depleted in UC patients \citep{Nomura2021BacteroidetesUC,deMeij2018CoreMicrobiotaIBD}. Collectively, these results suggest that $F_1$ serves as an overlapping pathway group that distinguishes the relatively ``healthy'' profiles of $S_1$ from the UC-enriched profiles of $S_3$, highlighting that microbial functional signals are heterogeneous across different IBD subtypes.

For space considerations, additional interpretations of the biclusters involving \(\pmb{F}_2\) and \(\pmb{F}_3\) are provided in Section~\ref{supp:ibd_additional_bicluster_interpretation} of the Supplementary Material.

\begin{table}[h]
\centering
\begin{tabular}{cccc}
\hline
 & $\pmb{F}_1$ & $\pmb{F}_2$ & $\pmb{F}_3$ \\
\hline
$\pmb{S}_1$ & \pmb{0.705} (0.049) & 0.270 (0.053) & 0.263 (0.057) \\
$\pmb{S}_2$ & 0.164 (0.046) & \pmb{0.845} (0.029) & 0.357 (0.063) \\
$\pmb{S}_3$ & \pmb{0.689} (0.048) & \pmb{0.602} (0.051) & \pmb{0.553} (0.057) \\
\hline
\end{tabular}
\caption{Bootstrap mean (SD) of \(r_{jk}\) between sample clusters \(\{\pmb{S}_j\}_{j=1}^{3}\) and pathway groups \(\{\pmb{F}_k\}_{k=1}^{3}\). Bootstrap replicates were obtained by resampling subjects with replacement and recomputing \(r_{jk}\) in each replicate. Larger values indicate stronger dependence between a sample cluster and a pathway group, and therefore provide stronger evidence for interpreting that pair as a bicluster. The table indicates an overlapping association structure: \(\pmb{F}_1\) is strongly associated with both \(\pmb{S}_1\) and \(\pmb{S}_3\), \(\pmb{F}_2\) is most strongly associated with \(\pmb{S}_2\) but also shows a relatively strong association with \(\pmb{S}_3\), and \(\pmb{F}_3\) is most strongly associated with \(\pmb{S}_3\).}
\label{tab:assoc-bootmean}
\end{table}

\begin{figure}[t!]
  \centering
  \includegraphics[width=0.84\textwidth]{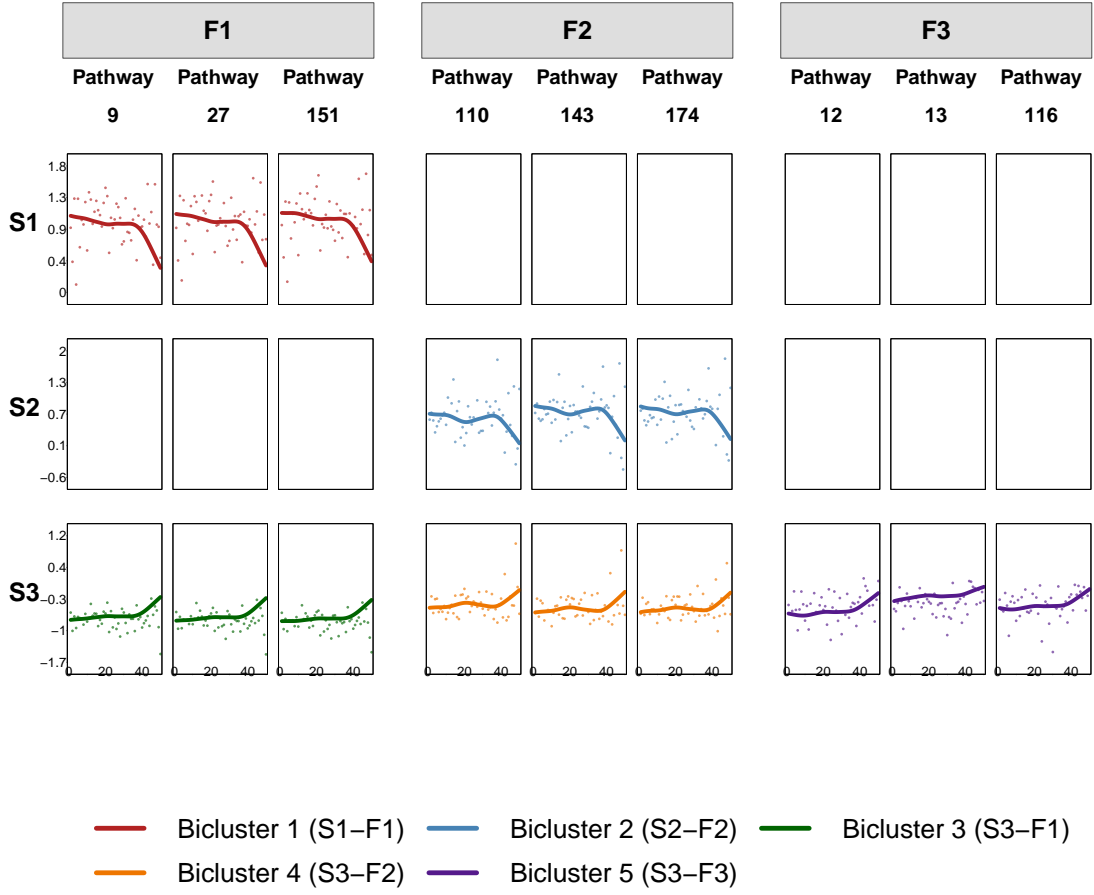}
  \caption{
Visual summary of the identified sample--pathway bicluster structure. Rows correspond to the sample clusters \(\pmb{S}_1,\pmb{S}_2,\pmb{S}_3\), and column groups correspond to the pathway groups \(\pmb{F}_1,\pmb{F}_2,\pmb{F}_3\), with three randomly selected pathways from each pathway group shown for illustration. Each \(\pmb{S}_i\)-\(\pmb{F}_j\) block represents one sample--pathway combination. Combinations identified as biclusters are shown as colored blocks, whereas combinations not identified are left blank. Within each colored block, colored points represent raw sample mean measurements for the selected pathway among samples in the corresponding sample cluster, computed after omitting missing observations at each time point, and colored solid curves represent the reconstructed mean trajectories from our method.
}
  \label{fig:ibd_overlap_bicluster_mean}
\end{figure}

\subsection{Comparison with Alternative Methods}
To assess the practical performance of the proposed method, we compared it with MFPCA, FPCA-BC, funLBM and funCC on the same IBD dataset. Detailed results are reported in section~\ref{supp:ibd_competing_methods} in the Supplementary Material. Briefly, MFPCA yielded sample clusters that were significantly associated with IBD type (\(P<0.001\)); however, because it does not perform pathway selection, it cannot identify interpretable pathway groups. FPCA-BC did not identify sample clusters with significant associations with either IBD type or sex in this application. Although funLBM identified meaningful sample clusters, it returned 15 feature clusters and a total of 45 biclusters, making the resulting structure less interpretable. funCC did not produce clearly meaningful bicluster structures. Overall, these comparisons suggest that the proposed method provides a more structured and biologically interpretable sample--pathway representation.

\section{Conclusion} 
\label{Conclusion}

We propose a triclustering framework for high-dimensional multivariate functional data to identify structured associations among samples, variables, and time. Built on a functional SVD formulation with sparse regularization, the method lets the data determine the underlying structure across samples, features, and time, rather than imposing a triclustering pattern a priori. The framework also accommodates sparse functional data with irregularly missing observations and is computationally scalable, since the optimization over functional loadings decomposes into independent subproblems that enable parallel implementation.

We illustrated the method through two real data applications. In the IBD dataset, it identified three pathway groups linked to sample clusters with distinct clinical and demographic profiles. In the EEG dataset, it identified three channel groups with distinct spatial and temporal patterns: one frontal channel group with signals over almost the full time domain and two posterior channel groups with signals concentrated in different localized temporal windows. These EEG patterns were associated with subject clusters enriched for alcoholic or control participants, suggesting that alcoholism-related electrophysiological differences may be reflected in region-specific and time-localized brain activity.

A computational limitation of the proposed method is that, as a regularized
framework with multiple tuning parameters, it requires careful specification of
candidate tuning ranges. In high-dimensional settings, evaluation over a large
number of tuning combinations can be computationally demanding and less efficient
than unregularized alternatives. Future work could therefore focus on developing
more scalable strategies for tuning-parameter selection.

Methodologically, another promising direction is to extend the proposed framework
to nonlinear functional SVD. For example, neural-network-based representations
could be incorporated to capture nonlinear functional patterns while retaining
sparse subject, feature, and time-subregion selection for interpretability. The
method could also be extended to multi-omics functional data by jointly modeling
multiple longitudinal omics layers, enabling the identification of shared or
omics-specific bicluster and tricluster structures across data modalities.

Overall, the proposed framework provides a flexible and interpretable approach for uncovering meaningful structures in complex functional data.


\section*{Funding and Acknowledgments}
The project was supported by the Award Numbers 1R35GM142695 and 1R35GM150537 of the National Institute of General Medical Sciences of the National Institutes of Health. The content is solely the responsibility of the authors and does not represent the official views of the National Institutes of Health.

\section*{Data Availability Statement}
The data used were obtained from https://ibdmdb.org. 
We provide an R package, \textit{TriSfsvd}, to facilitate the use of our method. Its source codes, along with a README file, are available at: https://github.com/YueZhao512/TriSfSVD.

\clearpage
\pagenumbering{arabic}
\renewcommand*{\thepage}{A\arabic{page}}
\appendix
\setcounter{section}{0}
\setcounter{subsection}{0}
\setcounter{subsubsection}{0}
\renewcommand{\thesection}{\Alph{section}}
\renewcommand{\thesubsection}{\thesection.\arabic{subsection}}
\renewcommand{\thesubsubsection}{\thesubsection.\arabic{subsubsection}}
\setcounter{equation}{0}
\renewcommand\theequation{\thesection.\arabic{equation}}
\setcounter{figure}{0}
\counterwithin{figure}{section}
\renewcommand\thefigure{\thesection.\arabic{figure}}
\setcounter{table}{0}
\renewcommand\thetable{\arabic{table}}

\section*{Supplementary Materials for \textit{Sparse Functional Singular Value Decomposition for Biclustering and Triclustering Longitudinal Data}}

\section{FISTA Details}
\label{algorithm}

\subsection{FISTA details for the \texorpdfstring{$\widetilde{\pmb{\varphi}}_j$}{varphi-j} update}

For the variable-specific scaled-loading subproblem in \eqref{best approximation step 2 matrix} of the main manuscript, we use a modified Fast Iterative Shrinkage-Thresholding Algorithm (FISTA) \citep{fista}. FISTA is applied to the scaled-loading subproblem for \(\widetilde{\pmb{\varphi}}_j=s\pmb{\varphi}_j\); the normalized loading direction is obtained after concatenating the variable-specific scaled loadings. Specifically, we solve
\begin{equation} \label{supp:best approximation step 2 fista}
    \min_{\widetilde{\pmb{\varphi}}_j}
    \left\{
    \Vert \tilde{\pmb{y}}_j - \pmb{U}_j\widetilde{\pmb{\varphi}}_j \Vert_2^2
    + \alpha_j
    \widetilde{\pmb{\varphi}}_j^\top\pmb{\Omega}_j
    \widetilde{\pmb{\varphi}}_j
    + \theta w_{2,j}\Vert\widetilde{\pmb{\varphi}}_j\Vert_2
    + \lambda_j
    \Vert \pmb{w}_{3,j}\odot\widetilde{\pmb{\varphi}}_j\Vert_1
    \right\}.
\end{equation}
We decompose the objective as \(f(\widetilde{\pmb{\varphi}}_j)+g(\widetilde{\pmb{\varphi}}_j)\), where
\[
f(\widetilde{\pmb{\varphi}}_j)
=
\Vert \tilde{\pmb{y}}_j
-
\pmb{U}_j\widetilde{\pmb{\varphi}}_j
\Vert_2^2
+ \alpha_j
\widetilde{\pmb{\varphi}}_j^\top
\pmb{\Omega}_j
\widetilde{\pmb{\varphi}}_j
\]
is continuously differentiable and convex, with gradient
\[
\nabla f(\widetilde{\pmb{\varphi}}_j)
=
-2\pmb{U}_j^\top
\left(
\tilde{\pmb{y}}_j
-
\pmb{U}_j\widetilde{\pmb{\varphi}}_j
\right)
+
2\alpha_j
\pmb{\Omega}_j
\widetilde{\pmb{\varphi}}_j .
\]
The nonsmooth component is
\[
g(\widetilde{\pmb{\varphi}}_j)
=
\theta w_{2,j}\Vert\widetilde{\pmb{\varphi}}_j\Vert_2
+
\lambda_j
\Vert \pmb{w}_{3,j}\odot\widetilde{\pmb{\varphi}}_j\Vert_1,
\]
which contains the variable-level group penalty and the temporal adaptive-lasso penalty. Given the current iterate \(\widetilde{\pmb{\varphi}}_j\), FISTA forms the gradient-descent point
\[
\pmb{z}=\widetilde{\pmb{\varphi}}_j-\eta \nabla f(\widetilde{\pmb{\varphi}}_j),
\]
where \(\eta>0\) is the step size, and updates
\[
\widetilde{\pmb{\varphi}}_j^{\mathrm{new}}=\operatorname{prox}_{\eta g}(\pmb{z}).
\]

\begin{lemma}[Proximal operator for the Tri-SfSVD \(\widetilde{\pmb{\varphi}}_j\)-update]
\label{lem:proximal_operator_tricluster_supp}
For
\[
g(\widetilde{\pmb{\varphi}}_j)
=
\theta w_{2,j}\Vert \widetilde{\pmb{\varphi}}_j\Vert_2
+ \lambda_j \Vert \pmb{w}_{3,j}\odot\widetilde{\pmb{\varphi}}_j\Vert_1,
\]
the proximal operator
\[
\operatorname{prox}_{\eta g}(\pmb{z})
=
\arg\min_{\widetilde{\pmb{\varphi}}_j}
\left\{
\frac{1}{2} \Vert \widetilde{\pmb{\varphi}}_j - \pmb{z} \Vert_2^2
+ \eta g(\widetilde{\pmb{\varphi}}_j)
\right\}
\]
has the closed-form expression
\begin{equation}
\operatorname{prox}_{\eta g}(\pmb{z})
=
\left(
1 - \frac{\eta \theta w_{2,j}}
{\Vert S(\pmb{z}, \eta \lambda_j \pmb{w}_{3,j}) \Vert_2}
\right)_{+}
S(\pmb{z}, \eta \lambda_j \pmb{w}_{3,j}),
\end{equation}
where \(S(\pmb{z},\pmb{a})\) is the component-wise soft-thresholding operator,
\[
S(\pmb{z},\pmb{a})_\ell
=
\operatorname{sign}(z_\ell)(|z_\ell|-a_\ell)_+.
\]
If \(S(\pmb{z}, \eta \lambda_j \pmb{w}_{3,j})=\pmb{0}\), then \(\operatorname{prox}_{\eta g}(\pmb{z})=\pmb{0}\).
\end{lemma}

The proof follows from the proximal mapping of a sparse-group penalty and is included here to keep the derivation of the FISTA update self-contained.
\begin{proof}[Proof of Lemma~\ref{lem:proximal_operator_tricluster_supp}]
In FISTA, at each iteration with Lipschitz constant \(L > 0\), we form the quadratic approximation
\[
M(\widetilde{\pmb{\varphi}}_j)
=
f(\widetilde{\pmb{\varphi}}_0)
+
\langle \widetilde{\pmb{\varphi}}_j - \widetilde{\pmb{\varphi}}_0, \nabla f(\widetilde{\pmb{\varphi}}_0) \rangle
+
\frac{L}{2} \Vert \widetilde{\pmb{\varphi}}_j - \widetilde{\pmb{\varphi}}_0 \Vert_2^2
+
\theta w_{2,j} \Vert \widetilde{\pmb{\varphi}}_j \Vert_2
+
\lambda_j \Vert \pmb{w}_{3,j} \odot \widetilde{\pmb{\varphi}}_j \Vert_1
\]
at current point \(\widetilde{\pmb{\varphi}}_0\). Minimizing \(M(\cdot)\) is equivalent to
\[
\tilde{M}(\widetilde{\pmb{\varphi}}_j)
=
\frac{1}{2\eta} \Vert \widetilde{\pmb{\varphi}}_j - \pmb{z} \Vert_2^2
+
\theta w_{2,j} \Vert \widetilde{\pmb{\varphi}}_j \Vert_2
+
\lambda_j \Vert \pmb{w}_{3,j} \odot \widetilde{\pmb{\varphi}}_j \Vert_1,
\]
where \(\eta = 1/L\) is the step size and
\[
\pmb{z} = \widetilde{\pmb{\varphi}}_0 - \eta \nabla f(\widetilde{\pmb{\varphi}}_0).
\]
Define the nonsmooth penalty term as
\[
g(\widetilde{\pmb{\varphi}}_j)
=
\theta w_{2,j} \Vert \widetilde{\pmb{\varphi}}_j \Vert_2
+
\lambda_j \Vert \pmb{w}_{3,j} \odot \widetilde{\pmb{\varphi}}_j \Vert_1 .
\]
Then minimizing \(\tilde{M}(\widetilde{\pmb{\varphi}}_j)\) is equivalent, after multiplying
the objective by the positive constant \(\eta\), to solving
\[
\arg\min_{\widetilde{\pmb{\varphi}}_j}
\left\{
\frac{1}{2} \Vert \widetilde{\pmb{\varphi}}_j - \pmb{z} \Vert_2^2
+
\eta g(\widetilde{\pmb{\varphi}}_j)
\right\}.
\]
Therefore, this minimizer is the proximal operator of \(\eta g\) at \(\pmb{z}\):
\[
\operatorname{prox}_{\eta g}(\pmb{z})
=
\arg\min_{\widetilde{\pmb{\varphi}}_j}
\left\{
\frac{1}{2} \Vert \widetilde{\pmb{\varphi}}_j - \pmb{z} \Vert_2^2
+
\eta \theta w_{2,j} \Vert \widetilde{\pmb{\varphi}}_j \Vert_2
+
\eta \lambda_j \Vert \pmb{w}_{3,j} \odot \widetilde{\pmb{\varphi}}_j \Vert_1
\right\}.
\]

Following \citep{sgl, sglnote}, the subgradient equation for the minimizer \(\widetilde{\pmb{\varphi}}_j\) is
\[
\pmb{0}
\in
\widetilde{\pmb{\varphi}}_j - \pmb{z}
+
\eta \theta w_{2,j} \partial \Vert \widetilde{\pmb{\varphi}}_j \Vert_2
+
\eta \lambda_j \partial \Vert \pmb{w}_{3,j} \odot \widetilde{\pmb{\varphi}}_j \Vert_1,
\]
where
\[
\partial \Vert \widetilde{\pmb{\varphi}}_j \Vert_2
=
\left\{
\pmb{w} :
\pmb{w}=\frac{\widetilde{\pmb{\varphi}}_j}{\Vert \widetilde{\pmb{\varphi}}_j \Vert_2}
\text{ if } \widetilde{\pmb{\varphi}}_j \neq \pmb{0},
\;
\Vert \pmb{w} \Vert_2 \le 1
\text{ if } \widetilde{\pmb{\varphi}}_j = \pmb{0}
\right\},
\]
and
\[
\partial \Vert \pmb{w}_{3,j} \odot \widetilde{\pmb{\varphi}}_j \Vert_1
=
\left\{
\pmb{w}_{3,j} \odot \pmb{v} :
v_i = \mathrm{sign}(\widetilde{\varphi}_{j,i})
\text{ if } \widetilde{\varphi}_{j,i} \neq 0,
\;
|v_i| \le 1
\text{ if } \widetilde{\varphi}_{j,i} = 0
\right\}.
\]
Thus, the subgradient equation can be rewritten as
\[
\pmb{0}
\in
\widetilde{\pmb{\varphi}}_j - \pmb{z}
+
\eta \theta w_{2,j} \pmb{w}
+
\eta \lambda_j (\pmb{w}_{3,j} \odot \pmb{v}).
\]

The necessary condition for \(\widetilde{\pmb{\varphi}}_j = \pmb{0}\) to be the minimizer is that there exist \(\pmb{w}\) and \(\pmb{v}\) satisfying the subgradient equation with \(\Vert \pmb{w} \Vert_2 \le 1\) and \(|v_i| \le 1\) for all \(i\). This is equivalent to checking whether the minimum of
\[
J(\pmb{v})
=
\frac{1}{(\eta \theta w_{2,j})^2}
\Vert
\pmb{z} - \eta \lambda_j (\pmb{w}_{3,j} \odot \pmb{v})
\Vert_2^2
=
\Vert \pmb{w} \Vert_2^2,
\]
subject to \(|v_i| \le 1\) for all \(i\), is at most \(1\). The minimizer \(\pmb{v}\) is obtained component-wise as
\[
v_i =
\begin{cases}
\dfrac{z_i}{\eta \lambda_j w_{3,j,i}}, & \text{if } \left| \dfrac{z_i}{\eta \lambda_j w_{3,j,i}} \right| \le 1, \\[2ex]
\mathrm{sign}\left( \dfrac{z_i}{\eta \lambda_j w_{3,j,i}} \right), & \text{if } \left| \dfrac{z_i}{\eta \lambda_j w_{3,j,i}} \right| > 1.
\end{cases}
\]
This \(\pmb{v}\) aligns with the soft-thresholding operator, such that
\[
\pmb{z} - \eta \lambda_j (\pmb{w}_{3,j} \odot \pmb{v})
=
S(\pmb{z}, \eta \lambda_j \pmb{w}_{3,j}),
\]
where
\[
S(\pmb{z}, \eta \lambda_j \pmb{w}_{3,j})
=
\Bigl(
\mathrm{sign}(z_i)\bigl(|z_i| - \eta \lambda_j w_{3,j,i}\bigr)_+
\Bigr)_{i=1}^{d_j}.
\]
Thus, \(\widetilde{\pmb{\varphi}}_j = \pmb{0}\) if
\[
\Vert S(\pmb{z}, \eta \lambda_j \pmb{w}_{3,j}) \Vert_2
\le
\eta \theta w_{2,j}.
\]

Otherwise, for \(\widetilde{\pmb{\varphi}}_j \neq \pmb{0}\), the subgradient equation simplifies to
\[
\left(
1 + \frac{\eta \theta w_{2,j}}{\Vert \widetilde{\pmb{\varphi}}_j \Vert_2}
\right)\widetilde{\pmb{\varphi}}_j
=
S(\pmb{z}, \eta \lambda_j \pmb{w}_{3,j}),
\]
since \(\pmb{w} = \widetilde{\pmb{\varphi}}_j / \Vert \widetilde{\pmb{\varphi}}_j \Vert_2\) and \(\pmb{v}\) corresponds to the subgradient for the \(\ell_1\) term post-soft-thresholding. Taking the \(\ell_2\)-norm of both sides yields
\[
\Vert \widetilde{\pmb{\varphi}}_j \Vert_2
=
\Vert S(\pmb{z}, \eta \lambda_j \pmb{w}_{3,j}) \Vert_2
-
\eta \theta w_{2,j}.
\]
Substituting back gives
\[
\widetilde{\pmb{\varphi}}_j
=
\left(
1 - \frac{\eta \theta w_{2,j}}
{\Vert S(\pmb{z}, \eta \lambda_j \pmb{w}_{3,j}) \Vert_2}
\right)
S(\pmb{z}, \eta \lambda_j \pmb{w}_{3,j}).
\]
Combining the zero and non-zero cases, the proximal operator is
\[
\operatorname{prox}_{\eta g}(\pmb{z})
=
\left(
1 - \frac{\eta \theta w_{2,j}}
{\Vert S(\pmb{z}, \eta \lambda_j \pmb{w}_{3,j}) \Vert_2}
\right)_{+}
S(\pmb{z}, \eta \lambda_j \pmb{w}_{3,j}).
\]
When \(S(\pmb{z}, \eta \lambda_j \pmb{w}_{3,j}) = \pmb{0}\), define
\(\operatorname{prox}_{\eta g}(\pmb{z}) = \pmb{0}\). 

This completes the proof.
\end{proof}

When \(\lambda_j=0\), the triclustering penalty reduces to the biclustering penalty \(g(\widetilde{\pmb{\varphi}}_j)=\theta w_{2,j}\Vert \widetilde{\pmb{\varphi}}_j\Vert_2\), and the proximal operator becomes the standard group-lasso thresholding update:
\[
\operatorname{prox}_{\eta g}(\pmb{z})
=
\left(1-\frac{\eta\theta w_{2,j}}{\Vert \pmb{z}\Vert_2}\right)_+\pmb{z}.
\]
The derivation of this biclustering special case is as follows.
\begin{proof}[Proof of Bicluster Proximal Operator]
In FISTA, at each iteration with Lipschitz constant \(L > 0\), we form the quadratic approximation
\[
M(\widetilde{\pmb{\varphi}}_j) = f(\widetilde{\pmb{\varphi}}_0) + \langle \widetilde{\pmb{\varphi}}_j - \widetilde{\pmb{\varphi}}_0, \nabla f(\widetilde{\pmb{\varphi}}_0) \rangle + \frac{L}{2} \Vert \widetilde{\pmb{\varphi}}_j - \widetilde{\pmb{\varphi}}_0 \Vert_2^2 + \theta w_{2,j} \Vert \widetilde{\pmb{\varphi}}_j \Vert_2
\]
at current point \(\widetilde{\pmb{\varphi}}_0\). Minimizing \(M(\cdot)\) is equivalent to
\[
\tilde{M}(\widetilde{\pmb{\varphi}}_j) = \frac{1}{2\eta} \Vert \widetilde{\pmb{\varphi}}_j - \pmb{z} \Vert_2^2 + \theta w_{2,j} \Vert \widetilde{\pmb{\varphi}}_j \Vert_2,
\]
where \(\eta = 1/L\) is the step size and \(\pmb{z} = \widetilde{\pmb{\varphi}}_0 - \eta \nabla f(\widetilde{\pmb{\varphi}}_0)\). The proximal operator is thus
\[
\text{prox}_{\eta g}(\pmb{z}) = \arg\min_{\widetilde{\pmb{\varphi}}_j} \frac{1}{2} \Vert \widetilde{\pmb{\varphi}}_j - \pmb{z} \Vert_2^2 + \eta \theta w_{2,j} \Vert \widetilde{\pmb{\varphi}}_j \Vert_2,
\]
where \(g(\widetilde{\pmb{\varphi}}_j) = \theta w_{2,j} \Vert \widetilde{\pmb{\varphi}}_j \Vert_2\).
The subgradient condition at the minimizer \(\widetilde{\pmb{\varphi}}_j\) is
\begin{equation}\label{bicluster-subgradient}
\mathbf{0} \in \widetilde{\pmb{\varphi}}_j - \pmb{z} + \eta \theta w_{2,j} \pmb{w},
\end{equation}
where \(\pmb{w}\) is a subgradient of \(\Vert \widetilde{\pmb{\varphi}}_j \Vert_2\):
\[
\pmb{w} =
\begin{cases}
\frac{\widetilde{\pmb{\varphi}}_j}{\Vert \widetilde{\pmb{\varphi}}_j \Vert_2} & \text{if } \widetilde{\pmb{\varphi}}_j \neq \mathbf{0}, \\
\in \{ \pmb{v} : \Vert \pmb{v} \Vert_2 \leq 1 \} & \text{if } \widetilde{\pmb{\varphi}}_j = \mathbf{0}.
\end{cases}
\]
The condition \(\widetilde{\pmb{\varphi}}_j = \mathbf{0}\) holds if \(\Vert \pmb{z} \Vert_2 \leq \eta \theta w_{2,j}\). Otherwise, for \(\widetilde{\pmb{\varphi}}_j \neq \mathbf{0}\),
\begin{equation}\label{bicluster-subgradient_no0}
\pmb{z} = \widetilde{\pmb{\varphi}}_j \left( 1 + \frac{\eta \theta w_{2,j}}{\Vert \widetilde{\pmb{\varphi}}_j \Vert_2} \right).
\end{equation}
Taking the \(\ell_2\)-norm on both sides yields
\[
\Vert \widetilde{\pmb{\varphi}}_j \Vert_2 = \Vert \pmb{z} \Vert_2 - \eta \theta w_{2,j}.
\]
Substituting back into \eqref{bicluster-subgradient_no0} gives
\[
\widetilde{\pmb{\varphi}}_j = \left( 1 - \frac{\eta \theta w_{2,j}}{\Vert \pmb{z} \Vert_2} \right) \pmb{z}.
\]
Combining both cases completes the proof.
\end{proof}

With the proximal updates for both Tri-SfSVD and Bi-SfSVD established above, Algorithm~\ref{alg:fista_sga_lasso} summarizes the resulting modified FISTA update with backtracking. Step 2 displays the triclustering update and its biclustering special case separately.

\LinesNotNumbered
\begin{algorithm}[H]
\caption{FISTA with Backtracking}
\label{alg:fista_sga_lasso}
\textbf{Initialization:} Set $\widetilde{\pmb{\varphi}}_j^{(0)} = \mathbf{v}^{(1)} = \mathbf{0}$ and $t^{(1)} = 1$.
Given $L > 0$, $\tau > 1$, and tolerance $\varepsilon$.
\medskip
\textbf{Repeat until convergence:}
\begin{enumerate}
    \item Set $\eta = 1/L$, and compute:
    \[
    \mathbf{z}^{(k)} = \mathbf{v}^{(k)} - \eta \nabla f(\mathbf{v}^{(k)}).
    \]
    \item Compute the proximal update using one of the following two cases:
    \[
    \begin{array}{@{}l@{\quad}l@{}}
    \text{Tri-SfSVD:} &
    \widetilde{\pmb{\varphi}}_j^{(k)}
    =
    \left(1 - \frac{\eta \theta w_{2,j}}
    {\Vert S(\mathbf{z}^{(k)}, \eta \lambda_j \pmb{w}_{3,j}) \Vert_2}\right)_+
    S(\mathbf{z}^{(k)}, \eta \lambda_j \pmb{w}_{3,j}),\\[1.5ex]
    \text{Bi-SfSVD }(\lambda_j=0): &
    \widetilde{\pmb{\varphi}}_j^{(k)}
    =
    \left(1 - \frac{\eta \theta w_{2,j}}
    {\Vert \mathbf{z}^{(k)} \Vert_2}\right)_+ \mathbf{z}^{(k)} .
    \end{array}
    \]
    For the Tri-SfSVD update, set \(\widetilde{\pmb{\varphi}}_j^{(k)}=\pmb{0}\) when
    \(S(\mathbf{z}^{(k)}, \eta \lambda_j \pmb{w}_{3,j})=\pmb{0}\).
    \item If $f(\widetilde{\pmb{\varphi}}_j^{(k)}) > f(\mathbf{v}^{(k)}) + \nabla f(\mathbf{v}^{(k)})^\top (\widetilde{\pmb{\varphi}}_j^{(k)}-\mathbf{v}^{(k)}) + \frac{L}{2}\Vert\widetilde{\pmb{\varphi}}_j^{(k)} - \mathbf{v}^{(k)}\Vert_2^2$, then set $L \leftarrow \tau L$ and repeat from step 1.
   
    \item Momentum Update:
    \[
    \mathbf{v}^{(k+1)} = \widetilde{\pmb{\varphi}}_j^{(k)} + \frac{t^{(k)} - 1}{t^{(k+1)}}(\widetilde{\pmb{\varphi}}_j^{(k)} - \widetilde{\pmb{\varphi}}_j^{(k-1)}), \text{ where } t^{(k+1)} = \frac{1 + \sqrt{1 + 4(t^{(k)})^2}}{2}.
    \]
    \item Stop if $\Vert\widetilde{\pmb{\varphi}}_j^{(k)} - \widetilde{\pmb{\varphi}}_j^{(k-1)}\Vert_2 \le \varepsilon$.
\end{enumerate}
\textbf{Output: } \(\widetilde{\pmb{\varphi}}_j\).
\end{algorithm}

\newpage
To connect the variable-specific FISTA update with the full estimation scheme, Algorithm~\ref{alg:sfpca_rank1} records the basic alternating procedure used when the tuning parameters are held fixed. This fixed-tuning version serves as the computational template for the conditional one-dimensional tuning searches described in Section~\ref{Tuning Selection}; the complete EBIC-based procedure is summarized later in Algorithm~\ref{alg:complete_supp}.

\begin{algorithm}[H]
\caption{Iterative method without tuning selection}
\label{alg:sfpca_rank1}

\textbf{1.~Initialize} ${\pmb{u}_i}, \{\pmb{\varphi}_j\}_{j=1}^{p}$

\vspace{1em}
\textbf{2.~Repeat until convergence:}

\begin{enumerate}
  \item[(a)] \textbf{$\widetilde{\pmb{u}}$-subproblem}:
    \[
      \widetilde u_i =
      \frac{
      \operatorname{sign}\left( {\pmb{\varphi}^{*}_{i\cdot}}^\top \pmb{y}_i \right)
      }{
      {\pmb{\varphi}^{*}_{i\cdot}}^{\top} \pmb{\varphi}^{*}_{i\cdot}
      }
      \left(
      \left| {\pmb{\varphi}^{*}_{i\cdot}}^\top \pmb{y}_i \right|
      - \frac{1}{2}\gamma w_{1,i}
      \right)_+.
    \]
    Then normalize
    \[
      \pmb u\leftarrow\widetilde{\pmb u}/\Vert\widetilde{\pmb u}\Vert_2.
    \]

  \item[(b)] \textbf{$\widetilde{\pmb{\varphi}}_j$-subproblem}:
    \[
      \text{For each variable } j,\text{ apply Algorithm}~\ref{alg:fista_sga_lasso}
      \text{ to obtain } \widetilde{\pmb{\varphi}}_j.
    \]
    \[
      \text{Concatenate } \widetilde{\pmb{\varphi}}=(\widetilde{\pmb{\varphi}}_1^\top,\ldots,\widetilde{\pmb{\varphi}}_p^\top)^\top,\quad
      \pmb{\varphi}\leftarrow\widetilde{\pmb{\varphi}}/\Vert\widetilde{\pmb{\varphi}}\Vert_2.
    \]
\end{enumerate}

\vspace{1em}
\textbf{3.~Final scale:}
\[
{s}
\leftarrow
\frac{
\sum_{i,j,k}
y_{i,j,k}u_i\varphi_j(t_{i,j,k})
}{
\sum_{i,j,k}
u_i^2\varphi_j(t_{i,j,k})^2
}.
\]
\textbf{Return } \({s},\ \pmb u,\ \{\pmb\varphi_j\}_{j=1}^p\).
\end{algorithm}

\clearpage
\section{More on Methods}

\subsection{Proofs}
\label{supp_proofs}
\begin{proof} [Proof of Lemma~\ref{lem:impl_bi_explicit}]
Under the scaled rank-one formulation in
\eqref{multivariate functional SVD}, the rank-one operator is
\(s(\pmb u\otimes\pmb\phi)\), where \(s\ge0\),
\(\Vert\pmb u\Vert_2=1\), and \(\Vert\pmb\phi\Vert_{\mathbb H}=1\). For
\(\pmb a=(a_1,\ldots,a_n)^\top\in\mathbb R^n\),
\[
s(\pmb u\otimes\pmb\phi)(\pmb a)
=
s\langle \pmb a,\pmb u\rangle_{\mathbb R^n}\pmb\phi
=
\sum_{i=1}^n a_i s u_i\,\pmb\phi .
\]
Thus \(s(\pmb u\otimes\pmb\phi)\in\mathbb F^{p\times n}\) is represented by
\([s u_i\pmb\phi]_{i=1}^n\). Therefore,
\begin{align*}
\left\Vert
\mathbfcal X-s(\pmb u\otimes\pmb\phi)
\right\Vert_{\mathbb{F}}^2
&=
\sum_{i=1}^n
\left\langle
\pmb x_i-s u_i\pmb\phi,\,
\pmb x_i-s u_i\pmb\phi
\right\rangle_{\mathbb{H}} \\
&=
\sum_{j=1}^p\sum_{i=1}^n
\left\Vert
x_{ij}-s u_i\phi_j
\right\Vert_{H_j}^{2}.
\end{align*}

In practice, the functional trajectories \(x_{ij}(\cdot)\) are observed on
discrete grids \(\{t_{j1},\ldots,t_{j d_j}\}\subset\mathcal T_j\). Let
\(\pmb Y_j\in\mathbb R^{n\times d_j}\) collect the sampled values
\(x_{ij}(t_{j\ell})\), and let
\[
\pmb\varphi_j
=
\left(
\phi_j(t_{j1}),\ldots,\phi_j(t_{j d_j})
\right)^\top .
\]
Then, up to the usual grid-weighting convention,
\[
\sum_{i=1}^n
\left\Vert
x_{ij}-s u_i\phi_j
\right\Vert_{H_j}^{2}
\approx
\left\lVert
\pmb Y_j-s\pmb u\pmb\varphi_j^\top
\right\rVert_F^2 .
\]

To enforce smoothness of the loading functions, we penalize the roughness of
the scaled loading functions \(s\phi_j\). In the discretized representation,
this yields
\[
\sum_{j=1}^p
\alpha_j
(s\pmb\varphi_j)^\top
\pmb\Omega_j
(s\pmb\varphi_j),
\]
where \(\pmb\Omega_j\) is a non-negative definite roughness penalty matrix and
\(\alpha_j\ge0\) controls the smoothness strength for variable \(j\). Combining
the discretized reconstruction loss and the roughness penalty, together with
the constraints \(s\ge0\), \(\Vert\pmb u\Vert_2=1\), and
\(\Vert\pmb\varphi\Vert_2=1\), gives \eqref{eq:bi_impl_explicit}.
\end{proof}

\subsection{Tuning Parameter Selection}
\label{Tuning Selection}

The tuning parameters are selected by minimizing an extended Bayesian information criterion (EBIC) \cite{chen2008ebic}. The EBIC extends the classical BIC with an additional penalty term on model complexity, weighted by a parameter $\sigma \in [0,1]$. When $\sigma = 0$, the EBIC reduces to the standard BIC. The resulting selection rules are incorporated into the complete procedure summarized below in Algorithm~\ref{alg:complete_supp}.

In the first step of our method, with the normalized loading directions $\{\pmb{\varphi}_j\}_{j=1}^{p}$ fixed, the EBIC is defined as 
\begin{equation} \label{ebic gamma}
\mathrm{EBIC(\gamma)}
= N \log\left(\frac{\mathrm{RSS}}{N}\right)
+ \mathrm{df}(\gamma) \log (N)
+ 2 \sigma\, \mathrm{df}(\gamma) \log n,
\end{equation}
where $N$ denotes the number of observed data points in the dataset and RSS is the residual sum of squares between the observed data and their corresponding fitted values.
The effective degrees of freedom for the subject-score update is taken as the number of active scaled subject scores,
\[
\mathrm{df}(\gamma)
=
\sum_{i=1}^{n}\mathbb{I}(\widehat{\widetilde u}_i\neq 0),
\]
where \(\widehat{\widetilde{\pmb u}}\) is obtained from the \(\widetilde{\pmb u}\)-update in Algorithm~\ref{alg:sfpca_rank1}.

In the second step of our method, with the normalized $\pmb{u}$ fixed, three types of tuning parameters,  $\{\alpha_j\}_{j=1}^{p}$, $\theta$, and $\{\lambda_j\}_{j=1}^{p}$, are introduced simultaneously for the scaled-loading subproblems. A joint optimization over a multi-dimensional tuning grid would be computationally intensive. To alleviate this burden, we adopt a conditional, sequential tuning strategy that updates the three types of tuning parameters one at a time. Specifically, when tuning $\{\alpha_j\}_{j=1}^{p}$, $\theta$, or $\{\lambda_j\}_{j=1}^{p}$, we hold the other two types of tuning parameters fixed at their current values and optimize the corresponding conditional criterion. We initialize $\{\alpha_j\}_{j=1}^{p}$, $\theta$, and $\{\lambda_j\}_{j=1}^{p}$ at the midpoint of their respective candidate ranges. After selecting an updated value for the tuning parameter currently under consideration, we carry it forward as the fixed value in subsequent tuning steps. We repeat this procedure until $\{\alpha_j\}_{j=1}^{p}$, $\theta$, and $\{\lambda_j\}_{j=1}^{p}$ have each been updated. Since the scaled \(\pmb{\varphi}\)-update decouples over variables, we select the smoothing parameter $\alpha_j$ independently for each variable by minimizing
\begin{equation} \label{ebic alpha}
\mathrm{EBIC(\alpha_j)}
= N_j \log\left(\frac{\mathrm{RSS}_j}{N_j}\right)
+ \mathrm{df}(\alpha_j) \log (N_j)
+ 2 \sigma\, \mathrm{df}(\alpha_j) \log d_j,
\end{equation}
where  $N_j$  denotes the number of observed data points for variable $j$, and $\mathrm{RSS}_j$ is the residual sum of squares between the observed values and their corresponding fitted values for that variable. Here 
\[\mathrm{df}(\alpha_j) = \text{tr}\left( \pmb{U}_{A_j} \left( \pmb{U}_{A_j}^{\top} \pmb{U}_{A_j} + \alpha_j \pmb{\Omega}_j \right)^{-1} \pmb{U}_{A_j}^{\top} \right),\] where $A_j$ denotes the active set of nonzero coefficients in the scaled loading vector \(\widetilde{\pmb{\varphi}}_j\), and $\pmb{U}_{A_j}$ is the corresponding submatrix of $\pmb{U}_j$ that keeps only the columns indexed by  $A_j$. This one-dimensional search is nested within Algorithm~\ref{alg:sfpca_rank1}, while keeping $\theta$ and $\{\lambda_j\}_{j=1}^{p}$ fixed at values obtained from an initial tuning stage.

Next, given the selected $\{\alpha_j\}_{j=1}^{p}$, we update the sparsity parameters 
$\{\lambda_j\}_{j=1}^{p}$ by minimizing, for each variable $j$,
\begin{equation} \label{ebic lambda}
\mathrm{EBIC}(\lambda_j)
= N_j \log\left(\frac{\mathrm{RSS}_j}{N_j}\right)
+ \mathrm{df}(\lambda_j)\,\log (N_j)
+ 2 \sigma\, \mathrm{df}(\lambda_j)\,\log d_j,
\end{equation}
where $\mathrm{df}(\lambda_j)$ is taken to be the number of nonzero elements in the scaled loading vector \(\widetilde{\pmb{\varphi}}_j\). This one-dimensional search is again nested within Algorithm~\ref{alg:sfpca_rank1} with $\theta$ kept fixed at its value from the initial tuning stage. 

Finally, with $\{\alpha_j\}_{j=1}^{p}$ 
and $\{\lambda_j\}_{j=1}^{p}$ fixed at their selected values, we choose the global tuning 
parameter $\theta$ by minimizing
\begin{equation} \label{ebic theta}
\mathrm{EBIC}(\theta)
= \sum_{j=1}^{p} \left[ N_j \log\left(\frac{\mathrm{RSS}_j}{N_j}\right)
+ \mathrm{df}(\theta)\,\log (N_j)
+ 2 \sigma\, \mathrm{df}(\theta)\,\log d_j\right],
\end{equation} 
where $\mathrm{df}(\theta)$ is given by proposition 1 in \citet{dfsgl} as
\[
\text{df}(\theta) = \sum_{j=1}^p \text{tr}\left( \pmb{U}_{A_j} \left( \pmb{U}_{A_j}^{\top} \pmb{U}_{A_j} + \theta \pmb{K}_j \right)^{-1} \pmb{U}_{A_j}^{\top} \right),
\]
where $\pmb{K}_j = \frac{1}{\Vert\widetilde{\pmb{\varphi}}_{\mathcal{A}_j}\Vert_2}
\left(
  \mathbf{I}
  - \frac{\widetilde{\pmb{\varphi}}_{\mathcal{A}_j}\,\widetilde{\pmb{\varphi}}_{\mathcal{A}_j}^\top}
         {\Vert\widetilde{\pmb{\varphi}}_{\mathcal{A}_j}\Vert_2^2}
\right)$ and \(\widetilde{\pmb{\varphi}}_{\mathcal{A}_j}\) is the subvector of the scaled loading vector \(\widetilde{\pmb{\varphi}}_j\) containing only the coefficients indexed by $\mathcal{A}_j$.

\subsection{Bi-SfSVD/Tri-SfSVD with EBIC-based tuning}
\label{supp_complete_algorithm}

Combining the alternating updates with the EBIC criteria above gives the full estimation procedure. In this algorithm, \(\gamma\) and \(\theta\) are selected as global tuning parameters, whereas \(\alpha_j\) and \(\lambda_j\) are selected within the variable-specific scaled-loading subproblems. This organization avoids an exhaustive joint search over all tuning parameters while preserving the update structure described in Section~\ref{algorithm}.

\begin{algorithm}[H] \LinesNotNumbered \caption{Bi-SfSVD/Tri-SfSVD with EBIC-based tuning} \label{alg:complete_supp} 
\textbf{Input:} Observed data \(\{\pmb{Y}_j\}_{j=1}^p\); candidate grids for \(\gamma\), \(\{\alpha_j\}_{j=1}^p\), \(\theta\), and \(\{\lambda_j\}_{j=1}^p\).\\ 
\textbf{Initialize:} \(\pmb{u}\) and \(\{\pmb{\varphi}_j\}_{j=1}^{p}\).\\ 
\textbf{Repeat until convergence:}\\ \quad 
\textbf{(a) \(\widetilde{\pmb{u}}\)-update:} update \(\widetilde{\pmb u}\) componentwise using the soft-thresholding solution in
\eqref{u solution}, and select $\gamma$ from the candidate grid by minimizing
$\mathrm{EBIC}(\gamma)$ defined in \eqref{ebic gamma}. Set
\[
\pmb u\leftarrow\widetilde{\pmb u}/\Vert\widetilde{\pmb u}\Vert_2.
\]
\quad 
\textbf{(b) \(\widetilde{\pmb{\varphi}}\)-update:} for $j=1,\ldots,p$, update \(\widetilde{\pmb{\varphi}}_j\) independently using Algorithm~\ref{alg:fista_sga_lasso} in Section~\ref{algorithm}. The tuning parameters $\{\alpha_j\}_{j=1}^{p}$, $\{\lambda_j\}_{j=1}^{p}$, and $\theta$ are selected over their respective candidate grids by minimizing the EBIC criteria defined in \eqref{ebic alpha}, \eqref{ebic lambda}, and \eqref{ebic theta}, respectively. Concatenate \(\widetilde{\pmb{\varphi}}=(\widetilde{\pmb{\varphi}}_1^\top,\ldots,\widetilde{\pmb{\varphi}}_p^\top)^\top\) and set
\[
\pmb\varphi\leftarrow\widetilde{\pmb\varphi}/\Vert\widetilde{\pmb\varphi}\Vert_2.
\]
\textbf{Final scale:} after convergence, compute
\[
{s}\leftarrow
\frac{
\sum_{i,j,k}
y_{i,j,k} u_i\varphi_j(t_{i,j,k})
}{
\sum_{i,j,k}
 u_i^2\varphi_j(t_{i,j,k})^2
}.
\]
\textbf{Output and deflation:} return \(({s},{\pmb u},\{{\pmb\varphi}_j\}_{j=1}^p)\); subtract the fitted rank-one component \({s} u_i\varphi_j(t_{i,j,k})\) from the observed data and repeat the above steps to extract subsequent sparse layers.
\end{algorithm}

\subsection{Selection of the Number of Clusters $K$}
\label{Selection of the Number of Clusters}
In this section, we discuss three strategies to guide the selection of the number of clusters, denoted by $K$.

First, we select \(K\) using a sequential stopping rule. We estimate sparse layers one at a time and stop when the newly estimated layer is empty, meaning that either the estimated sample scores \({\pmb{u}}^{(K)}\) or the estimated feature loadings \({\pmb{\varphi}}^{(K)}\) are all zero.

Second, we use the singular-value spectrum. Since each cluster corresponds to a rank-one component, we choose \(K\) at the elbow of the cumulative explained variance (CEV) curve,
\[
\mathrm{CEV}(K)=\frac{\sum_{r=1}^{k} d_r^{2}}{\sum_{r=1}^{K} d_r^{2}},
\]
and take the smallest \(K\) beyond which the incremental explained variance becomes negligible.

Third, we employ the BIC to select \(K\), defined as
\[
\text{BIC}(K)
=
\log\left(\tfrac{1}{N}\,\lVert \mathbf{X} -  \hat{\mathbf{X}}_K\rVert_2^2\right)
+
\frac{\log(N)}{N}\,\sum_{k=1}^{K}\mathrm{df}_k,
\]

where \(\hat{\mathbf{X}}_{K}\) denotes the rank-\(K\) reconstruction of \(\mathbf{X}\), and \(\mathrm{df}_k\) is the effective degrees of freedom for the \(k\)th layer, approximated by the total number of nonzero entries in \({\pmb{u}}^{(k)}\) and \({\pmb{\varphi}}^{(k)}\). In our simulation studies and real data analyses, we adopt the cumulative explained variance (CEV) strategy for selecting \(K\). For comparison, we also evaluate the performance of the CEV and BIC strategies in the simulation settings, as shown in Figures~\ref{ve:overall} and~\ref{bic:overall} in Section~\ref{Performance Evaluation}. In practice, users may choose the strategy according to the goal of the analysis: the sequential stopping rule is useful for identifying sparse and interpretable layers until no additional nonempty structure is detected; the CEV strategy is appropriate for retaining the dominant low-rank structure; and the BIC strategy provides a model-selection-oriented alternative that balances reconstruction accuracy and model complexity. 

\subsection{Methods Comparison}
\label{Methods Comparison}
Based on the simulation results presented above, we summarize the capabilities of each method across the tasks considered in Table~\ref{tab:task_yesno}, highlighting their relative strengths and limitations. Specifically, the first five columns indicate whether a method is able to produce sample clusters, variable clusters, subregion clusters, biclusters, and triclusters. The “Missing Data” column indicates whether a method can directly handle sparse functional data with irregularly missing observations. “In-Cluster Shape Diversity” reflects whether a method can accommodate heterogeneous trajectory shapes within the same cluster rather than enforcing identical within-cluster shapes. Finally, “Parameter Tuning” indicates whether the method provides an automatic or data-driven mechanism to select tuning parameters.

\setlength{\tabcolsep}{3pt}      
\newcolumntype{C}[1]{>{\centering\arraybackslash}p{#1}} 

\begin{table}[h!]
  \centering
  {
  \footnotesize
  \begin{tabularx}{\linewidth}{@{}l
      *{5}{C{1.47cm}}   
      C{2.4cm}        
      C{2.3cm}         
      C{1.2cm}@{}}     
    \toprule
    Method & Sample Cluster & Variable Cluster & Subregion Cluster &
    Bicluster & Tricluster & Sparse/Irregular Data &
    In-Cluster Heterogeneity & Parameter Tuning\\
    \midrule
    MFPCA        & \cmark & \xmark & \xmark & \xmark & \xmark & \cmark & \xmark & \cmark\\
    FPCA-BC      & \cmark & \cmark & \xmark & \cmark & \xmark & \cmark & \xmark & \cmark\\
    funCC        & \cmark & \cmark & \xmark & \cmark & \xmark & \xmark & \xmark & \cmark\\
    funLBM       & \cmark & \cmark & \xmark & \cmark & \xmark & \xmark & \xmark & \cmark\\
    Bi/Tri-SfSVD & \cmark & \cmark & \cmark & \cmark & \cmark & \cmark & \cmark & \cmark\\
    TriCluster   & \cmark & \cmark & \cmark & \cmark & \cmark & \xmark & \xmark & \xmark\\
    \bottomrule
  \end{tabularx}
  \caption{Comparison of the methods considered in the simulation and real data analyses. The table summarizes the clustering tasks each method is able to address, including sample clustering, variable clustering, subregion-level clustering, biclustering, and triclustering, as well as key data features relevant to our applications. Specifically, sparse/irregular data indicate whether the method can be applied to sparsely and irregularly observed functional data without requiring a complete functional representation for every subject--variable pair, and in-cluster heterogeneity indicates whether the method allows different temporal patterns within the same identified cluster structure. A checkmark indicates that the method supports the corresponding capability, whereas a cross indicates that it does not.}
  \label{tab:task_yesno}
  }
\end{table}

\newpage
\section{EEG Data Analysis} \label{real_data_eeg}

In this section, we analyze the EEG dataset obtained from \cite{Begleiter1995EEGDatabase}, with the scientific objective of identifying subject subgroups, channel groups, and localized temporal patterns through longitudinal EEG signals. The dataset contains 122 subjects, including 77 individuals with alcoholism and 45 control subjects. For each subject, EEG signals were recorded while viewing a picture for one second at a sampling rate of 256 Hz from 64 channels, resulting in 256 time points per channel. Among these channels, 61 correspond to scalp electrodes, two channels labeled X and Y record eye-movement activity, and the remaining channel labeled ND serves as the reference electrode signal.

As in the IBD analysis, each channel-specific signal was standardized before applying our method. To further illustrate the ability of our approach to handle irregularly and sparsely observed functional data, we additionally imposed a 50\% missingness rate by randomly masking half of the observed time points. Figure~\ref{eeg:sample} displays the trajectories across all 64 channels for one randomly selected alcoholic subject and one randomly selected control subject. Similar as IBD data, before applying
our method, each retained feature was standardized by subtracting its mean and dividing by
its standard deviation across all observed values.

\subsection{Interpretation of Feature Clusters Identified}

For the EEG data, we selected the number of sparse layers using the first stopping rule described in Section~\ref{Selection of the Number of Clusters} of the Supplementary Material. Specifically, we estimated sparse layers sequentially and stopped when the newly estimated layer became empty, that is, when either the estimated sample scores or the estimated feature loadings were all zero. Under this rule, the procedure selected \(K=3\) sparse layers. Under this choice, our method extracted three sparse layers spanning all 122 subjects and selected a total of 51 electrodes. 

Figure~\ref{fig:eeg_all_minipage} shows the spatial layout of the EEG channels and highlights the retained electrodes in each feature cluster after applying the near-zero loading filter. Specifically, the 61 scalp electrodes are plotted within the head diagram according to their electrode locations, whereas the remaining three channels are placed outside the diagram. In each panel, retained electrodes are shown in blue, whereas electrodes not retained for the corresponding feature cluster are shown in gray.

The identified feature clusters exhibit clear spatial structure. Feature cluster \(\pmb{F}_1\) mainly involves electrodes located on the front of the head, indicating a frontal pattern. In contrast, feature clusters \(\pmb{F}_2\) and \(\pmb{F}_3\) both mainly involve electrodes located on the back of the head, indicating posterior patterns. Unlike the IBD results, which did not reveal meaningful localization in time subregions, the EEG analysis additionally identifies localized time-domain structure within the selected electrodes. In particular, although \(\pmb{F}_2\) and \(\pmb{F}_3\) involve nearly the same set of electrodes, they differ in the time subregions retained by the method: \(\pmb{F}_2\) primarily selects the middle of the signal domain, whereas \(\pmb{F}_3\) primarily selects the later signal domain. This pattern indicates that the EEG data yield triclusters defined jointly over samples, electrodes, and their time subregions. Figure~\ref{fig:pc curve} provides the details on these temporal patterns by displaying the estimated singular functions for all 64 channels across the three sparse layers. Each sparse layer is represented by a different color. For a given electrode, a curve segment that departs from the zero baseline indicates that the corresponding time subregion of that electrode signal is selected in the corresponding sparse layer, whereas a curve that remains on the zero baseline throughout the entire domain indicates that the electrode is not selected in that layer.

\begin{figure}[t]
  \centering
  \begin{minipage}{0.32\textwidth}
    \centering
    \includegraphics[width=\linewidth]{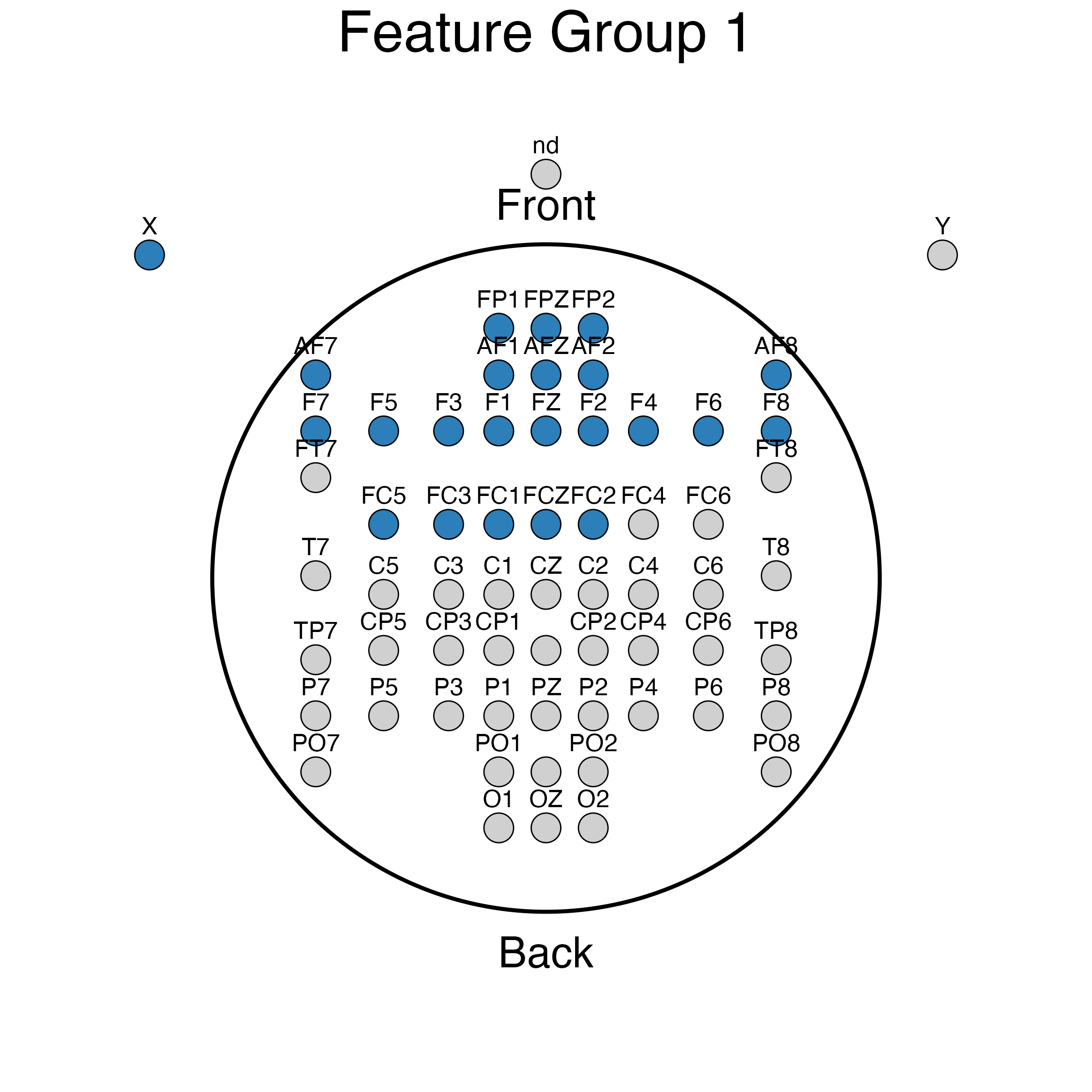}
  \end{minipage}\hfill
  \begin{minipage}{0.32\textwidth}
    \centering
    \includegraphics[width=\linewidth]{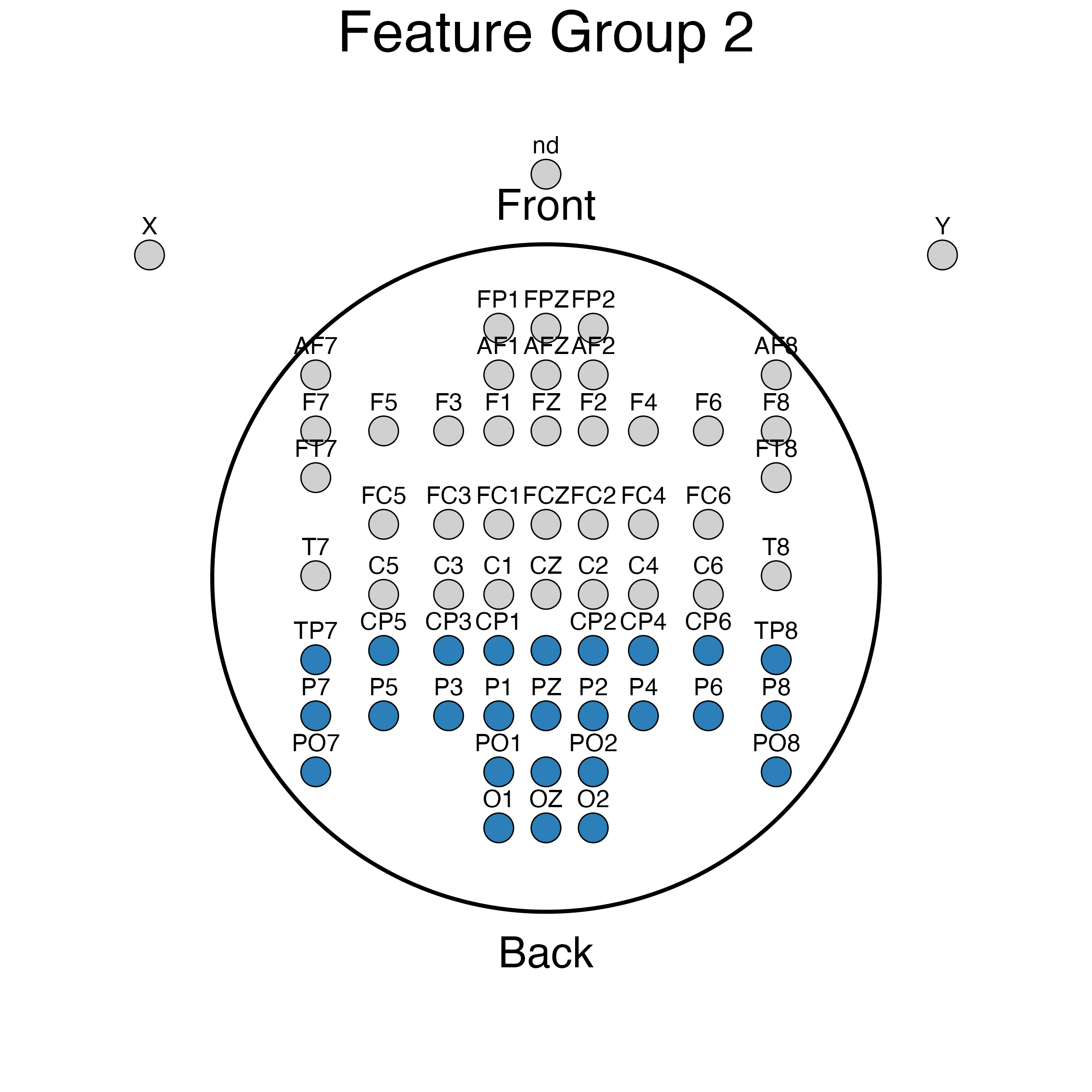}
  \end{minipage}\hfill
  \begin{minipage}{0.32\textwidth}
    \centering
    \includegraphics[width=\linewidth]{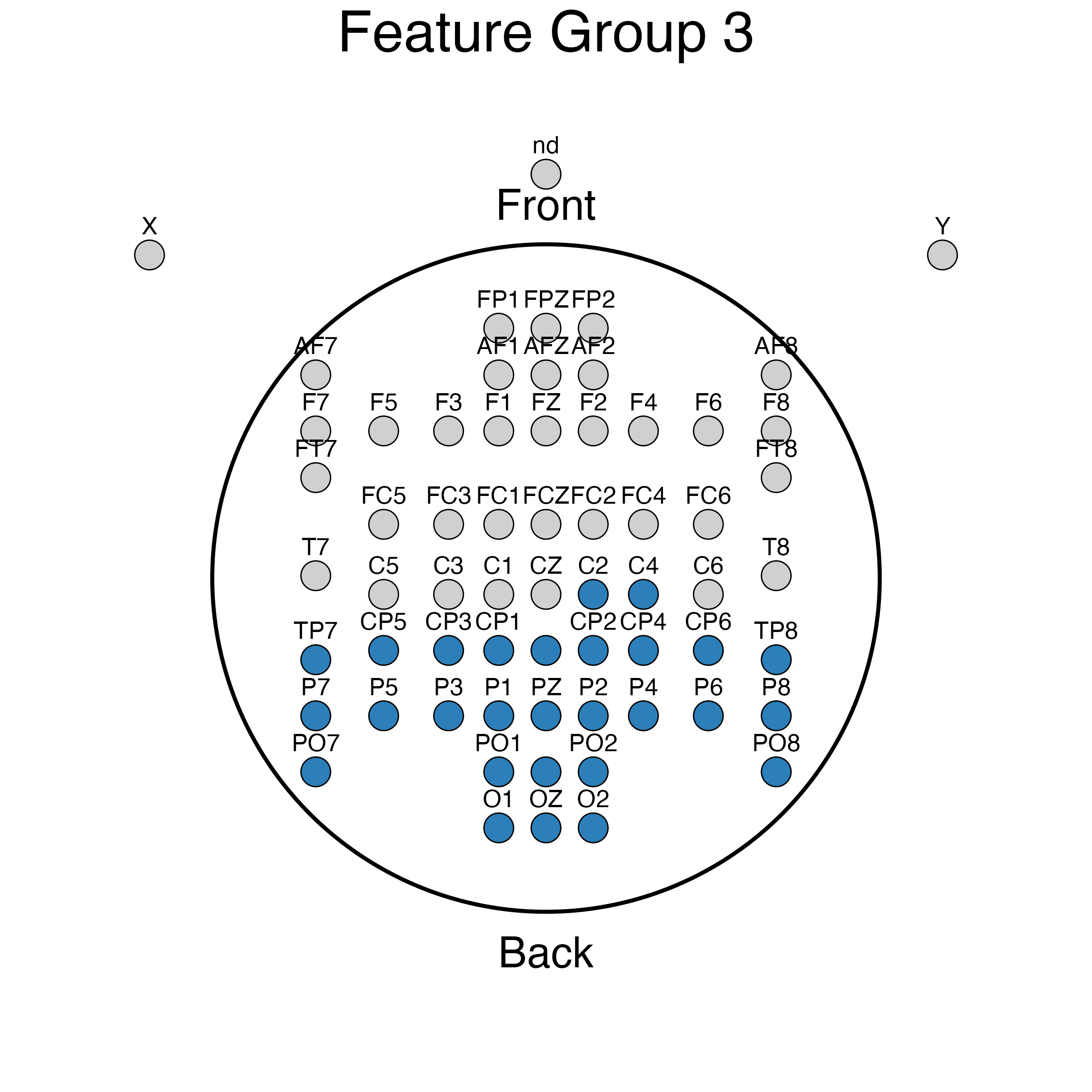}
  \end{minipage}

  \caption{Scalp maps of the EEG electrodes retained for each feature cluster after applying the near-zero loading filter. Retained electrodes are highlighted in blue, while non-retained electrodes are shown in gray. The eye-movement electrodes (X and Y) and the reference electrode ND are shown outside the scalp.}
  \label{fig:eeg_all_minipage}
\end{figure}

\subsection{Clinical Interpretation of Sample Clusters Identified}

In this dataset, we encountered the same issue as in the IBD analysis: the BIC criterion yielded relatively low sparsity in sample selection, making the estimated sample structure less directly interpretable. We therefore adopted the same post-processing procedure as in the IBD analysis. Specifically, we applied \(k\)-means clustering to the rows of the left singular matrix to partition the 122 subjects into three sample clusters, denoted by \(\{\pmb{S}_j\}_{j=1}^3\), and then linked the resulting sample clusters to the feature clusters \(\{\pmb{F}_k\}_{k=1}^3\) using the same WGCNA-based association procedure. The bootstrap results are summarized in Table~\ref{tab:assoc-bootmean-3x3-transposed}, where each entry reports the bootstrap mean and standard deviation of \(r_{jk}\), obtained by resampling subjects with replacement and recomputing the association scores across bootstrap replicates.

\begin{table}[h]
\centering
\small
\begin{tabular}{l *{3}{>{\raggedleft\arraybackslash}p{1.9cm}}}
\hline
 & $\pmb{F}_1$ & $\pmb{F}_2$ & $\pmb{F}_3$ \\
\hline
$\pmb{S}_1$ & 0.29 (0.08) & 0.49 (0.07) & \pmb{0.71 (0.06)} \\
$\pmb{S}_2$ & \pmb{0.67 (0.05)} & 0.33 (0.05) & 0.38 (0.07) \\
$\pmb{S}_3$ & 0.55 (0.06) & \pmb{0.62 (0.06)} & 0.29 (0.06) \\
\hline
\end{tabular}
\caption{Bootstrap mean (SD) of association scores \(r_{jk}\) between sample clusters \(\{\pmb{S}_j\}_{j=1}^{3}\) and feature clusters \(\{\pmb{F}_k\}_{k=1}^{3}\) in the EEG dataset. Bootstrap replicates were obtained by resampling subjects with replacement and recomputing \(r_{jk}\) in each replicate.}
\label{tab:assoc-bootmean-3x3-transposed}
\end{table}

Similar to the IBD analysis in Section~\ref{Sample Clusters Identified}, we next examined whether the identified sample clusters were associated with the available clinical phenotype. Specifically, we applied Fisher’s exact test to compare the distribution of participant type across the three sample clusters; the results are summarized in Table~\ref{eeg:phenotype}.

\begin{table}[h]
\centering
\small
\setlength{\tabcolsep}{6pt}
\begin{tabular}{l l
                >{\centering\arraybackslash}p{1.6cm}
                >{\centering\arraybackslash}p{1.6cm}
                >{\centering\arraybackslash}p{1.6cm}
                >{\centering\arraybackslash}p{2.2cm}}
\hline
 &  & \multicolumn{3}{c}{Sample clusters (N)} & \multirow{2}{*}{P-value} \\
\cline{3-5}
 &  & 1 (21) & 2 (71) & 3 (30) & \\
\hline
\multirow{2}{*}{Type}
  & Alcoholic  & 15 (71\%) & 52 (73\%) & 10 (33\%)  & \multirow{2}{*}{$<0.001$} \\
  & Control    & 6 (29\%)  & 19 (27\%) & 20 (67\%)  & \\
\hline
\end{tabular}
\caption{Association between sample clusters and participant type (Alcoholic vs.\ Control) in the EEG dataset. Entries are counts (percentages) within each sample cluster; the \(P\)-value is from Fisher’s exact test.}
\label{eeg:phenotype}
\end{table}

As shown in Table~\ref{eeg:phenotype}, the identified sample clusters are significantly associated with participant type (\(P<0.001\)), indicating that the sample structure recovered by our method is clinically meaningful in this application. In particular, \(\pmb{S}_1\) was enriched for alcoholic subjects (71\%), suggesting an alcoholic-dominant subgroup. \(\pmb{S}_2\) was also enriched for alcoholic subjects (73\%), indicating a second alcoholic-dominant subgroup with a distinct EEG pattern. In contrast, \(\pmb{S}_3\) was enriched for control subjects (67\%), corresponding to a control-dominant subgroup.

\subsection{Interpretation of Identified Triclusters}

Table~\ref{tab:assoc-bootmean-3x3-transposed} summarizes the associations between the sample clusters \(\{\pmb{S}_j\}_{j=1}^{3}\) and the electrode groups \(\{\pmb{F}_k\}_{k=1}^{3}\). To obtain an interpretable tricluster representation, we associated each sample cluster with the electrode group showing its strongest correlation. This yielded three dominant sample--electrode-group associations: \((\pmb{S}_1,\pmb{F}_3)\), \((\pmb{S}_2,\pmb{F}_1)\), and \((\pmb{S}_3,\pmb{F}_2)\). Combining these pairings with the time-subregion information identified in the feature loadings gives three triclusters: the alcoholic-enriched cluster \(\pmb{S}_1\) associated with the late subregion of the posterior electrode group \(\pmb{F}_3\), the alcoholic-enriched cluster \(\pmb{S}_2\) associated with the frontal electrode group \(\pmb{F}_1\), and the control-enriched cluster \(\pmb{S}_3\) associated with the early subregion of the posterior electrode group \(\pmb{F}_2\).

Figure~\ref{fig:supp_eeg_S2F1} provides a detailed display of the identified frontal tricluster \((\pmb{S}_2,\pmb{F}_1)\) through the raw sample mean measurements and reconstructed mean EEG trajectories for the retained electrodes in \(\pmb{F}_1\). In this tricluster, \(\pmb{S}_2\), an alcoholic-enriched sample cluster, is paired with the frontal electrode group \(\pmb{F}_1\). The reconstructed trajectories show a coherent frontal EEG pattern across the retained electrodes, suggesting that this subgroup is characterized by frontal electrophysiological variation. This interpretation is broadly consistent with prior alcoholism research, which has implicated frontal lobe dysfunction in alcohol-related brain abnormalities and reported electrophysiological abnormalities during tasks involving frontal inhibitory control in alcohol-dependent individuals \citep{Moselhy2001FrontalAlcoholism,Kamarajan2004FrontalInhibitoryControl}.

Both \(\pmb{S}_1\) and \(\pmb{S}_3\) are linked to posterior electrode groups, but they are distinguished by their time-domain supports rather than by spatial location alone. Specifically, \(\pmb{S}_3\) is associated with \(\pmb{F}_2\), whose selected posterior electrodes are primarily active in an intermediate time window, whereas \(\pmb{S}_1\) is associated with \(\pmb{F}_3\), whose selected posterior electrodes are primarily active in a later time window. The corresponding mean reconstructed trajectories in Figure~\ref{fig:supp_eeg_S3F2} and \ref{fig:supp_eeg_S1F3} reflect this distinction: the \(\pmb{S}_3\)-\(\pmb{F}_2\) block exhibits stronger signal patterns in the middle portion of the signal domain, whereas the \(\pmb{S}_1\)-\(\pmb{F}_3\) block exhibits later posterior signal patterns. This demonstrates that the subgroup structure in the EEG data is defined jointly by subjects, spatial electrode regions, and localized time subregions, highlighting the value of a triclustering perspective beyond analyses based only on sample clusters or spatial electrode groups.

Figure~\ref{fig:pc curve} provides additional details on the estimated singular functions across all 64 channels. For a given electrode, a curve segment that departs from the zero baseline indicates that the corresponding time subregion is selected in that sparse layer, whereas a curve that remains on the zero baseline over the entire domain indicates that the electrode is not selected in that layer.




\begin{figure}[h]
  \centering
  \begin{subfigure}[b]{0.45\textwidth}
    \centering
    \includegraphics[width=\linewidth]{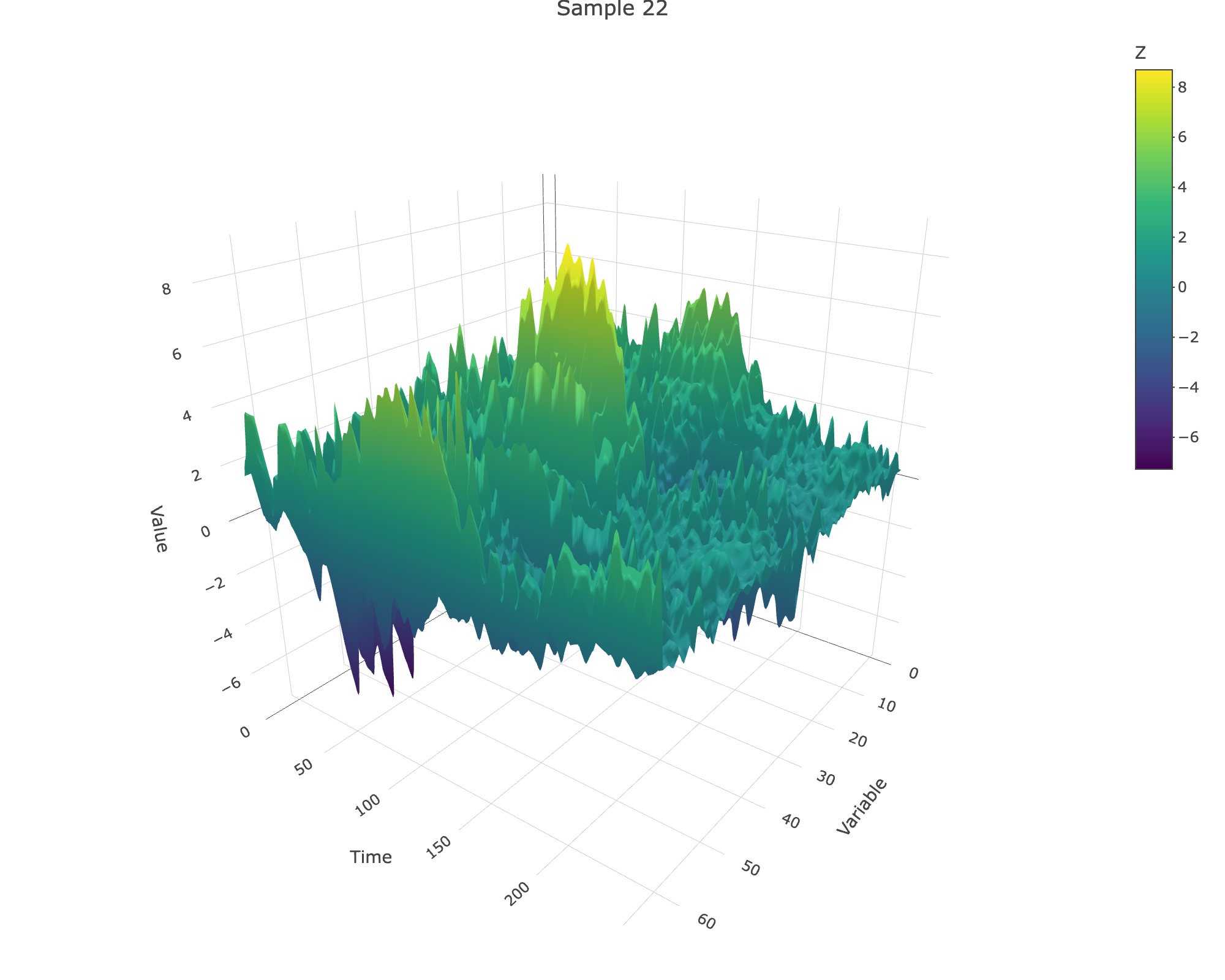}
    \caption{Alcoholic}
    \label{eeg:alcoholic}
  \end{subfigure}
  \hfill
  \begin{subfigure}[b]{0.45\textwidth}
    \centering
    \includegraphics[width=\linewidth]{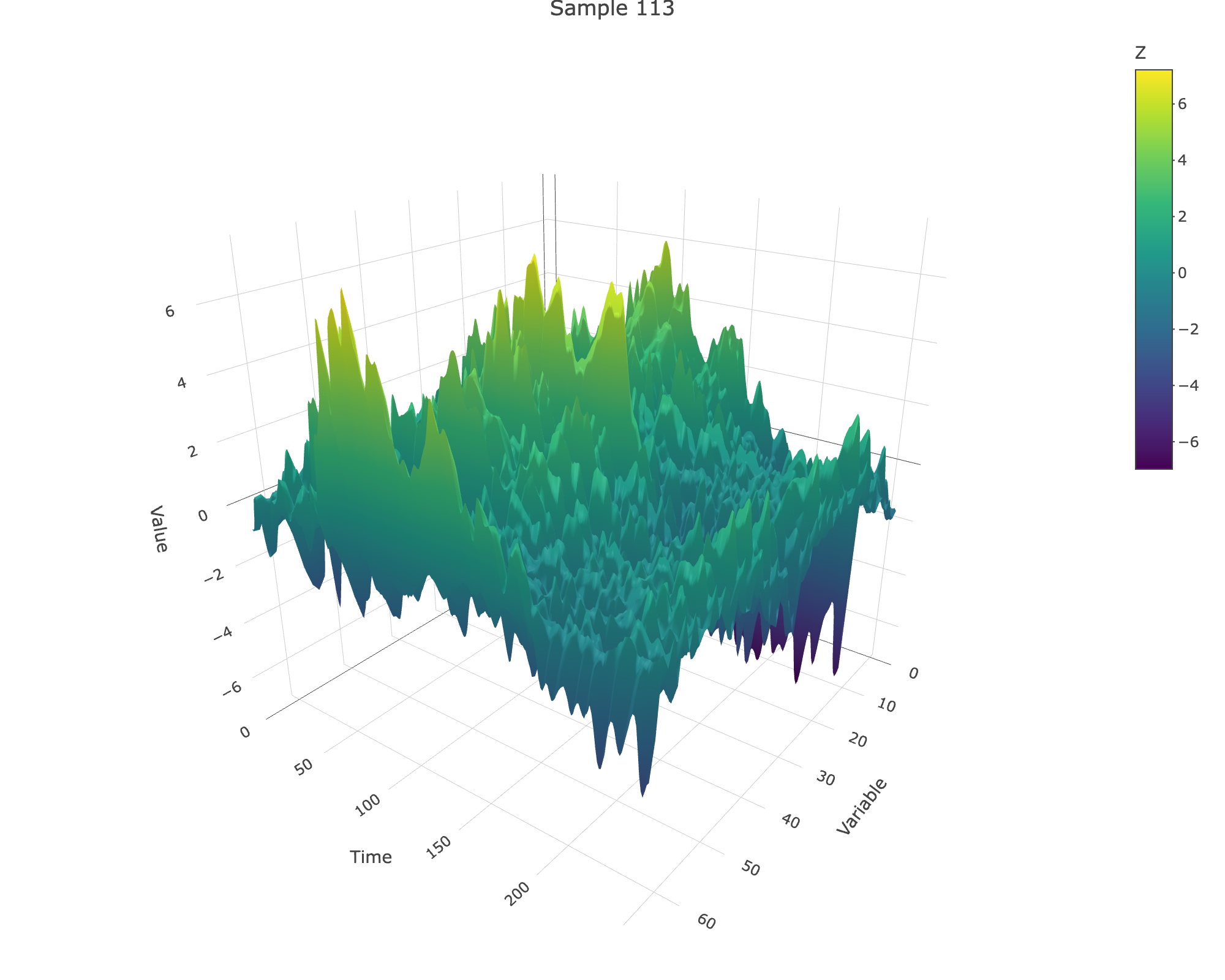}
    \caption{Control}
    \label{eeg:control}
  \end{subfigure}
  \caption{Trajectories across all 64 electrodes for two randomly selected subjects: (a) alcoholic, (b) control.}
  \label{eeg:sample}
\end{figure}

\clearpage

\begin{figure}[p]
    \centering
    \includegraphics[
        width=\textwidth,
        height=0.78\textheight,
        keepaspectratio
    ]{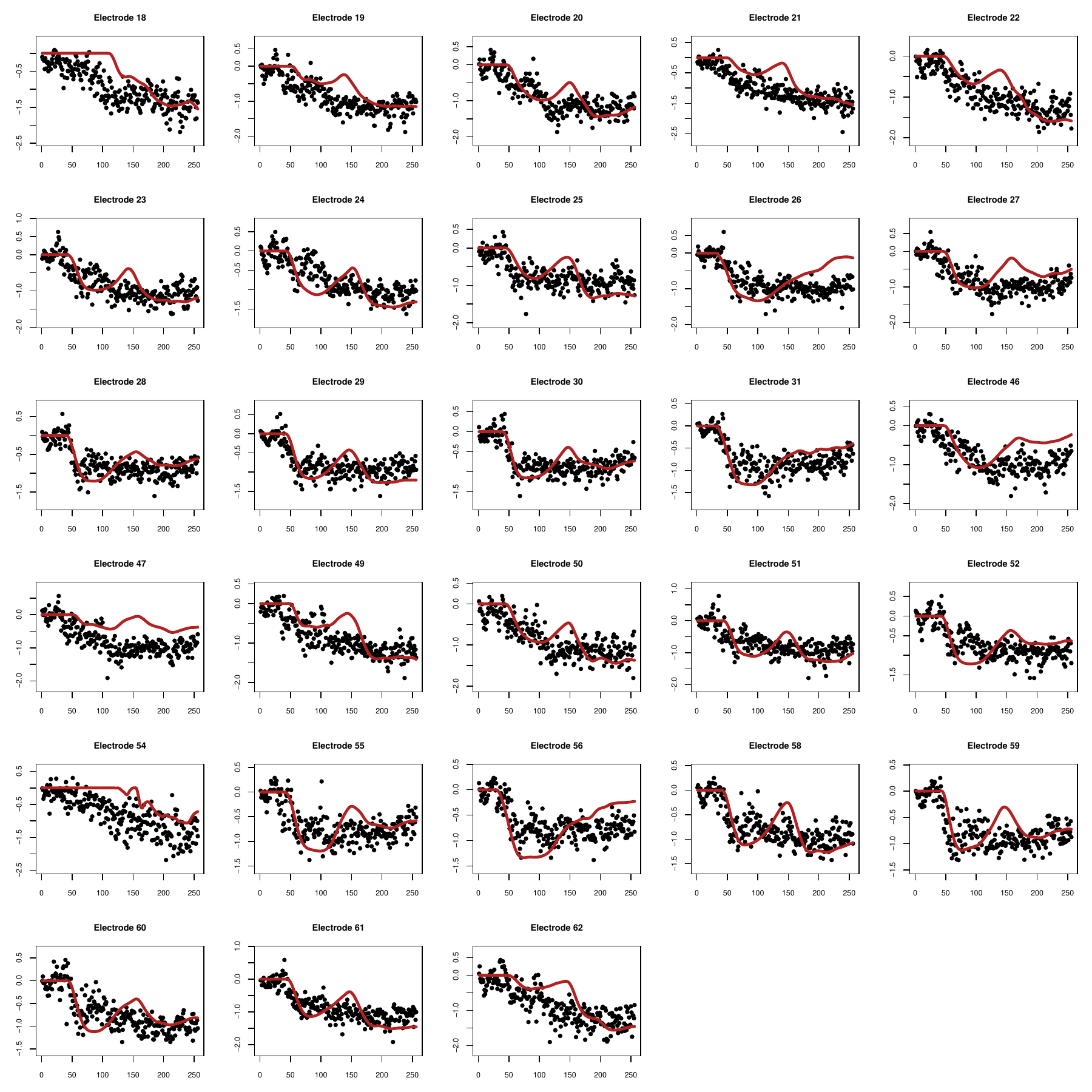}
    \caption{
    Detailed display of all selected electrodes within the identified EEG tricluster \((\pmb{S}_1,\pmb{F}_3)\). Here, \(\pmb{S}_1\) denotes the corresponding sample cluster and \(\pmb{F}_3\) denotes the corresponding electrode group. Each panel corresponds to one selected electrode in \(\pmb{F}_3\). Black points represent raw sample mean measurements among samples in \(\pmb{S}_1\), computed after omitting missing observations at each time point. Colored solid curves represent reconstructed sample mean trajectories obtained from the rank-\(K\) reconstruction of our sparse functional SVD model.
    }
    \label{fig:supp_eeg_S1F3}
\end{figure}

\clearpage
\begin{figure}[p]
    \centering
    \includegraphics[
        width=\textwidth,
        height=0.78\textheight,
        keepaspectratio
    ]{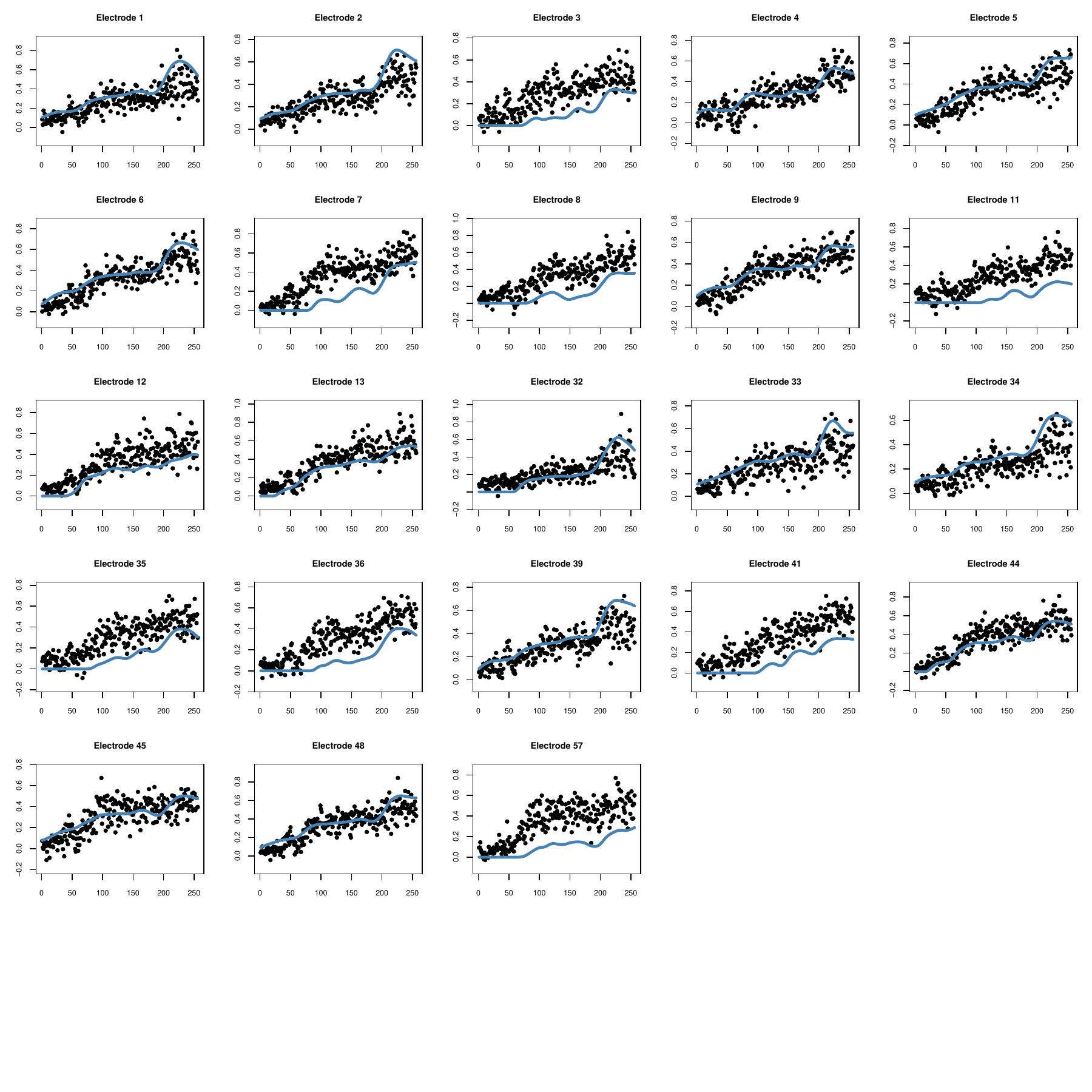}
    \caption{
    Detailed display of the retained electrodes within the identified EEG tricluster \((\pmb{S}_2,\pmb{F}_1)\). Here, \(\pmb{S}_2\) denotes the corresponding sample cluster and \(\pmb{F}_1\) denotes the corresponding electrode group. Each panel corresponds to one retained electrode in \(\pmb{F}_1\); electrode FC6 was excluded because its loading norm was effectively zero. Black points represent raw sample mean measurements among samples in \(\pmb{S}_2\), computed after omitting missing observations at each time point. Colored solid curves represent reconstructed sample mean trajectories obtained from the rank-\(K\) reconstruction of our sparse functional SVD model.
    }
    \label{fig:supp_eeg_S2F1}
\end{figure}

\clearpage
\begin{figure}[p]
    \centering
    \includegraphics[
        width=\textwidth,
        height=0.78\textheight,
        keepaspectratio
    ]{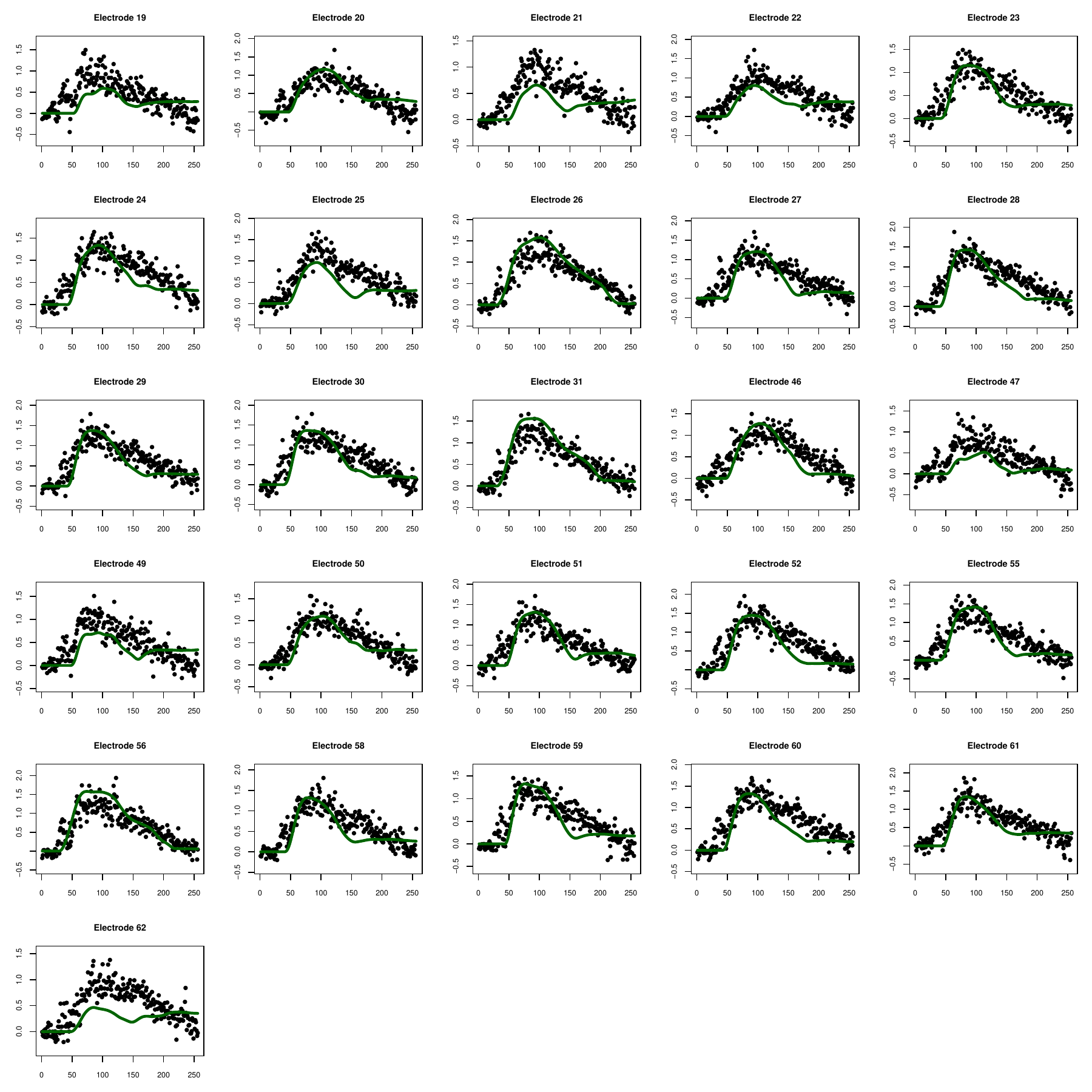}
    \caption{
    Detailed display of all selected electrodes within the identified EEG tricluster \((\pmb{S}_3,\pmb{F}_2)\). Here, \(\pmb{S}_3\) denotes the corresponding sample cluster and \(\pmb{F}_2\) denotes the corresponding electrode group. Each panel corresponds to one selected electrode in \(\pmb{F}_2\). Black points represent raw sample mean measurements among samples in \(\pmb{S}_3\), computed after omitting missing observations at each time point. Colored solid curves represent reconstructed sample mean trajectories obtained from the rank-\(K\) reconstruction of our sparse functional SVD model.
    }
    \label{fig:supp_eeg_S3F2}
\end{figure}


\begin{figure}[h]
  \centering
  \includegraphics[width=1\linewidth]{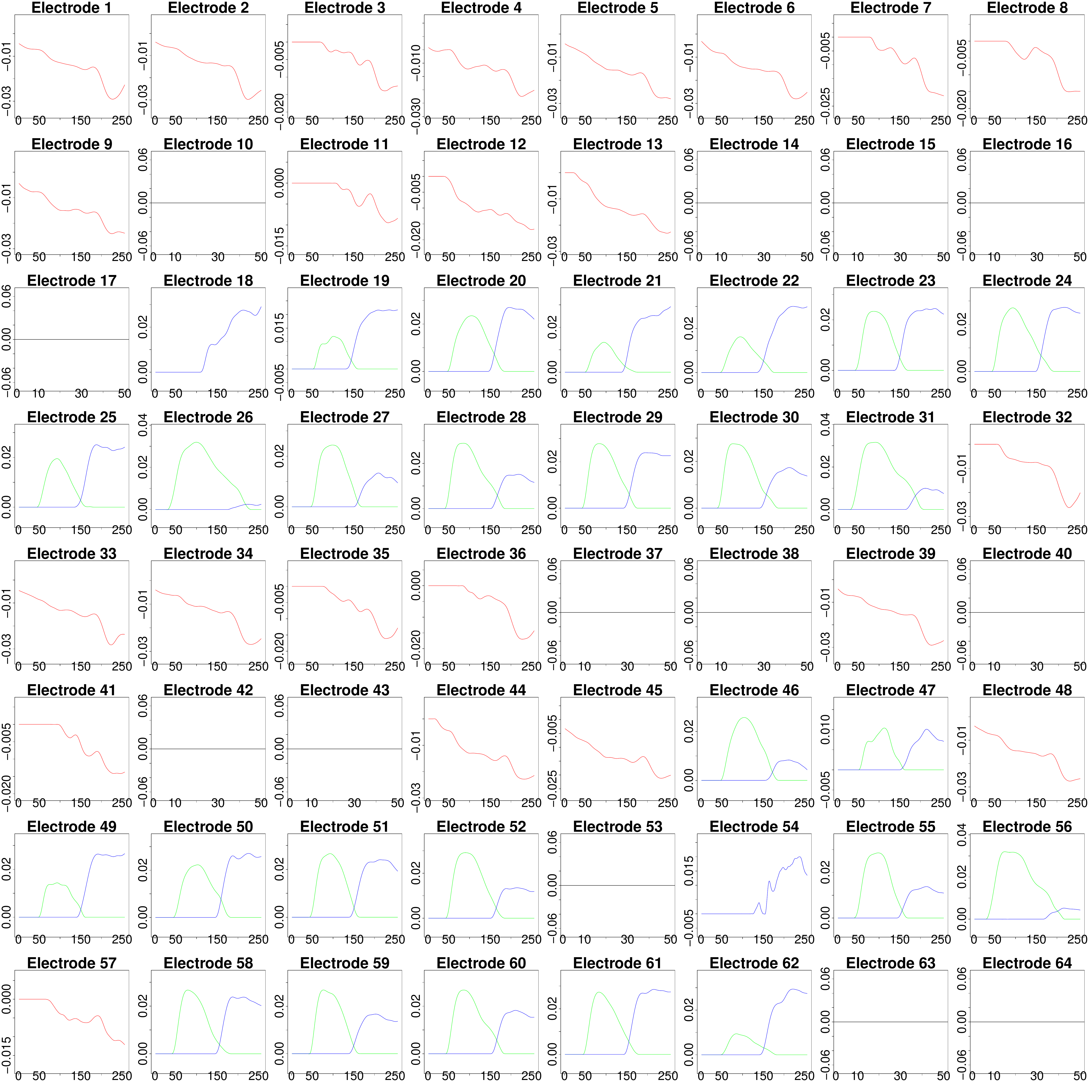}
  \caption{Estimated singular functions for each sparse layer after applying the same near-zero loading filter used for the electrode maps and raw/reconstruction figures. For a given feature, a curve that coincides with the zero baseline indicates that the feature is not retained for the corresponding tricluster, whereas a curve that deviates from zero indicates that the feature is retained, with the nonzero time interval identifying the selected time subregion. Panels showing only baseline at zero indicate features not retained by any layer.}

  \label{fig:pc curve}
\end{figure}

\clearpage
\section{More on IBD Data Analysis} \label{real_data_appendix}

\subsection{Additional Results in the IBD Real Data Analysis}
\label{supp:ibd_additional_results}

\subsubsection{Additional Interpretation of Identified Biclusters}
\label{supp:ibd_additional_bicluster_interpretation}

The main text focuses on the \(\pmb{F}_1\)-related biclusters. Here we provide the additional interpretation for the remaining pathway groups involved in the identified IBD bicluster structure.

A parallel pattern of heterogeneity is observed for pathway group \(\pmb{F}_2\), which is primarily characterized by features stratified to \textit{Ruminococcus torques}. Based on the distance-correlation criterion in Table~\ref{tab:assoc-bootmean} of the main manuscript, \(\pmb{F}_2\) appears in two selected biclusters, \((\pmb{S}_2,\pmb{F}_2)\) and \((\pmb{S}_3,\pmb{F}_2)\). Figure~\ref{fig:ibd_overlap_bicluster_mean} of the main manuscript illustrates these biclusters through representative pathways from \(\pmb{F}_2\), showing distinct longitudinal profiles between the two subject clusters. While the average pathway-feature levels in \(\pmb{S}_2\) show a sustained decrease over the 50-week period, the trajectories in \(\pmb{S}_3\) display a pronounced upward trend toward the end of the study duration. The observed functional contrast is consistent with the clinical phenotype distributions of these groups (Table~\ref{tab:phenotype} of the main manuscript). Specifically, \(\pmb{S}_2\) is characterized by a higher proportion of non-IBD control subjects, whereas \(\pmb{S}_3\) is heavily enriched with IBD patients, particularly those with CD. These findings align with recent metagenomic studies that have identified \textit{R. torques} as a key marker of IBD-related dysbiosis and have linked its functional activity to CD-specific metabolic shifts \citep{Zheng2024MicrobiomeDiagnosisIBD}. By capturing these non-homogeneous signals, our framework demonstrates that \(\pmb{F}_2\) serves as a bridge feature group that differentiates the relatively stable microbial states of the \(\pmb{S}_2\) subgroup from the more volatile, CD-enriched profiles of \(\pmb{S}_3\). This highlights the importance of allowing features to be shared across subject clusters while still capturing group-specific temporal variations.

Finally, \(\pmb{F}_3\) is most strongly associated with \(\pmb{S}_3\), suggesting that it represents an additional aspect of microbial functional variation in this IBD-enriched subgroup. In contrast to \(\pmb{F}_1\) and \(\pmb{F}_2\), which are dominated by pathways sharing a common bacterial label, \(\pmb{F}_3\) is composed primarily of pathways labeled as unclassified in this dataset. Its taxonomic interpretation is therefore less direct and remains open to further biological investigation. Nevertheless, its strong association with \(\pmb{S}_3\) suggests that, in addition to the patterns represented by \(\pmb{F}_1\) and \(\pmb{F}_2\), this IBD-enriched subgroup is linked to a further pathway-level functional signal.

\subsubsection{Comparison with Alternative Methods}
\label{supp:ibd_competing_methods}

To further evaluate the competing methods in the IBDMDB analysis, we examined
whether the sample clusters or sample memberships obtained from each method were
associated with the available clinical phenotypes. Tables~\ref{tab:MFPCA}--\ref{tab:funlbm-analysis-phenotype}
summarize the phenotype distributions for MFPCA-\(K\)mean, FPCA-BC, funCC, and
funLBM, respectively. These supplementary results focus on the clinical
interpretability of the recovered sample structures and provide additional
context for the comparisons reported in the main text.

For MFPCA-\(K\)mean, the resulting sample clusters were significantly
associated with IBD type (\(P<0.001\); Table~\ref{tab:MFPCA}), indicating that
MFPCA recovered some clinically relevant subject-level structure. However, this
approach only performs sample clustering based on MFPCA scores and does not
select pathway groups. Therefore, although it can identify phenotype-associated
sample clusters, it cannot provide interpretable sample--pathway biclusters.

For FPCA-BC, the sample clusters were not significantly associated with either
IBD type or sex in this application (Table~\ref{tab:ufpca}). This suggests that
the two-stage procedure, which first summarizes each feature using FPCA and then
applies a standard biclustering algorithm, did not recover sample structures with
clear clinical relevance in the IBDMDB dataset.

For funCC, the resulting sample memberships also showed no significant
associations with the examined clinical variables, including IBD type, sex, age
at diagnosis, CRP, ESR, and race (Table~\ref{tab:funcc-phenotype}). These results
suggest that funCC did not yield a clearly clinically interpretable bicluster
structure for this dataset.

For funLBM, the sample clusters were significantly associated with IBD type and
CRP (Table~\ref{tab:funlbm-analysis-phenotype}), suggesting that the method
captured some clinically meaningful sample structure. However, model selection in
funLBM is based on the integrated completed likelihood (ICL) criterion. As shown
in Figure~\ref{fig:ICL}, the model with the highest ICL value corresponds to 3
sample clusters and 15 feature clusters, resulting in a total of 45 biclusters.
The mean-curve patterns for these 45 biclusters are displayed in
Figure~\ref{fig:sim_candidate_curves}. Although this result contains clinically
associated sample clusters, the large number of biclusters makes the resulting
sample--feature structure less directly interpretable.

Overall, these supplementary comparisons suggest that the competing methods show
more limited interpretability in this application. MFPCA-\(K\)mean identifies
clinically associated sample clusters but does not perform pathway selection;
FPCA-BC and funCC do not recover clearly clinically meaningful sample structures;
and funLBM yields a relatively complex bicluster representation. In contrast, the
proposed method provides a more compact and biologically interpretable
sample--pathway representation for the IBDMDB data.

\begin{table}[h]
\centering
\small
\setlength{\tabcolsep}{6pt}
\begin{tabular}{l l
                >{\centering\arraybackslash}p{1.6cm}
                >{\centering\arraybackslash}p{1.6cm}
                >{\centering\arraybackslash}p{1.6cm}
                >{\centering\arraybackslash}p{1.6cm}
                >{\centering\arraybackslash}p{2.2cm}}
\hline
 & & \multicolumn{4}{c}{Sample clusters (N)} & \multirow{2}{*}{P-value} \\
\cline{3-6}
 & & $\pmb{S}_1$ (22) & $\pmb{S}_2$ (26) & $\pmb{S}_3$ (52) & $\pmb{S}_4$ (30) & \\
\hline
\multirow{3}{*}{IBD type}
  & Non-IBD & 9 (41\%) & 10 (38\%) & 2 (4\%) & 6 (20\%) & \multirow{3}{*}{$<0.001$} \\
  & CD & 12 (55\%) & 6 (24\%) & 29 (56\%) & 18 (60\%) & \\
  & UC & 1 (4\%) & 10 (38\%) & 21 (40\%) & 6 (20\%) & \\
\hline
\multirow{2}{*}{Gender}
  & Female & 10 (45\%) & 11 (42\%) & 25 (48\%) & 18 (60\%) & \multirow{2}{*}{0.5721} \\
  & Male & 12 (55\%) & 15 (58\%) & 27 (52\%) & 12 (40\%) & \\
\hline
\end{tabular}
\caption{
Associations between MFPCA-\(K\)mean sample clusters and clinical phenotypes in the IBDMDB dataset. Entries are counts (percentages) within each sample cluster. \(P\)-values are from Fisher's exact tests. The table shows that the sample clusters obtained from MFPCA-\(K\)mean differ significantly in IBD type, whereas no significant difference is observed for sex.
}
\label{tab:MFPCA}
\end{table}

\begin{table}[h]
\centering
\small
\setlength{\tabcolsep}{6pt}
\begin{tabular}{l l
                >{\centering\arraybackslash}p{1.6cm}
                >{\centering\arraybackslash}p{1.6cm}
                >{\centering\arraybackslash}p{1.6cm}
                >{\centering\arraybackslash}p{1.6cm}
                >{\centering\arraybackslash}p{2.2cm}}
\hline
 &  & \multicolumn{4}{c}{Sample clusters (N)} & \multirow{2}{*}{P-value} \\
\cline{3-6}
 &  & $\pmb{S}_1$ (59) & $\pmb{S}_2$ (45) & $\pmb{S}_3$ (53) & $\pmb{S}_4$ (53) & \\
\hline
\multirow{3}{*}{IBD type}
  & Non-IBD & 9 (15\%)  & 6 (13\%) & 12 (23\%)  & 3 (6\%)   & \multirow{3}{*}{$0.2337$} \\
  & CD      & 31 (53\%) & 25 (56\%) & 29 (54\%)  & 29 (54\%) & \\
  & UC      & 19 (32\%)  & 14 (31\%) & 12 (23\%)   & 21 (40\%) & \\
\hline
\multirow{2}{*}{Gender}
  & Female  & 30 (51\%) & 23 (51\%) & 31 (58\%)  & 29 (55\%) & \multirow{2}{*}{0.8446} \\
  & Male    & 29 (49\%)  & 22 (49\%) & 22 (42\%)  & 24 (45\%) & \\
\hline
\end{tabular}
\caption{
Associations between FPCA-BC sample clusters and clinical phenotypes in the IBDMDB dataset. Entries are counts (percentages) within each sample cluster. \(P\)-values are from Fisher's exact tests. The table shows that the sample clusters obtained from FPCA-BC are not significantly associated with either IBD type or sex in this application.
}
\label{tab:ufpca}
\end{table}

\clearpage

\begin{table}[h]
\centering
\small
\setlength{\tabcolsep}{5pt}
\resizebox{\textwidth}{!}{%
\begin{tabular}{l l
  >{\centering\arraybackslash}p{1.6cm}
  >{\centering\arraybackslash}p{1.6cm}
  >{\centering\arraybackslash}p{1.6cm}
  >{\centering\arraybackslash}p{1.6cm}
  >{\centering\arraybackslash}p{1.6cm}
  >{\centering\arraybackslash}p{1.6cm}
  >{\centering\arraybackslash}p{2.2cm}}
\hline
 & & \multicolumn{6}{c}{Sample clusters (N)} & \multirow{2}{*}{P-value} \\
\cline{3-8}
 & & $\pmb{S}_1$ (80) & $\pmb{S}_2$ (80) & $\pmb{S}_3$ (50) & $\pmb{S}_4$ (50) & $\pmb{S}_5$ (50) & $\pmb{S}_6$ (50) & \\
\hline
\multirow{3}{*}{IBD type} & Non-IBD & 15 (18.8\%) & 15 (18.8\%) & 12 (24.0\%) & 12 (24.0\%) & 12 (24.0\%) & 12 (24.0\%) & \multirow{3}{*}{0.9989} \\
 & CD & 41 (51.2\%) & 41 (51.2\%) & 24 (48.0\%) & 24 (48.0\%) & 24 (48.0\%) & 24 (48.0\%) &  \\
 & UC & 24 (30.0\%) & 24 (30.0\%) & 14 (28.0\%) & 14 (28.0\%) & 14 (28.0\%) & 14 (28.0\%) &  \\
\hline
\multirow{2}{*}{Sex} & Female & 36 (45.0\%) & 36 (45.0\%) & 28 (56.0\%) & 28 (56.0\%) & 28 (56.0\%) & 28 (56.0\%) & \multirow{2}{*}{0.5118} \\
 & Male & 44 (55.0\%) & 44 (55.0\%) & 22 (44.0\%) & 22 (44.0\%) & 22 (44.0\%) & 22 (44.0\%) &  \\
\hline
\multicolumn{2}{>{\raggedright\arraybackslash}p{5.8cm}}{\small Age at diagnosis, mean (SD)} & 19.89 (9.82) & 19.89 (9.82) & 23.74 (15.19) & 23.74 (15.19) & 23.74 (15.19) & 23.74 (15.19) & 0.9120 \\
\hline
\multicolumn{2}{>{\raggedright\arraybackslash}p{5.8cm}}{\small CRP (C-Reactive Protein), mean (SD)} & 40.13 (179.30) & 40.13 (179.30) & 75.40 (261.17) & 75.40 (261.17) & 75.40 (261.17) & 75.40 (261.17) & 1.0000 \\
\hline
\multicolumn{2}{>{\raggedright\arraybackslash}p{5.8cm}}{\small ESR (Erythrocyte Sedimentation Rate), mean (SD)} & 54.94 (176.39) & 54.94 (176.39) & 19.75 (17.96) & 19.75 (17.96) & 19.75 (17.96) & 19.75 (17.96) & 0.9069 \\
\hline
\multirow{5}{*}{Race} & American Indian or Alaska Native & 1 & 1 & 0 & 0 & 0 & 0 & \multirow{5}{*}{0.9421} \\
 & Black or African American & 7 & 7 & 3 & 3 & 3 & 3 &  \\
 & More than one race & 3 & 3 & 2 & 2 & 2 & 2 &  \\
 & Other & 4 & 4 & 0 & 0 & 0 & 0 &  \\
 & White & 65 & 65 & 45 & 45 & 45 & 45 &  \\
\hline
\end{tabular}%
}
\caption{
Associations between funCC sample memberships and clinical phenotypes in the IBDMDB dataset. Entries are counts (percentages) for categorical variables and mean (SD) for continuous variables. \(P\)-values are from Fisher's exact tests and Kruskal--Wallis tests. The table shows that the funCC sample memberships are not significantly associated with the examined clinical variables in this application.
}
\label{tab:funcc-phenotype}
\end{table}

\clearpage

\begin{table}[h]
\centering
\small
\setlength{\tabcolsep}{6pt}
\begin{tabular}{l l
  >{\centering\arraybackslash}p{1.6cm}
  >{\centering\arraybackslash}p{1.6cm}
  >{\centering\arraybackslash}p{1.6cm}
  >{\centering\arraybackslash}p{1.6cm}
  >{\centering\arraybackslash}p{2.2cm}}
\hline
 & & \multicolumn{4}{c}{Sample clusters (N)} & \multirow{2}{*}{P-value} \\
\cline{3-6}
 & & $\pmb{S}_1$ (27) & $\pmb{S}_2$ (41) & $\pmb{S}_3$ (33) & $\pmb{S}_4$ (29) & \\
\hline
\multirow{3}{*}{IBD type} & Non-IBD & 14 (51.9\%) & 2 (4.9\%) & 6 (18.2\%) & 5 (17.2\%) & \multirow{3}{*}{$<0.001$} \\
 & CD & 11 (40.7\%) & 23 (56.1\%) & 18 (54.5\%) & 13 (44.8\%) &  \\
 & UC & 2 (7.4\%) & 16 (39.0\%) & 9 (27.3\%) & 11 (37.9\%) &  \\
\hline
\multirow{2}{*}{Sex} & Female & 14 (51.9\%) & 20 (48.8\%) & 17 (51.5\%) & 13 (44.8\%) & \multirow{2}{*}{0.9510} \\
 & Male & 13 (48.1\%) & 21 (51.2\%) & 16 (48.5\%) & 16 (55.2\%) &  \\
\hline
\multicolumn{2}{>{\raggedright\arraybackslash}p{5.8cm}}{\small Age at diagnosis, mean (SD)} & 19.42 (9.77) & 21.00 (9.84) & 20.70 (11.61) & 23.50 (16.88) & 0.8689 \\
\hline
\multicolumn{2}{>{\raggedright\arraybackslash}p{5.8cm}}{\small CRP (C-Reactive Protein), mean (SD)} & 1.23 (2.35) & 3.81 (8.67) & 78.93 (276.46) & 112.92 (304.42) & 0.0381 \\
\hline
\multicolumn{2}{>{\raggedright\arraybackslash}p{5.8cm}}{\small ESR (Erythrocyte Sedimentation Rate), mean (SD)} & 12.09 (12.95) & 19.60 (16.39) & 25.31 (21.34) & 75.00 (218.54) & 0.0925 \\
\hline
\multirow{5}{*}{Race} & American Indian or Alaska Native & 0 & 1 & 0 & 0 & \multirow{5}{*}{0.3843} \\
 & Black or African American & 0 & 4 & 4 & 2 &  \\
 & More than one race & 0 & 1 & 2 & 2 &  \\
 & Other & 0 & 2 & 2 & 0 &  \\
 & White & 27 & 33 & 25 & 25 &  \\
\hline
\end{tabular}
\caption{
Associations between funLBM sample clusters and clinical phenotypes in the IBDMDB dataset. Entries are counts (percentages) for categorical variables and mean (SD) for continuous variables. \(P\)-values are from Fisher's exact tests for categorical variables and Kruskal--Wallis tests for continuous variables. The table shows that the funLBM sample clusters differ significantly in IBD type and CRP, whereas no significant differences are observed for sex, age at diagnosis, ESR, or race.
}
\label{tab:funlbm-analysis-phenotype}
\end{table}

\begin{figure}[htbp]
    \centering
    \includegraphics[width=\textwidth]{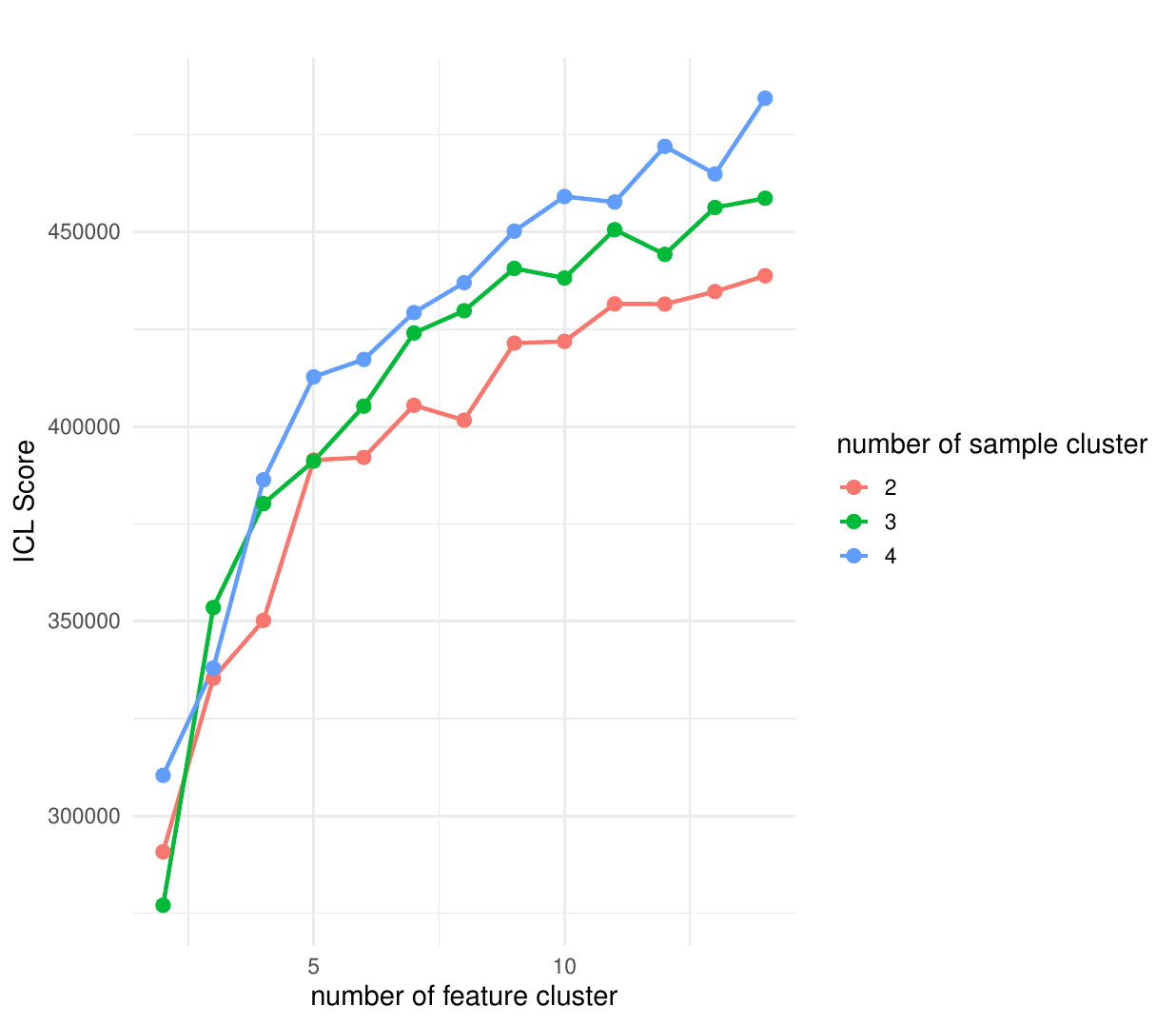}
\caption{
ICL values for funLBM model selection in the IBDMDB dataset. The selected model corresponds to the highest ICL value, which suggests 3 sample clusters and 15 feature clusters.
}
\label{fig:ICL}
\end{figure}

\begin{figure}[htbp]
    \centering
    \includegraphics[width=\textwidth]{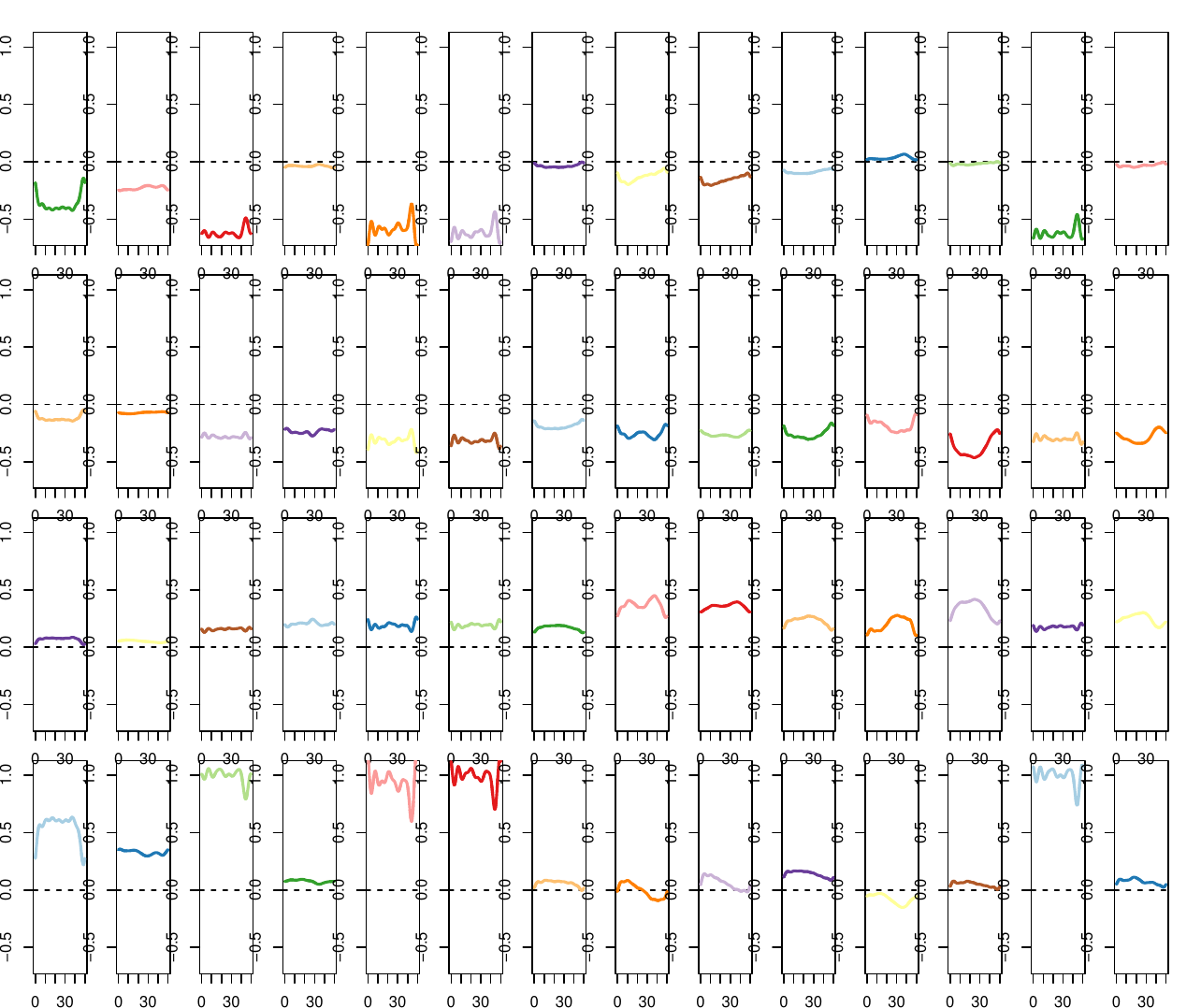}
\caption{
Mean curves for the 45 biclusters identified by funLBM in the IBDMDB dataset. The 45 biclusters are obtained from the selected funLBM model with 3 sample clusters and 15 feature clusters.
}
\label{fig:funlbm_mean_curve}
\end{figure}

\clearpage

\subsection{Supplementary Tables and Figures}
Table~\ref{tab:ibd_species_distribution} summarizes the distribution of bacterial species labels among the 200 retained species-stratified metagenomic pathways used in the IBDMDB analysis. The most frequent species label was \textit{Alistipes putredinis}, accounting for 65 features (32.5\%), followed by \textit{Ruminococcus torques} with 49 features (24.5\%) and unclassified labels with 46 features (23.0\%). This distribution provides additional context for the feature groups identified in the real data analysis, where pathway groups were largely driven by features stratified to these frequently represented species labels.
\begin{table}[ht]
\centering
\small
\setlength{\tabcolsep}{6pt}
\begin{tabular}{lcc}
\hline
Species label in feature name & Number of features & Percentage (\%) \\
\hline
\textit{Ruminococcus torques} & 49 & 24.5 \\
Unclassified & 46 & 23.0 \\
\textit{Alistipes putredinis} & 65 & 32.5 \\
\textit{Ruminococcus bicirculans} & 12 & 6.0 \\
No explicit species label & 10 & 5.0 \\
\textit{Alistipes finegoldii} & 9 & 4.5 \\
\textit{Bacteroides xylanisolvens} & 3 & 1.5 \\
Other single-occurrence species & 6 & 3.0 \\
\hline
Total & 200 & 100.0 \\
\hline
\end{tabular}
\caption{Distribution of bacterial species labels among the 200 retained species-stratified metagenomic pathways in the IBDMDB data. Species names are reported as they appear in the feature names. ``Unclassified'' denotes features ending with an unclassified label, and ``No explicit species label'' denotes features without a terminal species annotation. The category ``Other single-occurrence species'' includes \textit{Akkermansia muciniphila}, \textit{Roseburia hominis}, \textit{Roseburia inulinivorans}, \textit{Roseburia inulinivorans} CAG 15, \textit{Eubacterium eligens}, and \textit{Eubacterium rectale}.}
\label{tab:ibd_species_distribution}
\end{table}

Figure~\ref{fig:ibd_feature_curves} provides additional details on the estimated singular functions for the first three sparse layers in the IBDMDB analysis. For a given species-stratified pathway, a curve segment that departs from the zero baseline indicates that the pathway is selected in the corresponding sparse layer, with the nonzero segment reflecting its estimated temporal contribution. In contrast, a curve that remains on or near the zero baseline over the entire domain indicates that the pathway is not selected in that layer.
\begin{figure}[p]
    \centering
    \includegraphics[page=1,width=\textwidth]{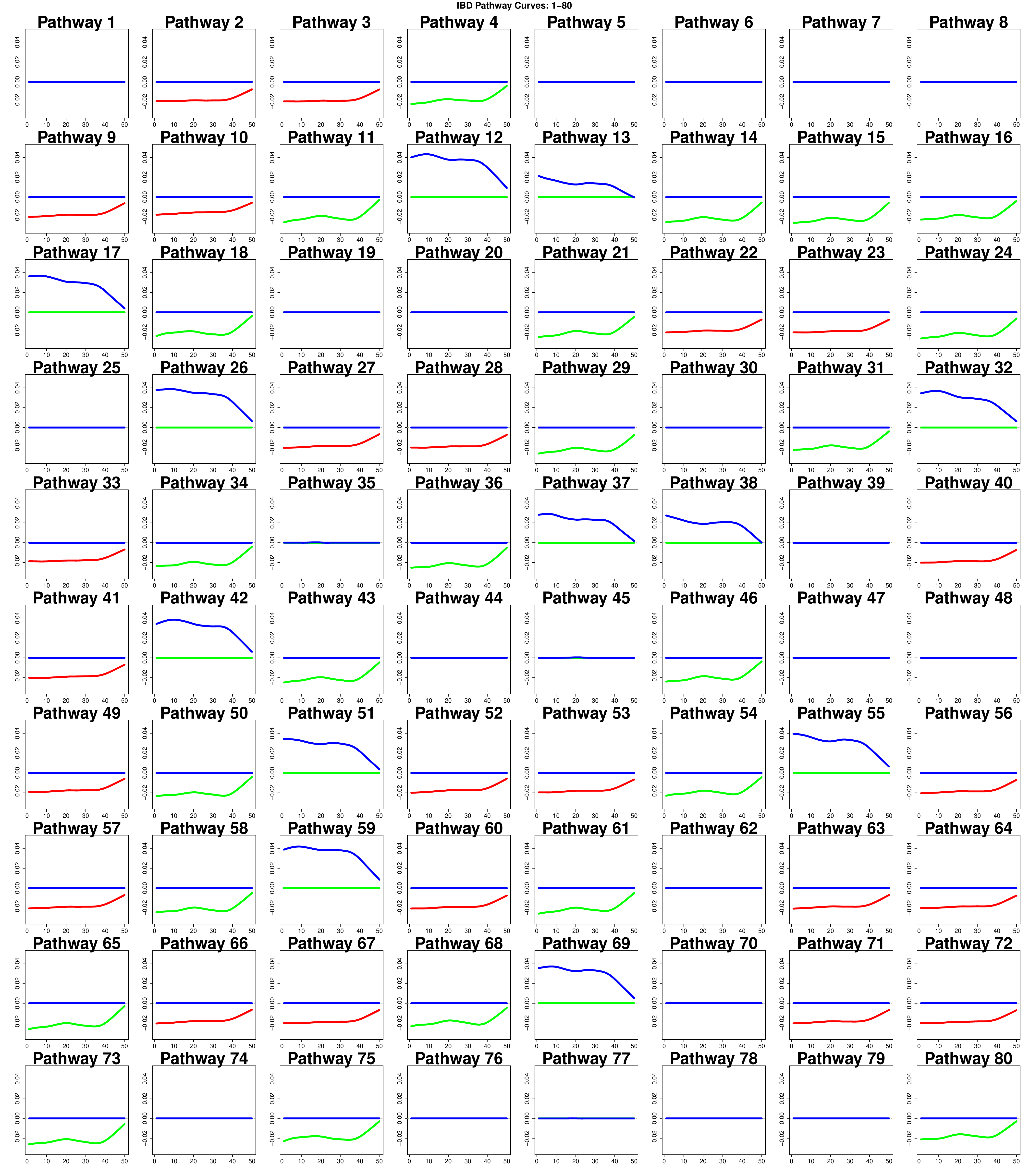}
    \caption{(continued)}
    \label{fig:ibd_feature_curves}
\end{figure}

\begin{figure}[p]\ContinuedFloat
    \centering
    \includegraphics[page=2,width=\textwidth]{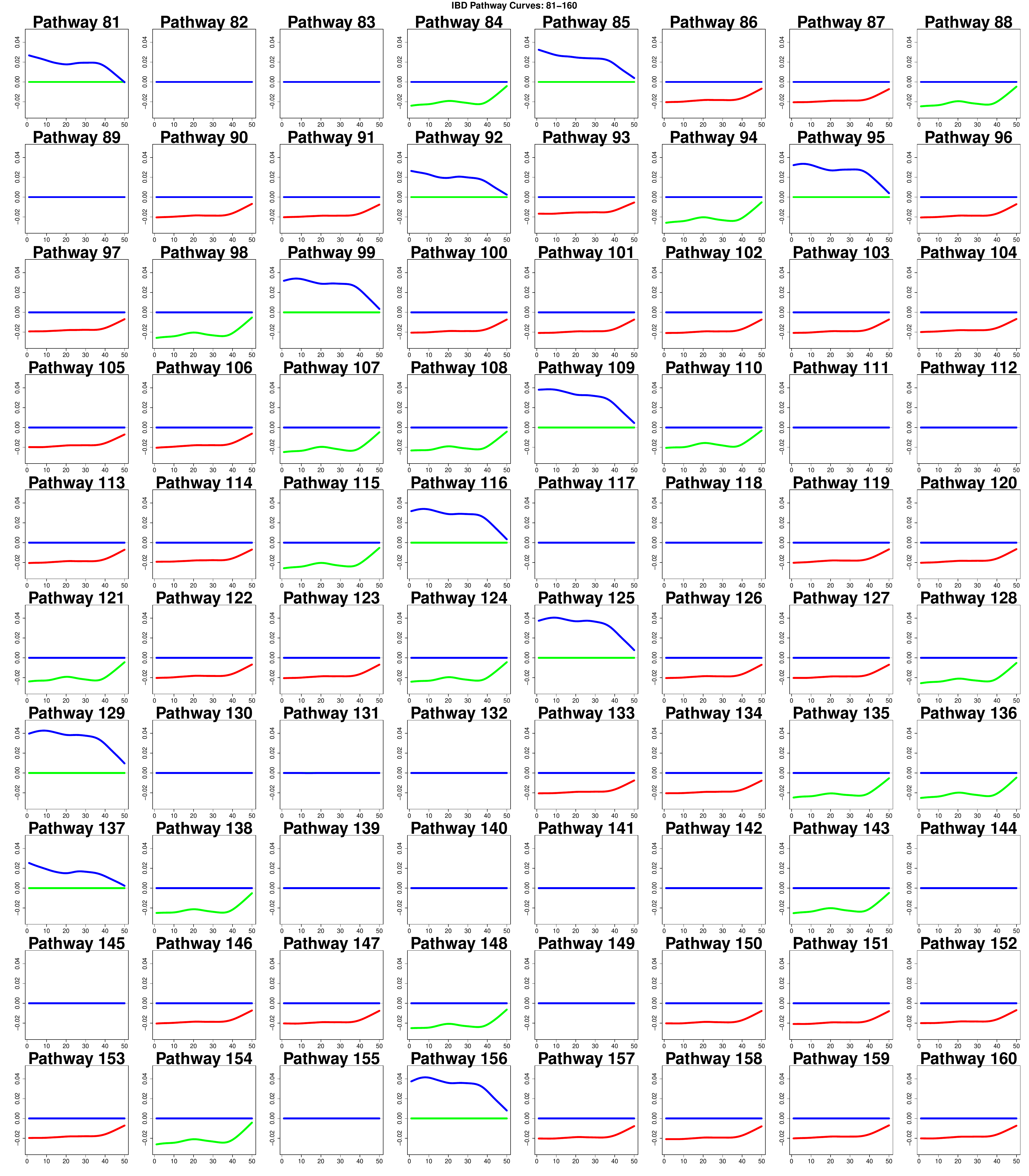}
    \caption[]{ (continued).}
\end{figure}

\begin{figure}[p]\ContinuedFloat
    \centering
    \includegraphics[page=3,width=\textwidth]{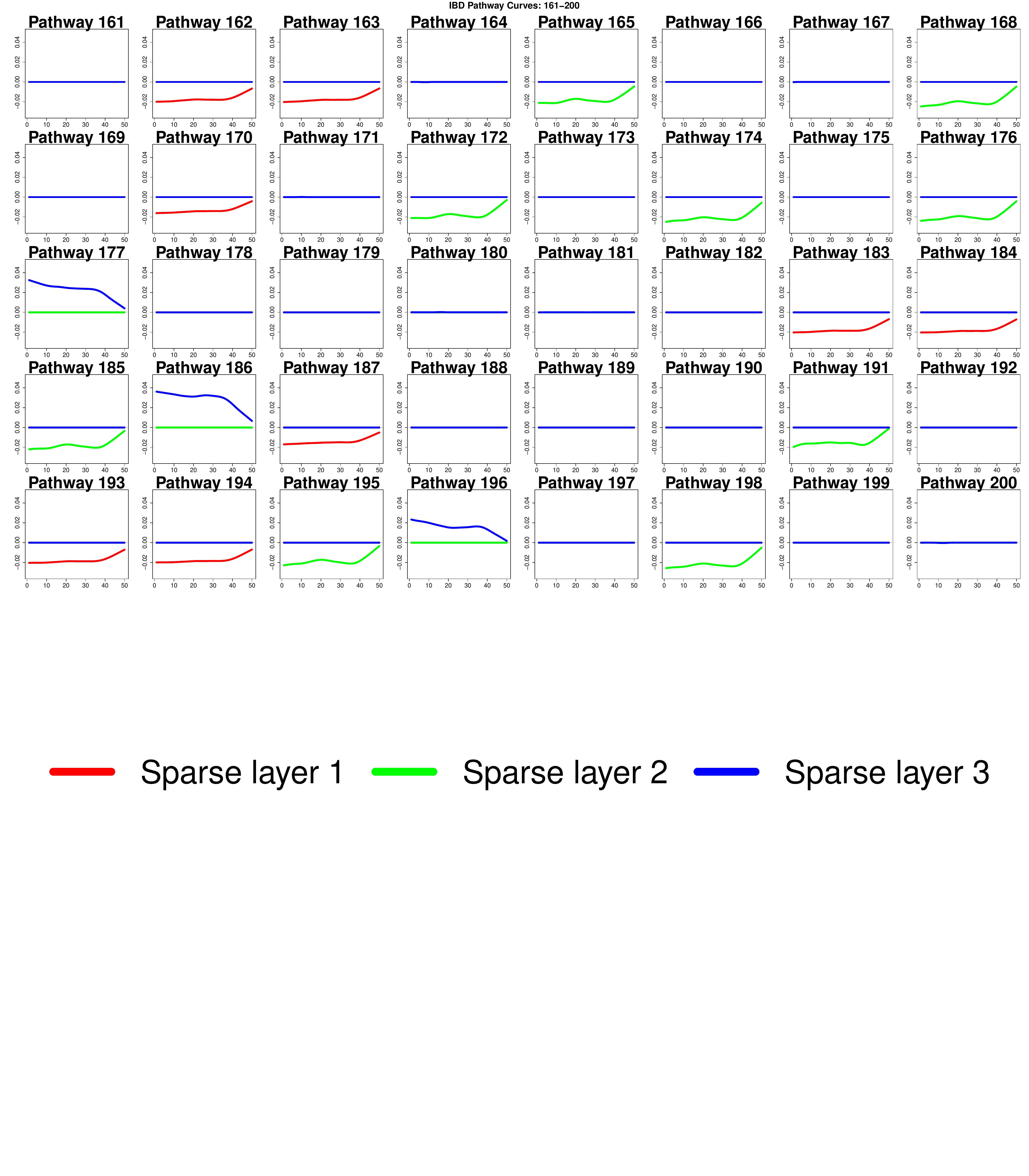}
    \vspace{-5cm}
    \caption[]{Estimated pathway loading curves for the first three sparse layers. Each panel corresponds to one microbial pathway feature, and the three colored curves represent its estimated loading function in sparse layers 1, 2, and 3, respectively. For a given pathway, a curve that remains on or very close to the zero baseline indicates that the pathway is not selected by the corresponding sparse layer, whereas a curve that deviates from zero indicates that the pathway is selected by that layer. The magnitude and sign of the deviation reflect the estimated contribution pattern of that pathway over time. Panels in which all three curves remain at or near zero indicate pathways that are not meaningfully selected by any of the first three sparse layers.}
\end{figure}

\clearpage
To provide a more complete visualization of the identified IBD biclusters,
Figures~\ref{fig:supp_S1F1}--\ref{fig:supp_S3F3} display the pathway-specific
longitudinal patterns for all pathways within each selected pathway group. In each supplementary figure, black
points denote the raw sample mean measurements within the corresponding sample
cluster, computed after omitting missing observations at each time point, while
colored solid curves denote the reconstructed mean trajectories obtained from our
method. Together, these plots provide detailed evidence of the temporal patterns
underlying the identified biclusters.
\begin{figure}[p]
    \centering
    \includegraphics[page=1,width=\textwidth]{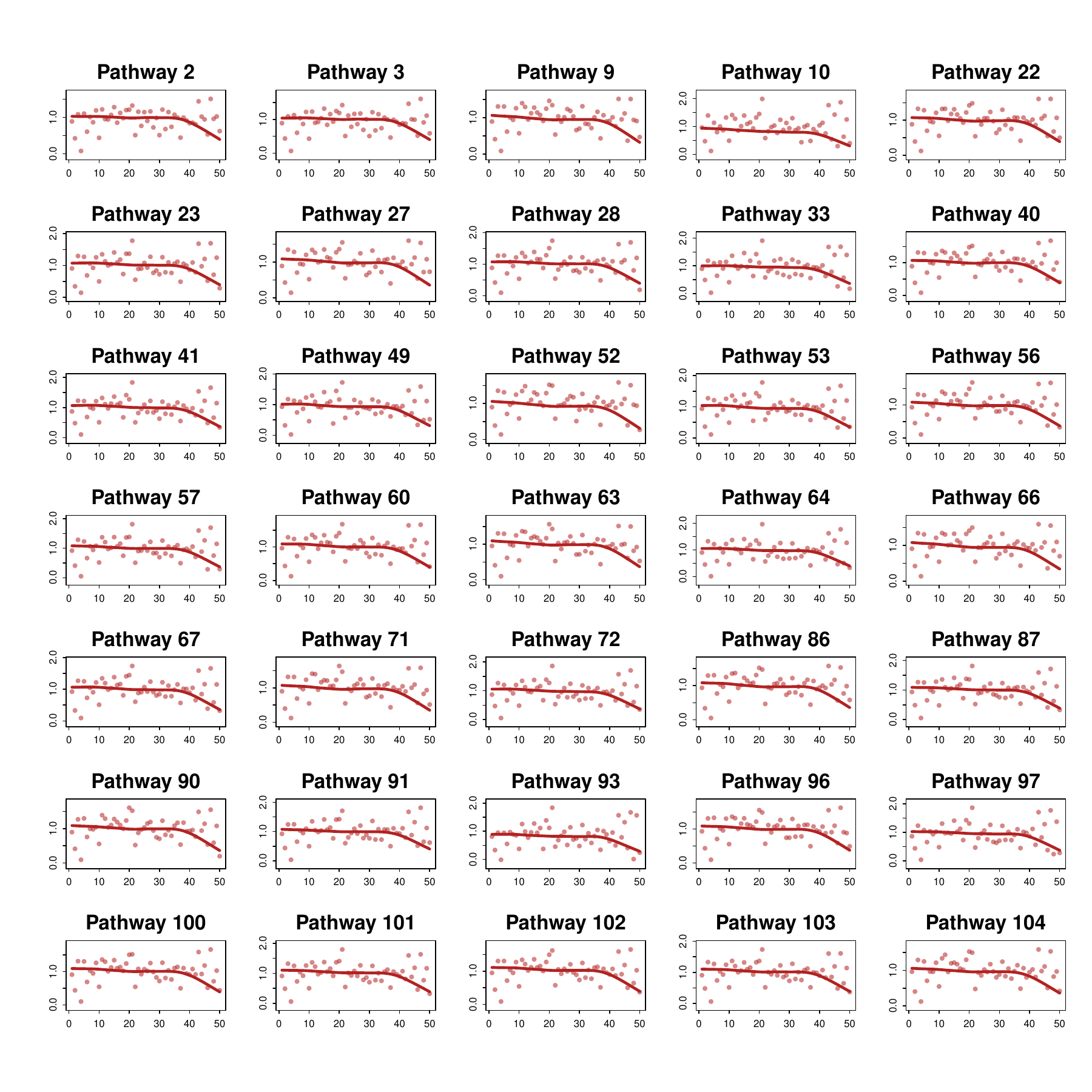}
    \caption{
    Detailed display of all pathways within the identified bicluster \((\pmb{S}_1,\pmb{F}_1)\). Each panel corresponds to one pathway in pathway group \(\pmb{F}_1\). Black points represent raw sample mean measurements among samples in \(\pmb{S}_1\), computed after omitting missing observations at each time point. Colored solid curves represent the reconstructed mean trajectories obtained from our method.
    }
    \label{fig:supp_S1F1}
\end{figure}

\begin{figure}[p]\ContinuedFloat
    \centering
    \includegraphics[page=2,width=\textwidth]{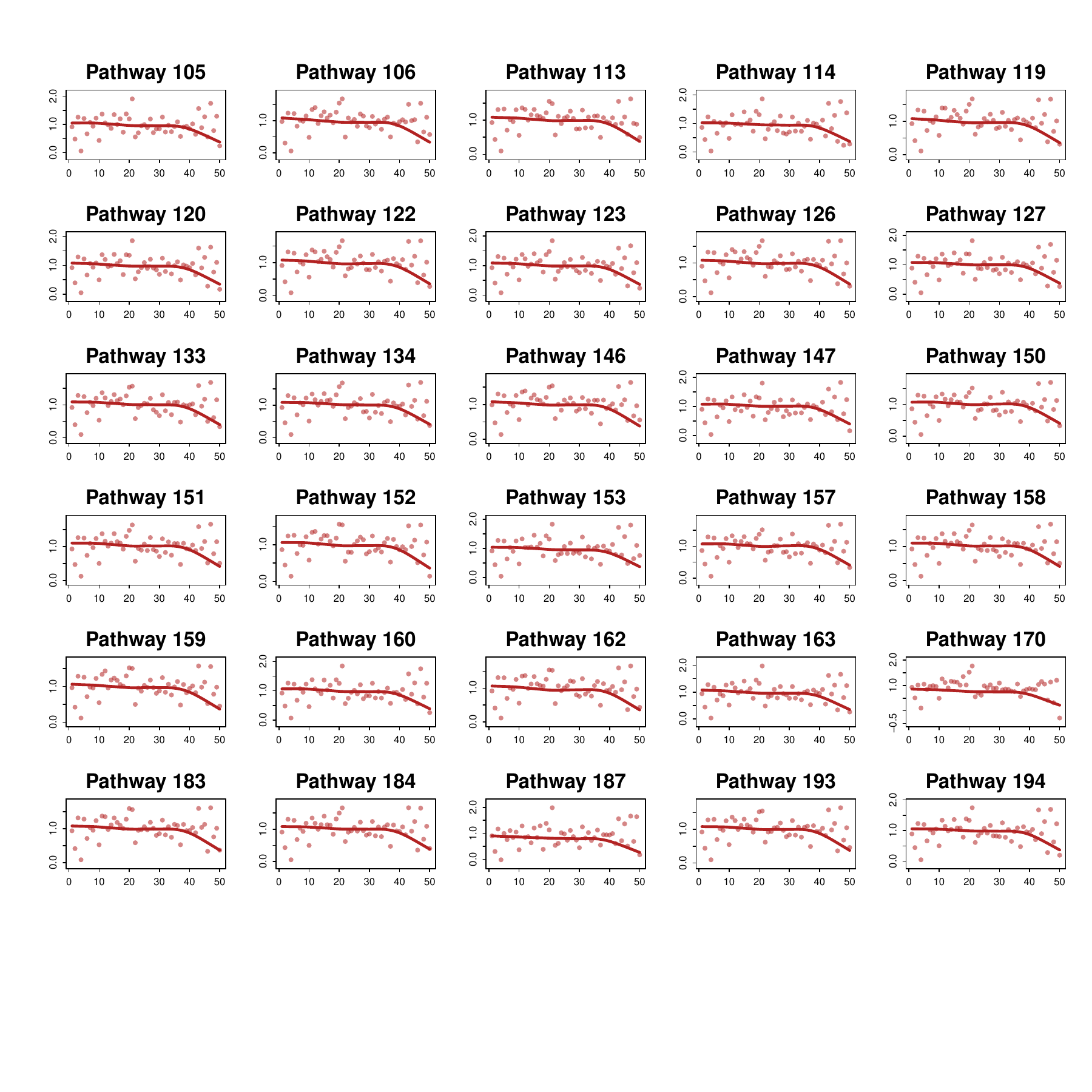}
    \caption[]{Detailed display of all pathways within the identified bicluster \((\pmb{S}_1,\pmb{F}_1)\) continued.}
\end{figure}

\clearpage
\begin{figure}[p]
    \centering
    \includegraphics[page=1,width=\textwidth]{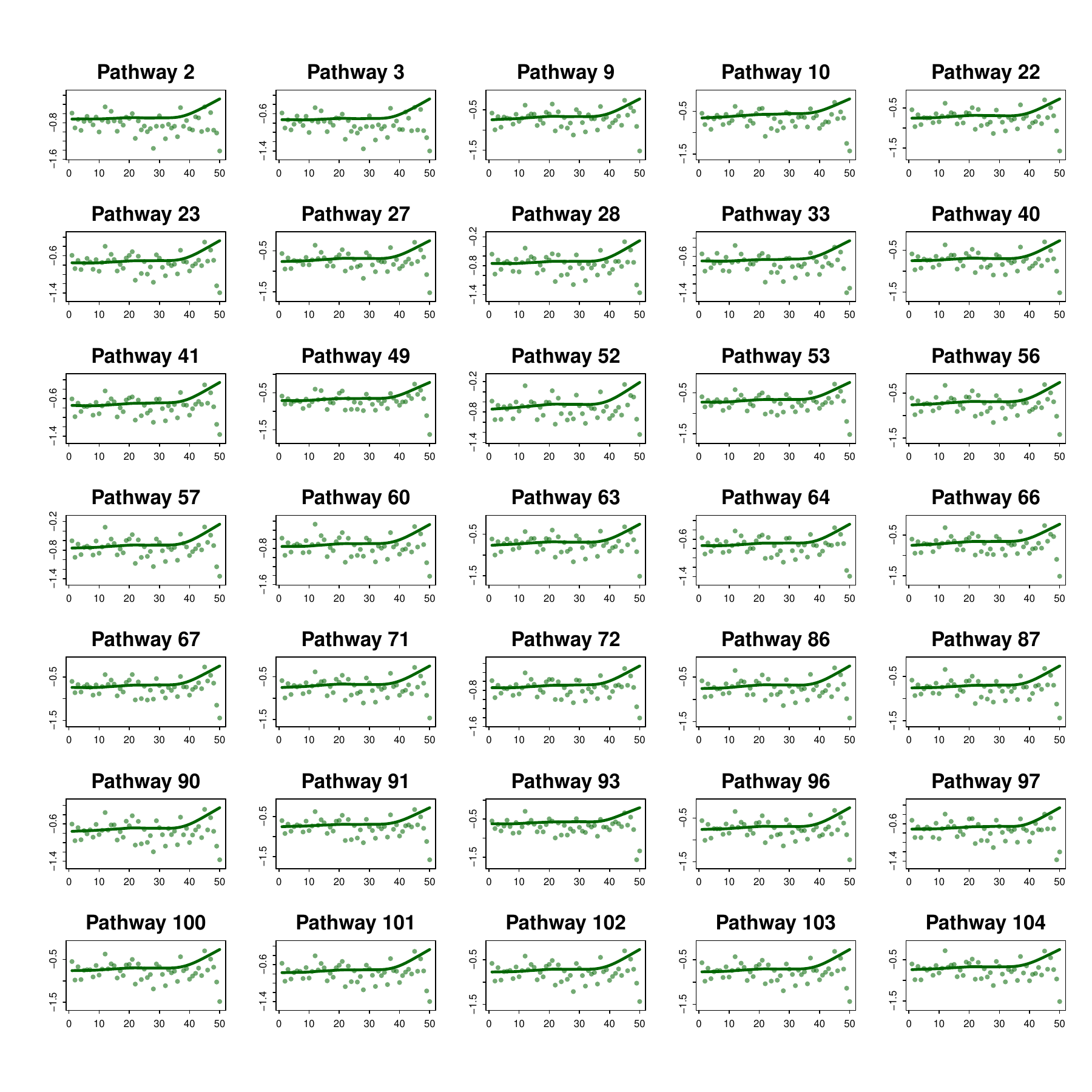}
    \caption{
    Detailed display of all pathways within the identified bicluster \((\pmb{S}_3,\pmb{F}_1)\). Each panel corresponds to one pathway in pathway group \(\pmb{F}_1\). Black points represent raw sample mean measurements among samples in \(\pmb{S}_3\), computed after omitting missing observations at each time point. Colored solid curves represent the reconstructed mean trajectories obtained from our method.
    }
    \label{fig:supp_S3F1}
\end{figure}

\begin{figure}[p]\ContinuedFloat
    \centering
    \includegraphics[page=2,width=\textwidth]{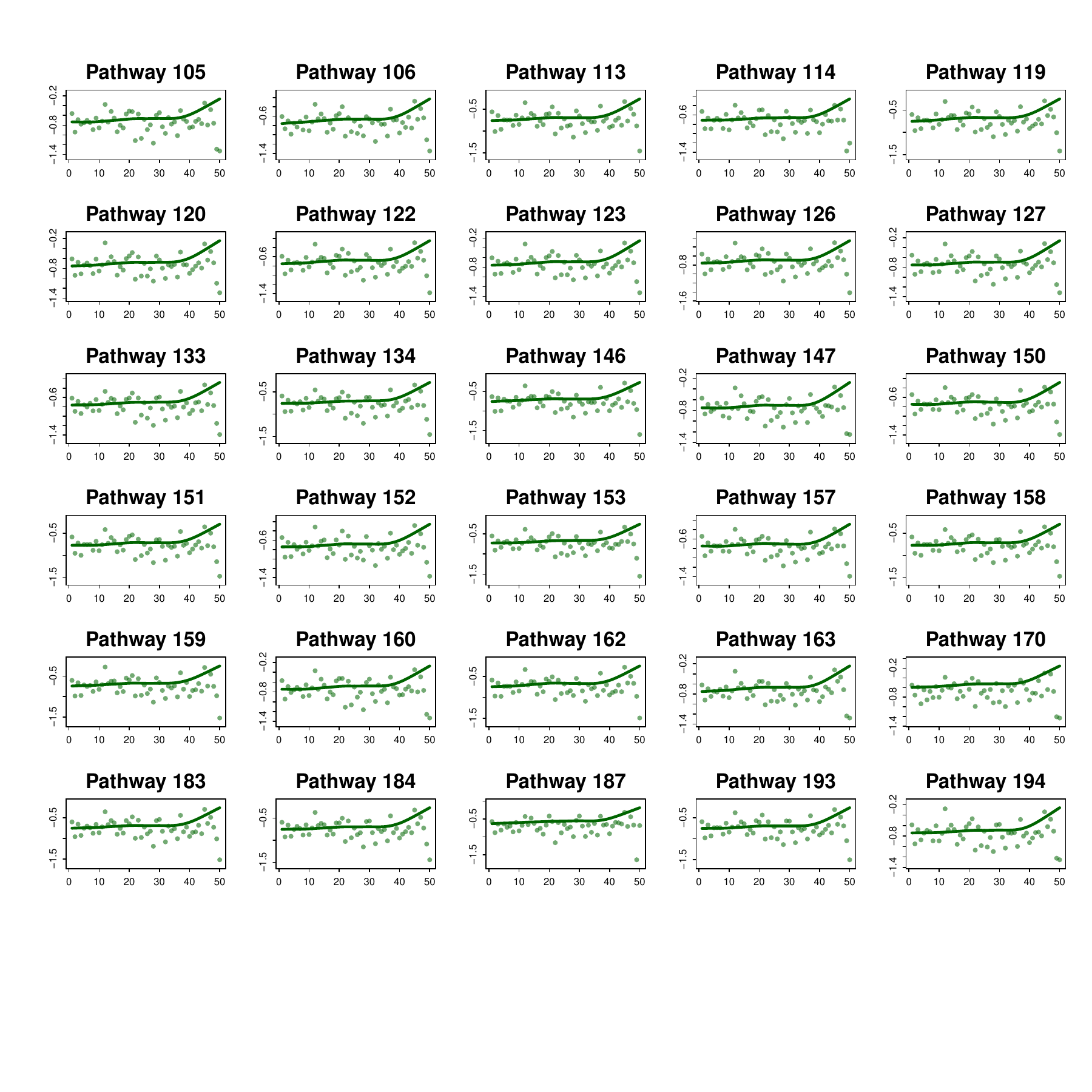}
    \caption[]{Detailed display of all pathways within the identified bicluster \((\pmb{S}_3,\pmb{F}_1)\) continued.}
\end{figure}

\clearpage
\begin{figure}[p]
    \centering
    \includegraphics[page=1,width=\textwidth]{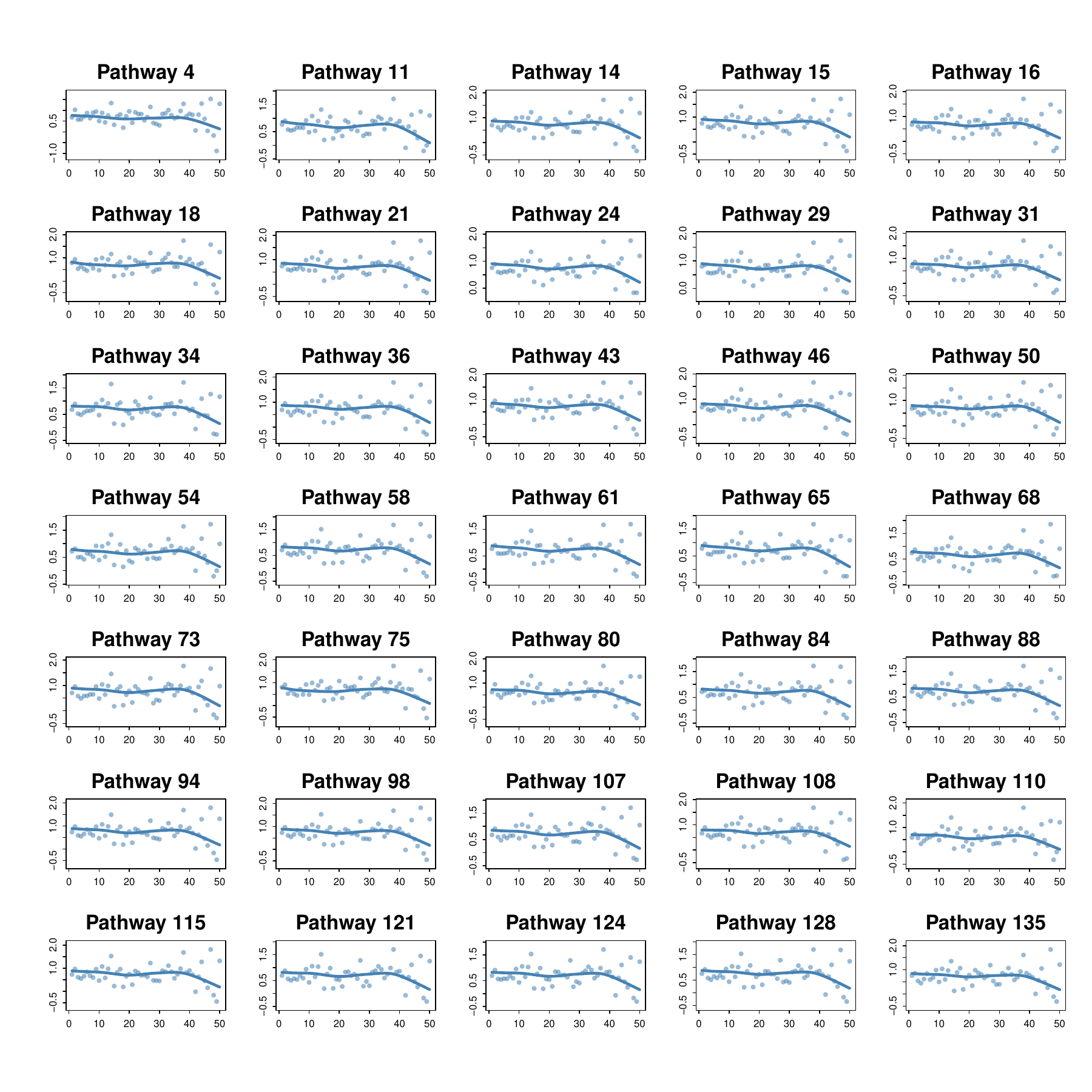}
    \caption{
    Detailed display of all pathways within the identified bicluster \((\pmb{S}_2,\pmb{F}_2)\). Each panel corresponds to one pathway in pathway group \(\pmb{F}_2\). Black points represent raw sample mean measurements among samples in \(\pmb{S}_2\), computed after omitting missing observations at each time point. Colored solid curves represent the reconstructed mean trajectories obtained from our method.
    }
    \label{fig:supp_S2F2}
\end{figure}

\begin{figure}[p]\ContinuedFloat
    \centering
    \includegraphics[page=2,width=\textwidth]{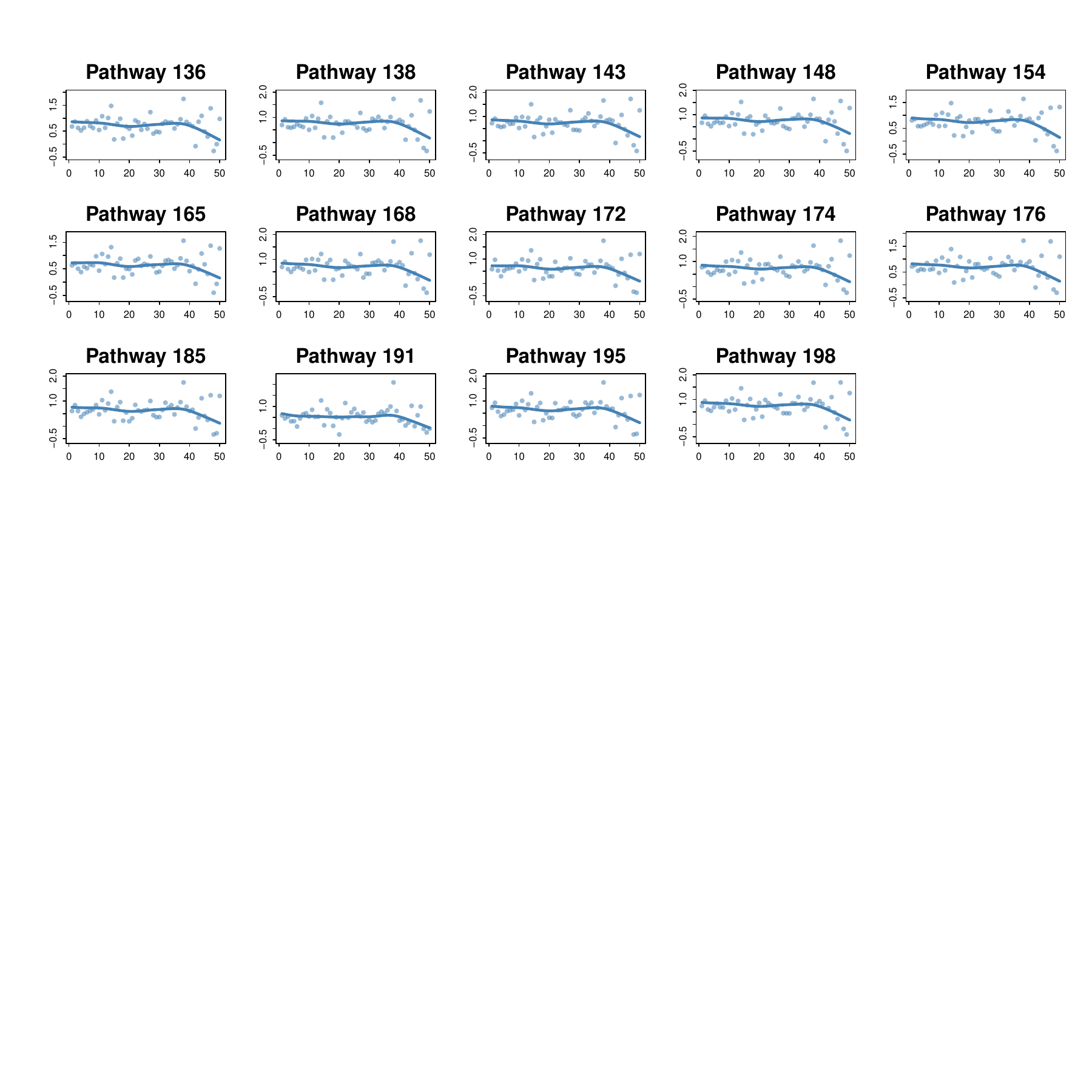}
    \caption[]{Detailed display of all pathways within the identified bicluster \((\pmb{S}_2,\pmb{F}_2)\) continued.}
\end{figure}

\clearpage
\begin{figure}[p]
    \centering
    \includegraphics[page=1,width=\textwidth]{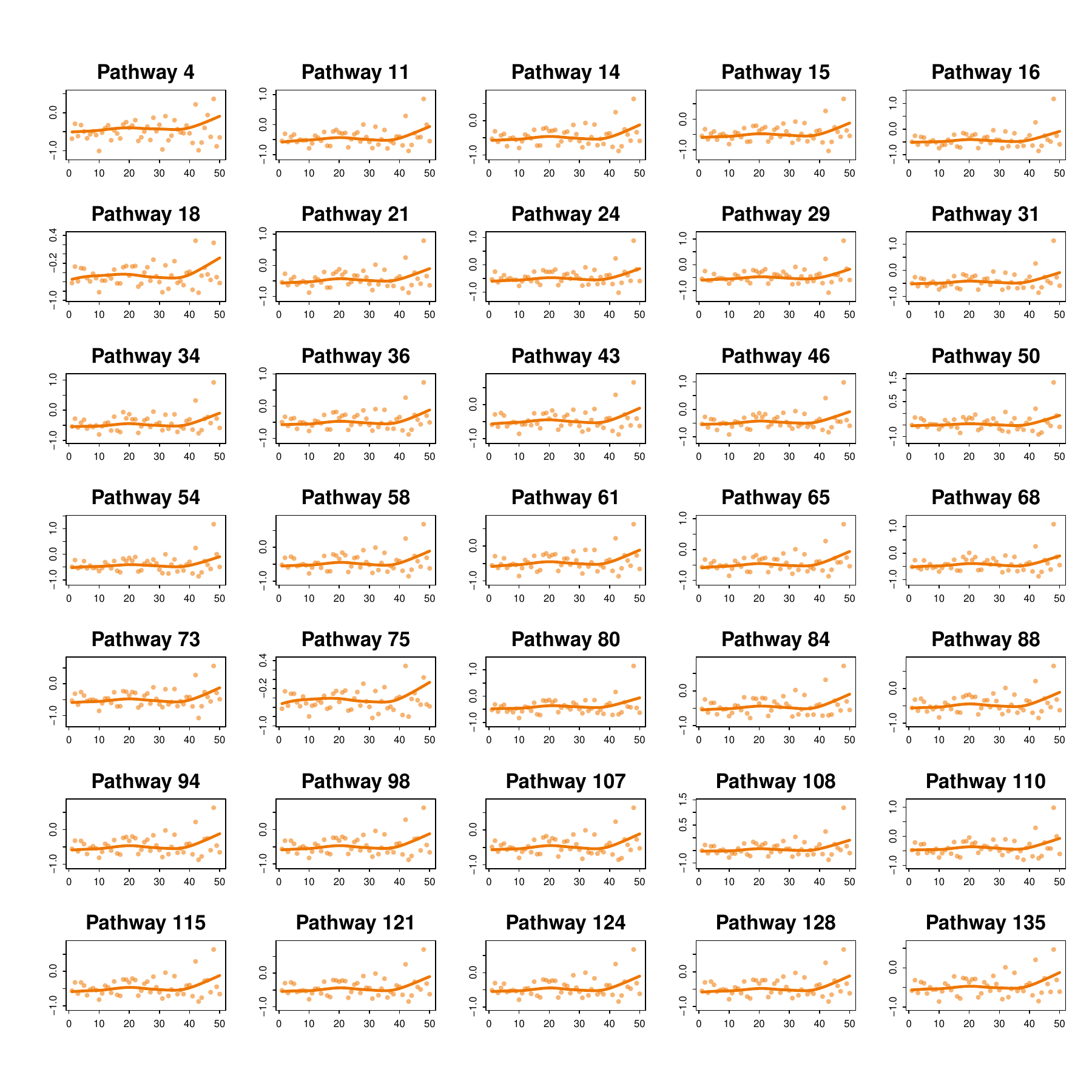}
    \caption{
    Detailed display of all pathways within the identified bicluster \((\pmb{S}_3,\pmb{F}_2)\). Each panel corresponds to one pathway in pathway group \(\pmb{F}_2\). Black points represent raw sample mean measurements among samples in \(\pmb{S}_3\), computed after omitting missing observations at each time point. Colored solid curves represent the reconstructed mean trajectories obtained from our method.
    }
    \label{fig:supp_S3F2}
\end{figure}

\begin{figure}[p]\ContinuedFloat
    \centering
    \includegraphics[page=2,width=\textwidth]{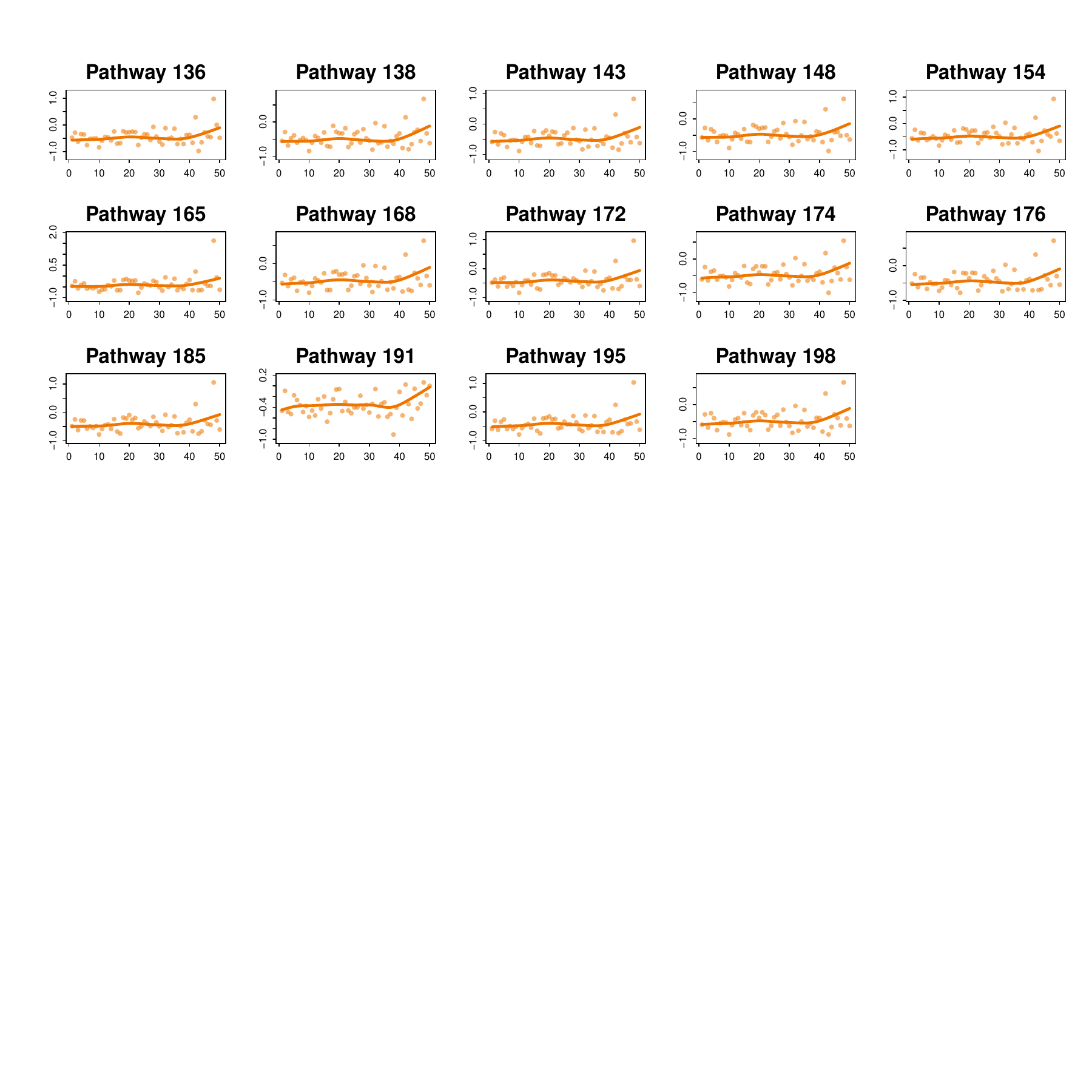}
    \caption[]{Detailed display of all pathways within the identified bicluster \((\pmb{S}_3,\pmb{F}_2)\) continued.}
\end{figure}

\clearpage
\begin{figure}[p]
    \centering
    \includegraphics[page=1,width=\textwidth]{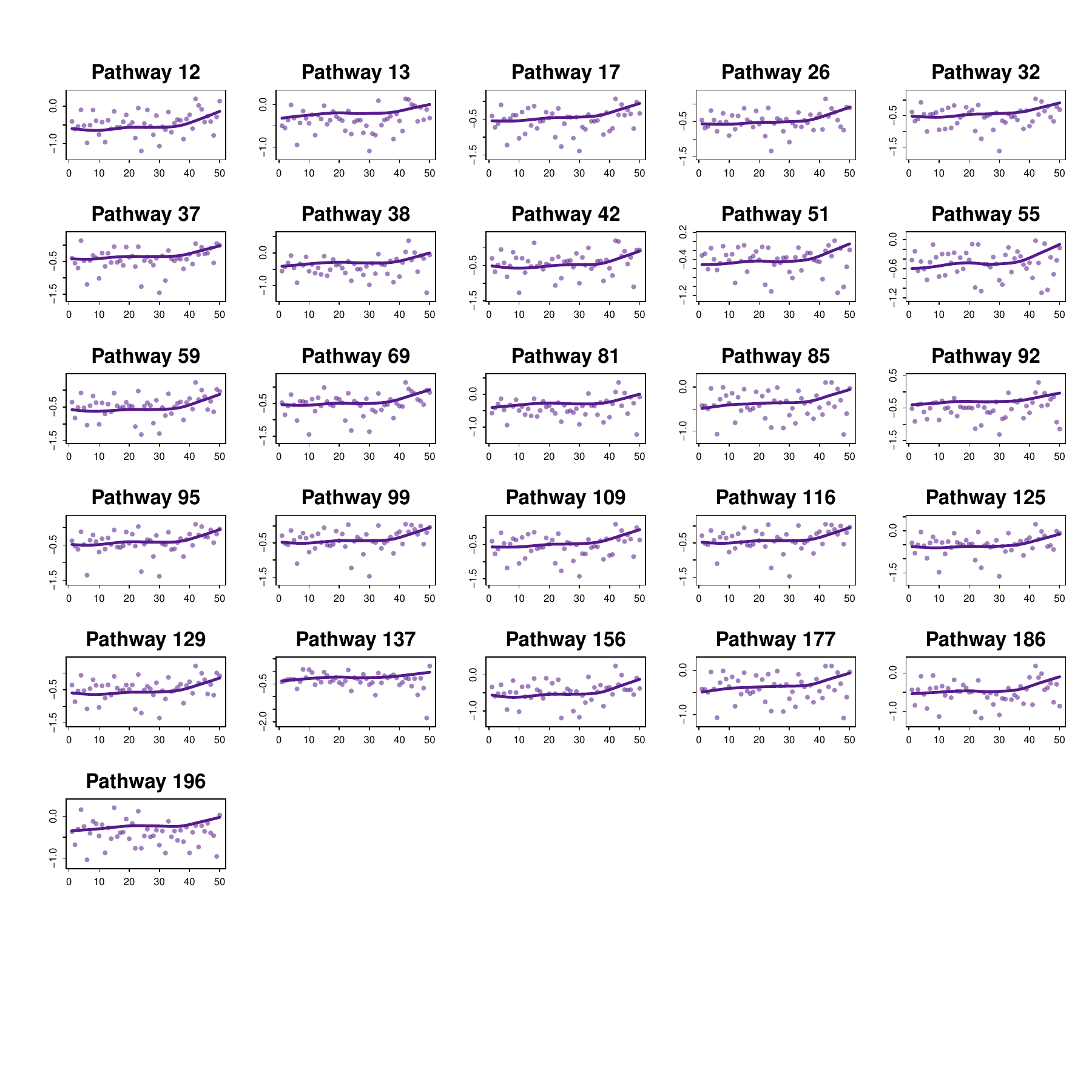}
    \caption{
    Detailed display of all pathways within the identified bicluster \((\pmb{S}_3,\pmb{F}_3)\). Each panel corresponds to one pathway in pathway group \(\pmb{F}_3\). Black points represent raw sample mean measurements among samples in \(\pmb{S}_3\), computed after omitting missing observations at each time point. Colored solid curves represent the reconstructed mean trajectories obtained from our method.
    }
    \label{fig:supp_S3F3}
\end{figure}
\clearpage


\begin{figure}[h]
  \centering
  \includegraphics[width=.9\linewidth]{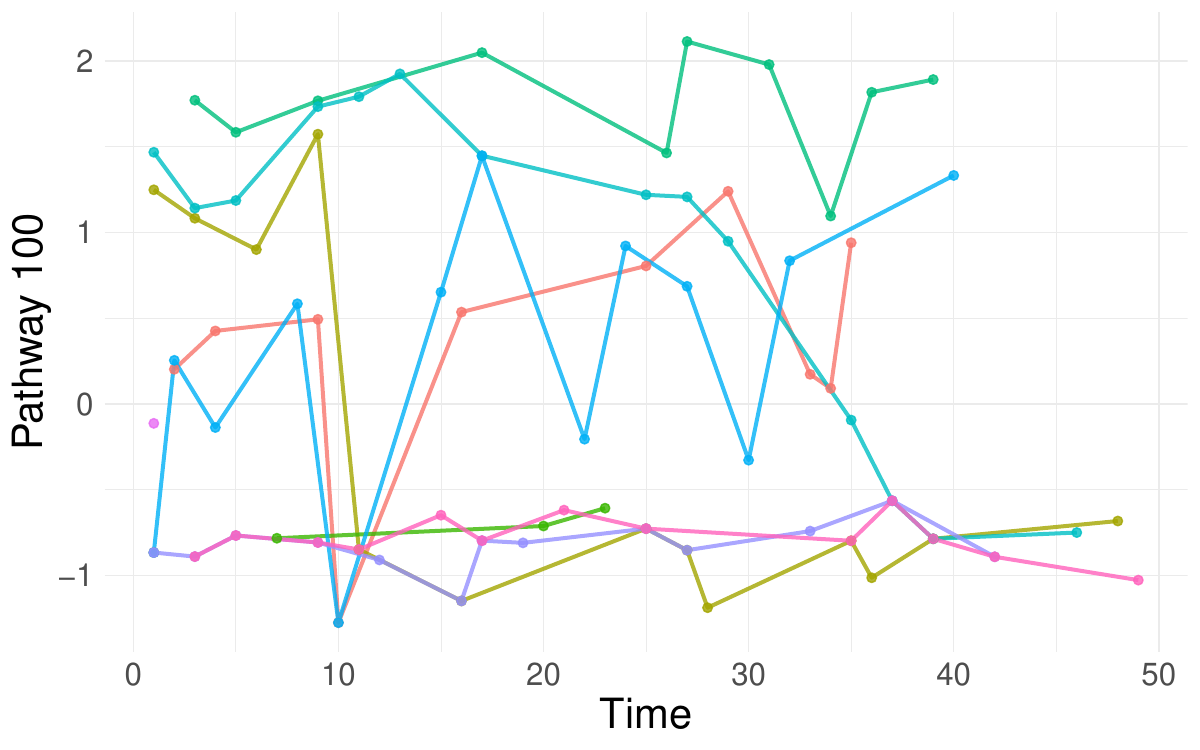}
  \caption{Trajectories for 10 randomly selected samples from Pathway 100.}
  \label{SampleTrajectories}
\end{figure}


\clearpage
\subsection{Features Cluster details of IBD data}

The lists below report the pathway features retained for display in the IBD raw-and-reconstruction figures. They use the same near-zero loading filter as the plotting code; consequently, feature cluster $\pmb{F}_3$ lists 26 plotted features rather than all 38 nonzero-loading features.


\newcolumntype{L}[1]{>{\RaggedRight\arraybackslash}p{#1}}

\newcommand{\featitem}[1]{%
  \textbullet\ \parbox[t]{\linewidth}{#1\par\vspace{0.7em}}\\%
}
\newcommand{\featitemlast}[1]{%
  \textbullet\ \parbox[t]{\linewidth}{#1}\\%
}

\vspace{3mm}
{\Large \textbf{Features included in feature cluster $\pmb{F}_1$ (65 features):}}
\vspace{2mm}
\footnotesize
\setlength{\LTpre}{0pt}
\setlength{\LTpost}{0pt}
\renewcommand{\arraystretch}{1.0}

\begin{longtable}{@{}L{0.95\textwidth}@{}}
\toprule
\textbf{Feature cluster $\pmb{F}_1$} \\
\midrule
\endfirsthead

\toprule
\textbf{Feature cluster $\pmb{F}_1$ (continued)} \\
\midrule
\endhead

\bottomrule
\endfoot

\featitem{\path{UNINTEGRATED.g__Alistipes.s__Alistipes_putredinis}}
\featitem{\path{UNINTEGRATED.g__Alistipes.s__Alistipes_putredinis_CAG_67}}
\featitem{\path{X1CMET2.PWY..N10.formyl.tetrahydrofolate.biosynthesis.g__Alistipes.s__Alistipes_putredinis}}
\featitem{\path{X1CMET2.PWY..N10.formyl.tetrahydrofolate.biosynthesis.g__Alistipes.s__Alistipes_putredinis_CAG_67}}
\featitem{\path{COA.PWY.1..coenzyme.A.biosynthesis.II..mammalian..g__Alistipes.s__Alistipes_putredinis}}
\featitem{\path{COA.PWY.1..coenzyme.A.biosynthesis.II..mammalian..g__Alistipes.s__Alistipes_putredinis_CAG_67}}
\featitem{\path{COA.PWY..coenzyme.A.biosynthesis.I.g__Alistipes.s__Alistipes_putredinis}}
\featitem{\path{COA.PWY..coenzyme.A.biosynthesis.I.g__Alistipes.s__Alistipes_putredinis_CAG_67}}
\featitem{\path{DTDPRHAMSYN.PWY..dTDP.L.rhamnose.biosynthesis.I.g__Alistipes.s__Alistipes_putredinis_CAG_67}}
\featitem{\path{HISDEG.PWY..L.histidine.degradation.I.g__Alistipes.s__Alistipes_putredinis}}
\featitem{\path{HISDEG.PWY..L.histidine.degradation.I.g__Alistipes.s__Alistipes_putredinis_CAG_67}}
\featitem{\path{PANTO.PWY..phosphopantothenate.biosynthesis.I.g__Alistipes.s__Alistipes_putredinis_CAG_67}}
\featitem{\path{PANTOSYN.PWY..pantothenate.and.coenzyme.A.biosynthesis.I.g__Alistipes.s__Alistipes_putredinis}}
\featitem{\path{PANTOSYN.PWY..pantothenate.and.coenzyme.A.biosynthesis.I.g__Alistipes.s__Alistipes_putredinis_CAG_67}}
\featitem{\path{PEPTIDOGLYCANSYN.PWY..peptidoglycan.biosynthesis.I..meso.diaminopimelate.containing..g__Alistipes.s__Alistipes_putredinis}}
\featitem{\path{PEPTIDOGLYCANSYN.PWY..peptidoglycan.biosynthesis.I..meso.diaminopimelate.containing..g__Alistipes.s__Alistipes_putredinis_CAG_67}}
\featitem{\path{PWY.1042..glycolysis.IV..plant.cytosol..g__Alistipes.s__Alistipes_putredinis}}
\featitem{\path{PWY.3841..folate.transformations.II.g__Alistipes.s__Alistipes_putredinis}}
\featitem{\path{PWY.3841..folate.transformations.II.g__Alistipes.s__Alistipes_putredinis_CAG_67}}
\featitem{\path{PWY.4242..pantothenate.and.coenzyme.A.biosynthesis.III.g__Alistipes.s__Alistipes_putredinis}}
\featitem{\path{PWY.4242..pantothenate.and.coenzyme.A.biosynthesis.III.g__Alistipes.s__Alistipes_putredinis_CAG_67}}
\featitem{\path{PWY.5030..L.histidine.degradation.III.g__Alistipes.s__Alistipes_putredinis}}
\featitem{\path{PWY.5030..L.histidine.degradation.III.g__Alistipes.s__Alistipes_putredinis_CAG_67}}
\featitem{\path{PWY.5686..UMP.biosynthesis.g__Alistipes.s__Alistipes_putredinis}}
\featitem{\path{PWY.5686..UMP.biosynthesis.g__Alistipes.s__Alistipes_putredinis_CAG_67}}
\featitem{\path{PWY.5695..urate.biosynthesis.inosine.5..phosphate.degradation.g__Alistipes.s__Alistipes_putredinis}}
\featitem{\path{PWY.5695..urate.biosynthesis.inosine.5..phosphate.degradation.g__Alistipes.s__Alistipes_putredinis_CAG_67}}
\featitem{\path{PWY.6121..5.aminoimidazole.ribonucleotide.biosynthesis.I.g__Alistipes.s__Alistipes_putredinis_CAG_67}}
\featitem{\path{PWY.6122..5.aminoimidazole.ribonucleotide.biosynthesis.II.g__Alistipes.s__Alistipes_putredinis}}
\featitem{\path{PWY.6122..5.aminoimidazole.ribonucleotide.biosynthesis.II.g__Alistipes.s__Alistipes_putredinis_CAG_67}}
\featitem{\path{PWY.6123..inosine.5..phosphate.biosynthesis.I.g__Alistipes.s__Alistipes_putredinis}}
\featitem{\path{PWY.6123..inosine.5..phosphate.biosynthesis.I.g__Alistipes.s__Alistipes_putredinis_CAG_67}}
\featitem{\path{PWY.6124..inosine.5..phosphate.biosynthesis.II.g__Alistipes.s__Alistipes_putredinis}}
\featitem{\path{PWY.6124..inosine.5..phosphate.biosynthesis.II.g__Alistipes.s__Alistipes_putredinis_CAG_67}}
\featitem{\path{PWY.6126..superpathway.of.adenosine.nucleotides.de.novo.biosynthesis.II.g__Alistipes.s__Alistipes_putredinis}}
\featitem{\path{PWY.6126..superpathway.of.adenosine.nucleotides.de.novo.biosynthesis.II.g__Alistipes.s__Alistipes_putredinis_CAG_67}}
\featitem{\path{PWY.6147..6.hydroxymethyl.dihydropterin.diphosphate.biosynthesis.I.g__Alistipes.s__Alistipes_putredinis_CAG_67}}
\featitem{\path{PWY.6277..superpathway.of.5.aminoimidazole.ribonucleotide.biosynthesis.g__Alistipes.s__Alistipes_putredinis}}
\featitem{\path{PWY.6277..superpathway.of.5.aminoimidazole.ribonucleotide.biosynthesis.g__Alistipes.s__Alistipes_putredinis_CAG_67}}
\featitem{\path{PWY.6385..peptidoglycan.biosynthesis.III..mycobacteria..g__Alistipes.s__Alistipes_putredinis}}
\featitem{\path{PWY.6385..peptidoglycan.biosynthesis.III..mycobacteria..g__Alistipes.s__Alistipes_putredinis_CAG_67}}
\featitem{\path{PWY.6386..UDP.N.acetylmuramoyl.pentapeptide.biosynthesis.II..lysine.containing..g__Alistipes.s__Alistipes_putredinis}}
\featitem{\path{PWY.6386..UDP.N.acetylmuramoyl.pentapeptide.biosynthesis.II..lysine.containing..g__Alistipes.s__Alistipes_putredinis_CAG_67}}
\featitem{\path{PWY.6387..UDP.N.acetylmuramoyl.pentapeptide.biosynthesis.I..meso.diaminopimelate.containing..g__Alistipes.s__Alistipes_putredinis}}
\featitem{\path{PWY.6387..UDP.N.acetylmuramoyl.pentapeptide.biosynthesis.I..meso.diaminopimelate.containing..g__Alistipes.s__Alistipes_putredinis_CAG_67}}
\featitem{\path{PWY.6609..adenine.and.adenosine.salvage.III.g__Alistipes.s__Alistipes_putredinis}}
\featitem{\path{PWY.6609..adenine.and.adenosine.salvage.III.g__Alistipes.s__Alistipes_putredinis_CAG_67}}
\featitem{\path{PWY.7219..adenosine.ribonucleotides.de.novo.biosynthesis.g__Alistipes.s__Alistipes_putredinis}}
\featitem{\path{PWY.7219..adenosine.ribonucleotides.de.novo.biosynthesis.g__Alistipes.s__Alistipes_putredinis_CAG_67}}
\featitem{\path{PWY.7220..adenosine.deoxyribonucleotides.de.novo.biosynthesis.II.g__Alistipes.s__Alistipes_putredinis}}
\featitem{\path{PWY.7220..adenosine.deoxyribonucleotides.de.novo.biosynthesis.II.g__Alistipes.s__Alistipes_putredinis_CAG_67}}
\featitem{\path{PWY.7221..guanosine.ribonucleotides.de.novo.biosynthesis.g__Alistipes.s__Alistipes_putredinis}}
\featitem{\path{PWY.7221..guanosine.ribonucleotides.de.novo.biosynthesis.g__Alistipes.s__Alistipes_putredinis_CAG_67}}
\featitem{\path{PWY.7222..guanosine.deoxyribonucleotides.de.novo.biosynthesis.II.g__Alistipes.s__Alistipes_putredinis}}
\featitem{\path{PWY.7222..guanosine.deoxyribonucleotides.de.novo.biosynthesis.II.g__Alistipes.s__Alistipes_putredinis_CAG_67}}
\featitem{\path{PWY.7229..superpathway.of.adenosine.nucleotides.de.novo.biosynthesis.I.g__Alistipes.s__Alistipes_putredinis}}
\featitem{\path{PWY.7229..superpathway.of.adenosine.nucleotides.de.novo.biosynthesis.I.g__Alistipes.s__Alistipes_putredinis_CAG_67}}
\featitem{\path{PWY.7234..inosine.5..phosphate.biosynthesis.III.g__Alistipes.s__Alistipes_putredinis}}
\featitem{\path{PWY.7234..inosine.5..phosphate.biosynthesis.III.g__Alistipes.s__Alistipes_putredinis_CAG_67}}
\featitem{\path{PWY.7539..6.hydroxymethyl.dihydropterin.diphosphate.biosynthesis.III..Chlamydia..g__Alistipes.s__Alistipes_putredinis_CAG_67}}
\featitem{\path{PYRIDNUCSYN.PWY..NAD.biosynthesis.I..from.aspartate..g__Alistipes.s__Alistipes_putredinis}}
\featitem{\path{PYRIDNUCSYN.PWY..NAD.biosynthesis.I..from.aspartate..g__Alistipes.s__Alistipes_putredinis_CAG_67}}
\featitem{\path{SER.GLYSYN.PWY..superpathway.of.L.serine.and.glycine.biosynthesis.I.g__Alistipes.s__Alistipes_putredinis_CAG_67}}
\featitem{\path{TRNA.CHARGING.PWY..tRNA.charging.g__Alistipes.s__Alistipes_putredinis}}
\featitemlast{\path{TRNA.CHARGING.PWY..tRNA.charging.g__Alistipes.s__Alistipes_putredinis_CAG_67}}

\end{longtable}

\vspace{2mm}
{\Large \textbf{Features included in feature cluster $\pmb{F}_2$ (49 features):}}
\vspace{3mm}
\footnotesize
\setlength{\LTpre}{0pt}
\setlength{\LTpost}{0pt}
\renewcommand{\arraystretch}{1.0}

\begin{longtable}{@{}L{0.95\textwidth}@{}}
\toprule
\textbf{Feature cluster $\pmb{F}_2$} \\
\midrule
\endfirsthead

\toprule
\textbf{Feature cluster $\pmb{F}_2$ (continued)} \\
\midrule
\endhead

\bottomrule
\endfoot

\featitem{\path{UNINTEGRATED.g__Blautia.s__Ruminococcus_torques}}
\featitem{\path{X1CMET2.PWY..N10.formyl.tetrahydrofolate.biosynthesis.g__Blautia.s__Ruminococcus_torques}}
\featitem{\path{ARGSYN.PWY..L.arginine.biosynthesis.I..via.L.ornithine..g__Blautia.s__Ruminococcus_torques}}
\featitem{\path{ARGSYNBSUB.PWY..L.arginine.biosynthesis.II..acetyl.cycle..g__Blautia.s__Ruminococcus_torques}}
\featitem{\path{ARO.PWY..chorismate.biosynthesis.I.g__Blautia.s__Ruminococcus_torques}}
\featitem{\path{BRANCHED.CHAIN.AA.SYN.PWY..superpathway.of.branched.amino.acid.biosynthesis.g__Blautia.s__Ruminococcus_torques}}
\featitem{\path{CALVIN.PWY..Calvin.Benson.Bassham.cycle.g__Blautia.s__Ruminococcus_torques}}
\featitem{\path{COA.PWY.1..coenzyme.A.biosynthesis.II..mammalian..g__Blautia.s__Ruminococcus_torques}}
\featitem{\path{COA.PWY..coenzyme.A.biosynthesis.I.g__Blautia.s__Ruminococcus_torques}}
\featitem{\path{COMPLETE.ARO.PWY..superpathway.of.aromatic.amino.acid.biosynthesis.g__Blautia.s__Ruminococcus_torques}}
\featitem{\path{DTDPRHAMSYN.PWY..dTDP.L.rhamnose.biosynthesis.I.g__Blautia.s__Ruminococcus_torques}}
\featitem{\path{GLUTORN.PWY..L.ornithine.biosynthesis.g__Blautia.s__Ruminococcus_torques}}
\featitem{\path{ILEUSYN.PWY..L.isoleucine.biosynthesis.I..from.threonine..g__Blautia.s__Ruminococcus_torques}}
\featitem{\path{NONOXIPENT.PWY..pentose.phosphate.pathway..non.oxidative.branch..g__Blautia.s__Ruminococcus_torques}}
\featitem{\path{PANTO.PWY..phosphopantothenate.biosynthesis.I.g__Blautia.s__Ruminococcus_torques}}
\featitem{\path{PANTOSYN.PWY..pantothenate.and.coenzyme.A.biosynthesis.I.g__Blautia.s__Ruminococcus_torques}}
\featitem{\path{PEPTIDOGLYCANSYN.PWY..peptidoglycan.biosynthesis.I..meso.diaminopimelate.containing..g__Blautia.s__Ruminococcus_torques}}
\featitem{\path{PWY.1042..glycolysis.IV..plant.cytosol..g__Blautia.s__Ruminococcus_torques}}
\featitem{\path{PWY.3841..folate.transformations.II.g__Blautia.s__Ruminococcus_torques}}
\featitem{\path{PWY.4242..pantothenate.and.coenzyme.A.biosynthesis.III.g__Blautia.s__Ruminococcus_torques}}
\featitem{\path{PWY.5100..pyruvate.fermentation.to.acetate.and.lactate.II.g__Blautia.s__Ruminococcus_torques}}
\featitem{\path{PWY.5103..L.isoleucine.biosynthesis.III.g__Blautia.s__Ruminococcus_torques}}
\featitem{\path{PWY.5304..superpathway.of.sulfur.oxidation..Acidianus.ambivalens..g__Blautia.s__Ruminococcus_torques}}
\featitem{\path{PWY.5667..CDP.diacylglycerol.biosynthesis.I.g__Blautia.s__Ruminococcus_torques}}
\featitem{\path{PWY.5686..UMP.biosynthesis.g__Blautia.s__Ruminococcus_torques}}
\featitem{\path{PWY.6121..5.aminoimidazole.ribonucleotide.biosynthesis.I.g__Blautia.s__Ruminococcus_torques}}
\featitem{\path{PWY.6122..5.aminoimidazole.ribonucleotide.biosynthesis.II.g__Blautia.s__Ruminococcus_torques}}
\featitem{\path{PWY.6151..S.adenosyl.L.methionine.cycle.I.g__Blautia.s__Ruminococcus_torques}}
\featitem{\path{PWY.6163..chorismate.biosynthesis.from.3.dehydroquinate.g__Blautia.s__Ruminococcus_torques}}
\featitem{\path{PWY.6168..flavin.biosynthesis.III..fungi..g__Blautia.s__Ruminococcus_torques}}
\featitem{\path{PWY.6277..superpathway.of.5.aminoimidazole.ribonucleotide.biosynthesis.g__Blautia.s__Ruminococcus_torques}}
\featitem{\path{PWY.6385..peptidoglycan.biosynthesis.III..mycobacteria..g__Blautia.s__Ruminococcus_torques}}
\featitem{\path{PWY.6386..UDP.N.acetylmuramoyl.pentapeptide.biosynthesis.II..lysine.containing..g__Blautia.s__Ruminococcus_torques}}
\featitem{\path{PWY.6387..UDP.N.acetylmuramoyl.pentapeptide.biosynthesis.I..meso.diaminopimelate.containing..g__Blautia.s__Ruminococcus_torques}}
\featitem{\path{PWY.6609..adenine.and.adenosine.salvage.III.g__Blautia.s__Ruminococcus_torques}}
\featitem{\path{PWY.6700..queuosine.biosynthesis.g__Blautia.s__Ruminococcus_torques}}
\featitem{\path{PWY.6737..starch.degradation.V.g__Blautia.s__Ruminococcus_torques}}
\featitem{\path{PWY.7111..pyruvate.fermentation.to.isobutanol..engineered..g__Blautia.s__Ruminococcus_torques}}
\featitem{\path{PWY.7219..adenosine.ribonucleotides.de.novo.biosynthesis.g__Blautia.s__Ruminococcus_torques}}
\featitem{\path{PWY.7221..guanosine.ribonucleotides.de.novo.biosynthesis.g__Blautia.s__Ruminococcus_torques}}
\featitem{\path{PWY.7357..thiamin.formation.from.pyrithiamine.and.oxythiamine..yeast..g__Blautia.s__Ruminococcus_torques}}
\featitem{\path{PWY.7400..L.arginine.biosynthesis.IV..archaebacteria..g__Blautia.s__Ruminococcus_torques}}
\featitem{\path{PWY0.1061..superpathway.of.L.alanine.biosynthesis.g__Blautia.s__Ruminococcus_torques}}
\featitem{\path{PWY0.1296..purine.ribonucleosides.degradation.g__Blautia.s__Ruminococcus_torques}}
\featitem{\path{PWY0.1319..CDP.diacylglycerol.biosynthesis.II.g__Blautia.s__Ruminococcus_torques}}
\featitem{\path{RIBOSYN2.PWY..flavin.biosynthesis.I..bacteria.and.plants..g__Blautia.s__Ruminococcus_torques}}
\featitem{\path{THISYNARA.PWY..superpathway.of.thiamin.diphosphate.biosynthesis.III..eukaryotes..g__Blautia.s__Ruminococcus_torques}}
\featitem{\path{TRNA.CHARGING.PWY..tRNA.charging.g__Blautia.s__Ruminococcus_torques}}
\featitemlast{\path{VALSYN.PWY..L.valine.biosynthesis.g__Blautia.s__Ruminococcus_torques}}

\end{longtable}

\vspace{2mm}
{\Large \textbf{Features included in feature cluster $\pmb{F}_3$ (26 features):}}
\vspace{3mm}
\footnotesize
\setlength{\LTpre}{0pt}
\setlength{\LTpost}{0pt}
\renewcommand{\arraystretch}{1.0}

\begin{longtable}{@{}L{0.95\textwidth}@{}}
\toprule
\textbf{Feature cluster $\pmb{F}_3$} \\
\midrule
\endfirsthead

\toprule
\textbf{Feature cluster $\pmb{F}_3$ (continued)} \\
\midrule
\endhead

\bottomrule
\endfoot

\featitem{\path{X1CMET2.PWY..N10.formyl.tetrahydrofolate.biosynthesis.unclassified}}
\featitem{\path{ANAGLYCOLYSIS.PWY..glycolysis.III..from.glucose..unclassified}}
\featitem{\path{ARO.PWY..chorismate.biosynthesis.I.unclassified}}
\featitem{\path{COA.PWY.1..coenzyme.A.biosynthesis.II..mammalian..unclassified}}
\featitem{\path{COMPLETE.ARO.PWY..superpathway.of.aromatic.amino.acid.biosynthesis.unclassified}}
\featitem{\path{GLUTORN.PWY..L.ornithine.biosynthesis.unclassified}}
\featitem{\path{GLYCOLYSIS..glycolysis.I..from.glucose.6.phosphate..unclassified}}
\featitem{\path{HISTSYN.PWY..L.histidine.biosynthesis.unclassified}}
\featitem{\path{PANTO.PWY..phosphopantothenate.biosynthesis.I.unclassified}}
\featitem{\path{PANTOSYN.PWY..pantothenate.and.coenzyme.A.biosynthesis.I.unclassified}}
\featitem{\path{PEPTIDOGLYCANSYN.PWY..peptidoglycan.biosynthesis.I..meso.diaminopimelate.containing..unclassified}}
\featitem{\path{PWY.4242..pantothenate.and.coenzyme.A.biosynthesis.III.unclassified}}
\featitem{\path{PWY.5484..glycolysis.II..from.fructose.6.phosphate..unclassified}}
\featitem{\path{PWY.5667..CDP.diacylglycerol.biosynthesis.I.unclassified}}
\featitem{\path{PWY.5973..cis.vaccenate.biosynthesis.unclassified}}
\featitem{\path{PWY.6121..5.aminoimidazole.ribonucleotide.biosynthesis.I.unclassified}}
\featitem{\path{PWY.6122..5.aminoimidazole.ribonucleotide.biosynthesis.II.unclassified}}
\featitem{\path{PWY.6163..chorismate.biosynthesis.from.3.dehydroquinate.unclassified}}
\featitem{\path{PWY.6277..superpathway.of.5.aminoimidazole.ribonucleotide.biosynthesis.unclassified}}
\featitem{\path{PWY.6386..UDP.N.acetylmuramoyl.pentapeptide.biosynthesis.II..lysine.containing..unclassified}}
\featitem{\path{PWY.6387..UDP.N.acetylmuramoyl.pentapeptide.biosynthesis.I..meso.diaminopimelate.containing..unclassified}}
\featitem{\path{PWY.6703..preQ0.biosynthesis.unclassified}}
\featitem{\path{PWY.7221..guanosine.ribonucleotides.de.novo.biosynthesis.unclassified}}
\featitem{\path{PWY0.1319..CDP.diacylglycerol.biosynthesis.II.unclassified}}
\featitem{\path{RIBOSYN2.PWY..flavin.biosynthesis.I..bacteria.and.plants..unclassified}}
\featitemlast{\path{TRPSYN.PWY..L.tryptophan.biosynthesis.unclassified}}

\end{longtable}

\normalsize
\clearpage
\section{More on Simulation}
\subsection{Supplementary Figures}

\begin{figure}[htbp]
    \centering
    \includegraphics[width=\textwidth]{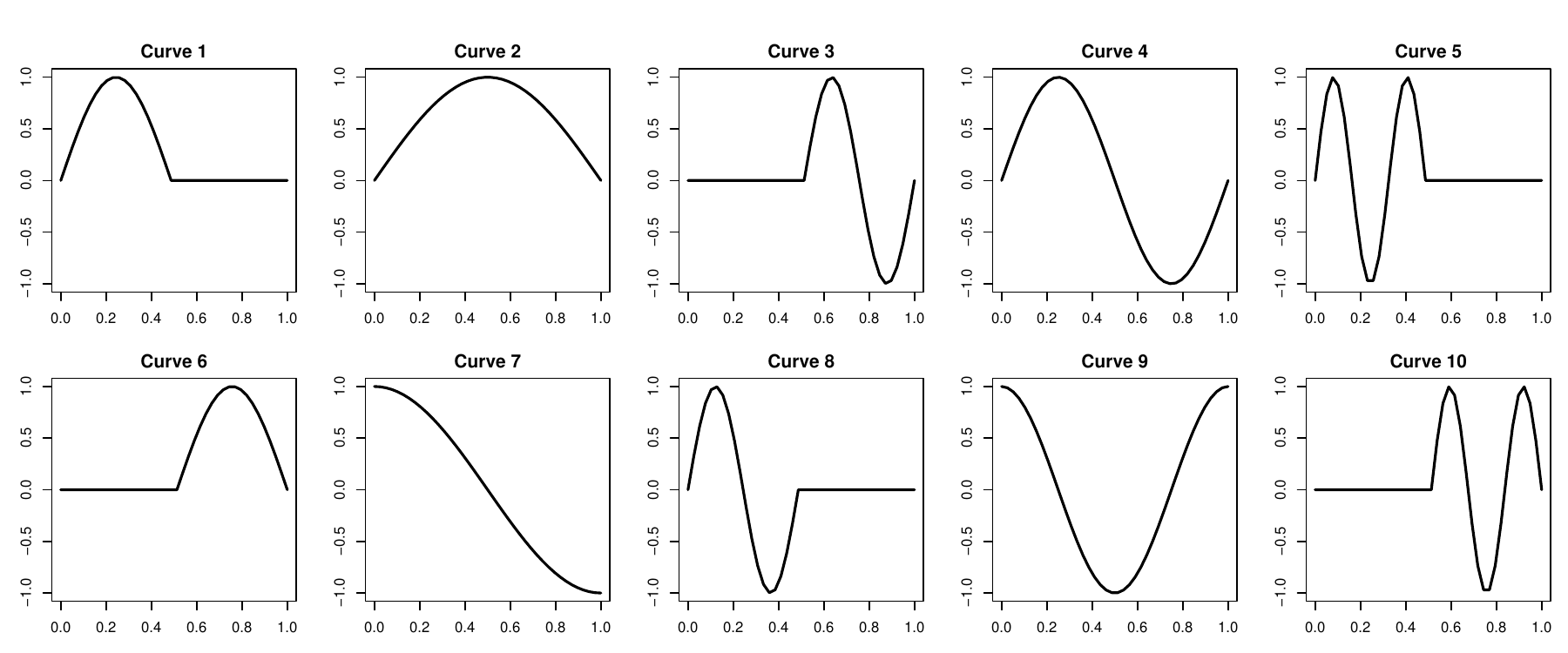}
    \caption{Prespecified candidate functional shapes used to generate active variables in the simulation study. The ten curves comprise two types of signals: full-domain continuous functions defined over the entire interval \([0,1]\) ({e.g., Curves 2, 4, 7, 9}), and subregion-localized functions obtained by restricting the signal to either the left  half ({e.g., Curves 1,5,8}) or the right half ({e.g., Curves 3, 6, 10}) of the discrete grid and setting the remaining half to zero. In each simulation replicate, the loading function for an active variable is selected from this set, while inactive variables are assigned identically zero trajectories.}
    \label{fig:sim_candidate_curves}
\end{figure}

\begin{figure}[h]
  \centering
\label{Supplementary Figures}
  \begin{minipage}{\textwidth}
    \centering
    \begin{subfigure}[b]{0.35\textwidth}
      \centering
      \includegraphics[width=\linewidth]{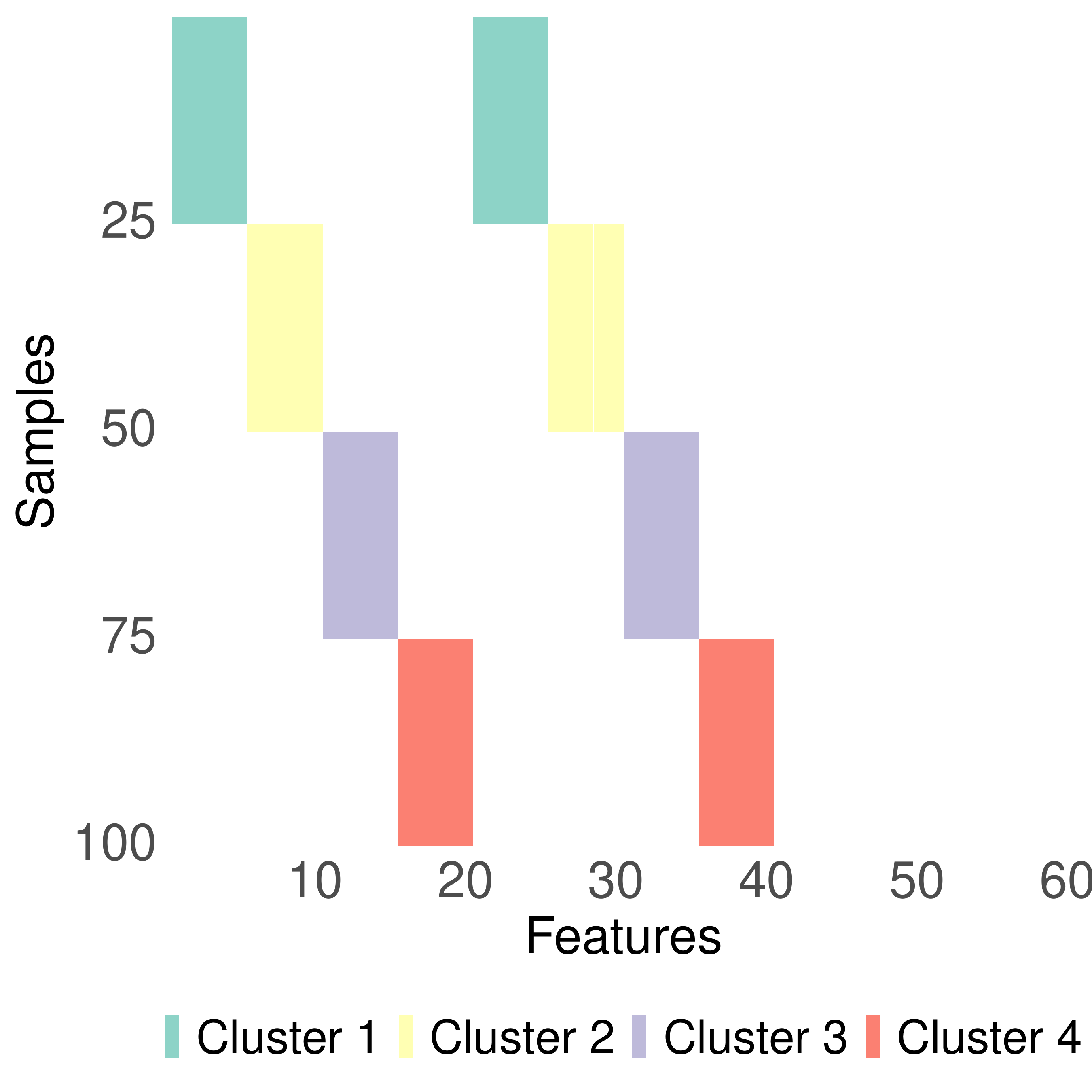}
      \caption{Overlap true cluster}
    \end{subfigure}
    \hspace{0.08\textwidth}
    \begin{subfigure}[b]{0.35\textwidth}
      \centering
      \includegraphics[width=\linewidth]{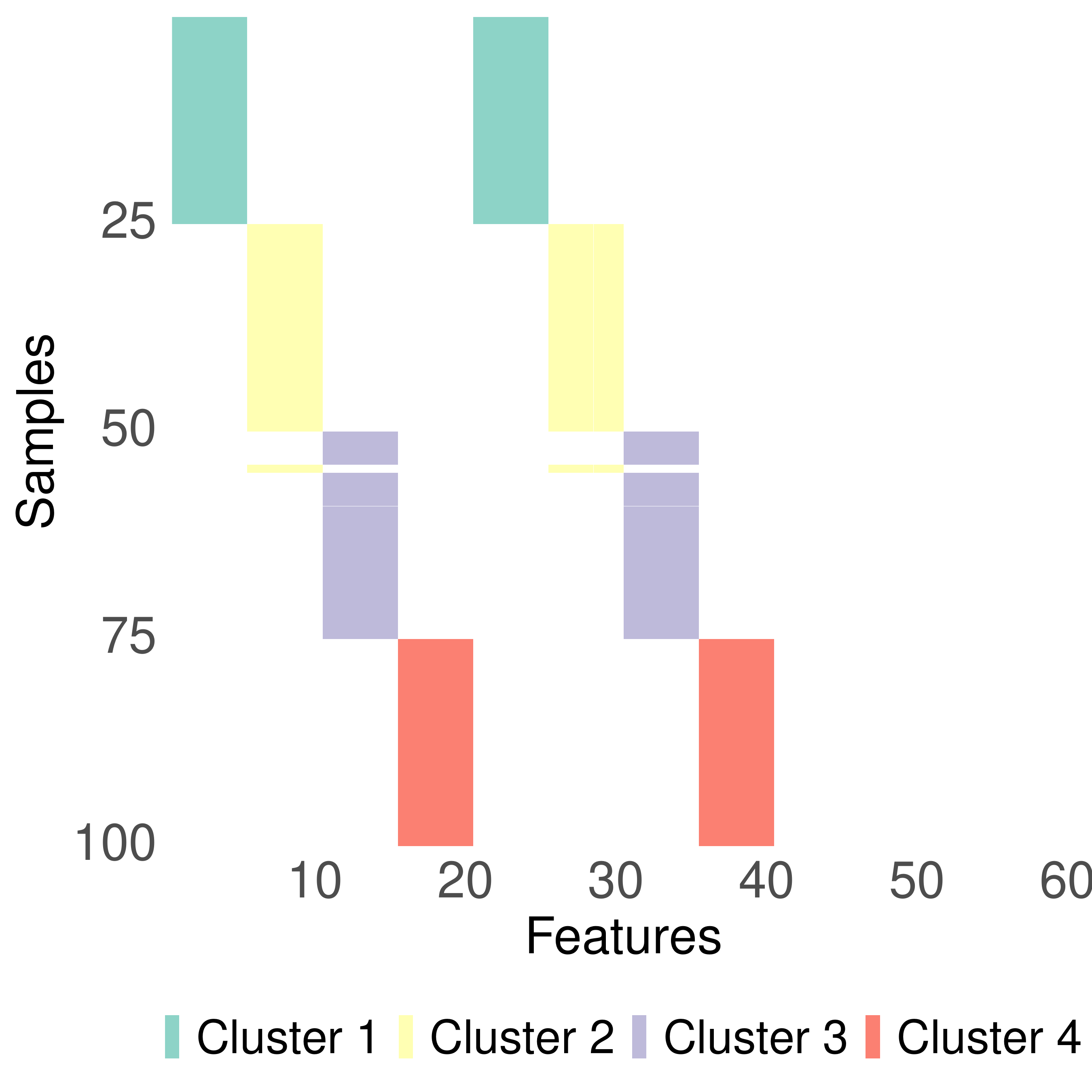}
      \caption{Overlap estimated cluster}
    \end{subfigure}
  \end{minipage}

  \vspace{0.6em}

  \begin{minipage}{\textwidth}
    \centering
    \begin{subfigure}[b]{0.35\textwidth}
      \centering
      \includegraphics[width=\linewidth]{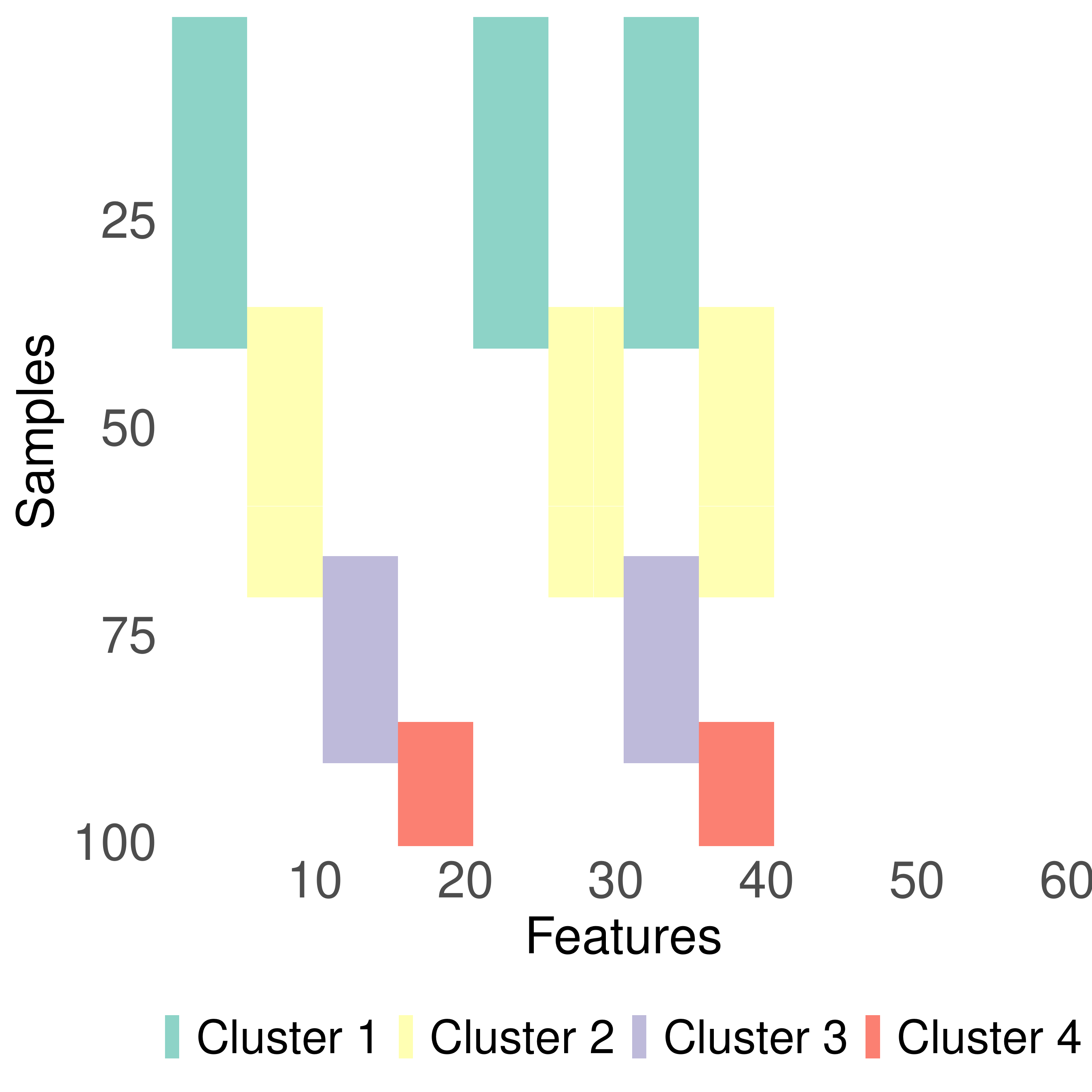}
      \caption{Nonoverlap true cluster}
    \end{subfigure}
    \hspace{0.08\textwidth}
    \begin{subfigure}[b]{0.35\textwidth}
      \centering
      \includegraphics[width=\linewidth]{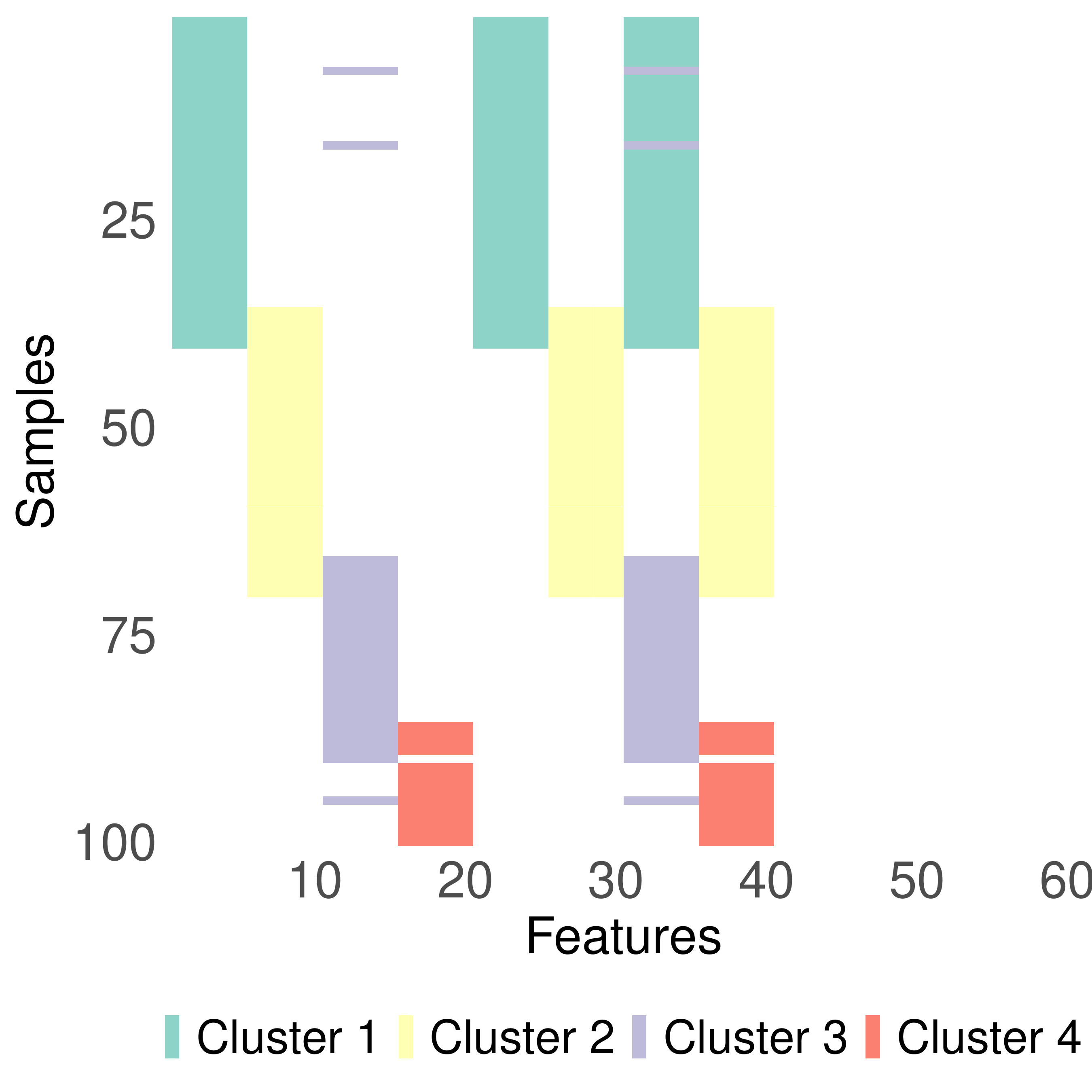}
      \caption{Nonoverlap estimated cluster}
    \end{subfigure}
  \end{minipage}

  \caption{The true and Bi-SfSVD estimated bicluster structure for the overlapping and non-overlapping settings under $p=60$ and a $0.6$ missing rate.}
  \label{clustercomparisons}
\end{figure}

\clearpage
\subsection{Orthogonality under the overlapping setting}
\label{sec:overlap_orthogonality}

In the overlapping setting, different latent components are allowed to share nonzero sample indices in \(\{\pmb{u}^k\}_{k=1}^K\) and nonzero functional regions in \(\{\pmb{\varphi}^k\}_{k=1}^K\). Consequently, orthogonality is no longer automatic and must be handled separately for the sample loading vectors and the right singular functions.

For the sample loading vectors, let
\[
S_k=\{i:\,u_i^k\neq 0\}, \qquad k=1,\dots,K,
\]
denote the nonzero sample support of the \(k\)th component. In the nonoverlapping setting, \(S_j\cap S_k=\emptyset\) for \(j\neq k\), so
\[
\langle \pmb{u}^j,\pmb{u}^k\rangle
=
\sum_{i=1}^n u_i^j u_i^k
=0,
\qquad j\neq k.
\]
After normalization,
\[
\Vert\pmb{u}^k\Vert_2=1,
\]
the vectors are orthonormal.

In the overlapping setting, however, two distinct components, say the \(j\)th and \(k\)th components, may share a subset of sample indices,
\[
O_{jk}=S_j\cap S_k \neq \emptyset.
\]
To reduce the dependence induced by this overlap, we apply a localized Gram--Schmidt-type adjustment on the shared sample indices only. Specifically, restricting attention to the overlapping subset \(O_{jk}\), we update the subvector \(\pmb{u}_{O_{jk}}^k\) by removing its projection onto \(\pmb{u}_{O_{jk}}^j\):
\[
\pmb{u}_{O_{jk}}^k
\leftarrow
\pmb{u}_{O_{jk}}^k
-
\frac{\langle \pmb{u}_{O_{jk}}^j,\pmb{u}_{O_{jk}}^k\rangle}
{\Vert\pmb{u}_{O_{jk}}^j\Vert_2^2}\,
\pmb{u}_{O_{jk}}^j.
\]
Equivalently, this can be written in projection-matrix form as
\[
\pmb{u}_{O_{jk}}^k
\leftarrow
\left(
\pmb{I}
-
\frac{\pmb{u}_{O_{jk}}^j(\pmb{u}_{O_{jk}}^j)^\top}
{\Vert\pmb{u}_{O_{jk}}^j\Vert_2^2}
\right)\pmb{u}_{O_{jk}}^k.
\]
The matrix in parentheses is the orthogonal projector onto the complement of the span of \(\pmb{u}_{O_{jk}}^j\). Therefore, after this update,
\[
\langle \pmb{u}_{O_{jk}}^j,\pmb{u}_{O_{jk}}^k\rangle = 0,
\]
so the overlapping portion of the \(k\)th component is orthogonal to that of the \(j\)th component on the shared sample indices.

After this local projection adjustment, each loading vector is normalized to unit Euclidean norm:
\[
\pmb{u}^k
\leftarrow
\frac{\pmb{u}^k}{\Vert\pmb{u}^k\Vert_2},
\qquad k=1,\dots,K.
\]

For the right singular functions, orthogonality is handled differently. Let
\[
\pmb{\varphi}^k(t)
=
\bigl(\varphi_1^k(t),\dots,\varphi_p^k(t)\bigr)^\top,
\qquad k=1,\dots,K.
\]
In our simulation, orthogonality across components is enforced by construction: for each active variable, the component-specific loading functions are assigned from a prespecified collection of mutually orthogonal candidate functions. Therefore, for \(j\neq k\),
\[
\langle \pmb{\varphi}^j,\pmb{\varphi}^k\rangle_{\mathbb H}=0
\]
by design. Hence, unlike the sample loading vectors, the right singular functions do not require an additional overlap-adjustment step; their orthogonality is imposed directly through the function assignment in the data-generating mechanism.

\clearpage
\subsection{Evaluation criteria}
To quantify performance, we adopt the similarity framework from ~\cite{Sill2011}. Let $\widehat{\mathcal{M}} = \{\widehat{M}_1, \dots, \widehat{M}_k\}$ and $\mathcal{M} = \{M_1, \dots, M_{k^{'}}\}$ denote the collections of index sets induced by the estimated and true clusters, respectively, where each set records the indices belonging to one cluster.
For an estimated cluster $a$ and a true cluster $b$, the Jaccard index is defined as
\[
\mathrm{Jac}(\widehat{M}_a, M_b) = \frac{|\widehat{M}_a \cap M_b|}{|\widehat{M}_a \cup M_b|},
\quad a=1,\ldots,k,\ \ b=1,\ldots,k^{'}.
\]

Relevance and recovery are the averages of the best Jaccard matches across predicted and true clusters, respectively:
\[
\mathrm{Rel} = \frac{1}{k} \sum_{a=1}^k \max_{b=1}^{k^{'}} \mathrm{Jac}(\widehat{M}_a, M_b),
\quad
\mathrm{Rec} = \frac{1}{k^{'}} \sum_{b=1}^{k^{'}} \max_{a=1}^k \mathrm{Jac}(\widehat{M}_a, M_b).
\]
The overall score is their harmonic mean: $
F = \dfrac{2 \cdot \mathrm{Rel} \cdot \mathrm{Rec}}{\mathrm{Rel} + \mathrm{Rec}}.$
In our simulation studies, we use the $F$ score to quantify clustering performance at multiple clustering levels, including subject clustering, variable clustering, subregion clustering, biclustering, and triclustering. We use $\widehat u_i^{a}$ and $u_i^{a}$ to denote the estimated and true values of the $i$th element of the singular score vector $\pmb{u}$ in sparse layer $a$, respectively. Likewise, we use $\widehat{\psi}_j^{b}$ and $\psi_j^{b}$ to denote the estimated and true singular functions corresponding to variable $j$ in sparse layer $b$, respectively. The index sets are defined at each level as follows:
\[
\renewcommand{\arraystretch}{1.1}
\begin{array}{ll}
\text{Sample level:} &
  \widehat{M}^{sam}_a = \{i : \widehat u_{i}^{a} \neq 0\}, \;
  M^{sam}_b = \{i : u_{i}^{b} \neq 0\};
\\[0.2em]
\text{Variable level:} &
  \widehat{M}^{var}_a = \{j : \Vert \boldsymbol{\widehat \psi}_{j}^{a} \Vert_{H_j} \neq 0\}, \;
  M^{var}_b = \{j : \Vert \boldsymbol{\psi}_{j}^{b} \Vert_{H_j} \neq 0\};
\\[0.2em]
\text{Subregion level:} &
  \widehat{M}^{sub}_a = \{(j, t) : \boldsymbol{\widehat \psi}_{j}^{a}(t) \neq 0\}, \;
  M^{sub}_b = \{(j, t) : \boldsymbol{\psi}_{j}^{b}(t) \neq 0\};
\\[0.2em]
\text{Bicluster level:} &
  \widehat{M}^{bi}_a = \widehat{M}^{sam}_a \times \widehat{M}^{var}_a, \;
  M^{bi}_b = M^{sam}_b \times M^{var}_b;
\\[0.2em]
\text{Tricluster level:} &
  \widehat{M}^{tri}_a = \widehat{M}^{sam}_a \times \widehat{M}^{sub}_a, \;
  M^{tri}_b = M^{sam}_{b} \times M^{sub}_b.
\end{array}
\]

Using the level-specific index sets in the above definition yields five $F$-scores, corresponding to subject, variable, subregion, bicluster, and tricluster accuracy.

\clearpage
\subsection{Performance Evaluation of Cluster Number Selection}
\label{Performance Evaluation}
Furthermore, we report the performance of using the BIC score and variance explained to select the number of clusters $K$. As shown in Figure~\ref{bic:overall}, we present the distribution of BIC scores for estimating different values of $K$ across 20 simulation repetitions for three biclustering scenarios and three triclustering scenarios. Similarly, Figure~\ref{ve:overall} illustrates the distribution of variance explained under the same scenarios. Given that the true number of biclusters/triclusters is 4 in all simulations, both criteria demonstrate objective and reliable performance in selecting the appropriate number of biclusters/triclusters.
\begin{figure}[htbp]
    \centering
    \begin{subfigure}[b]{0.25\textwidth}
        \includegraphics[width=\textwidth]{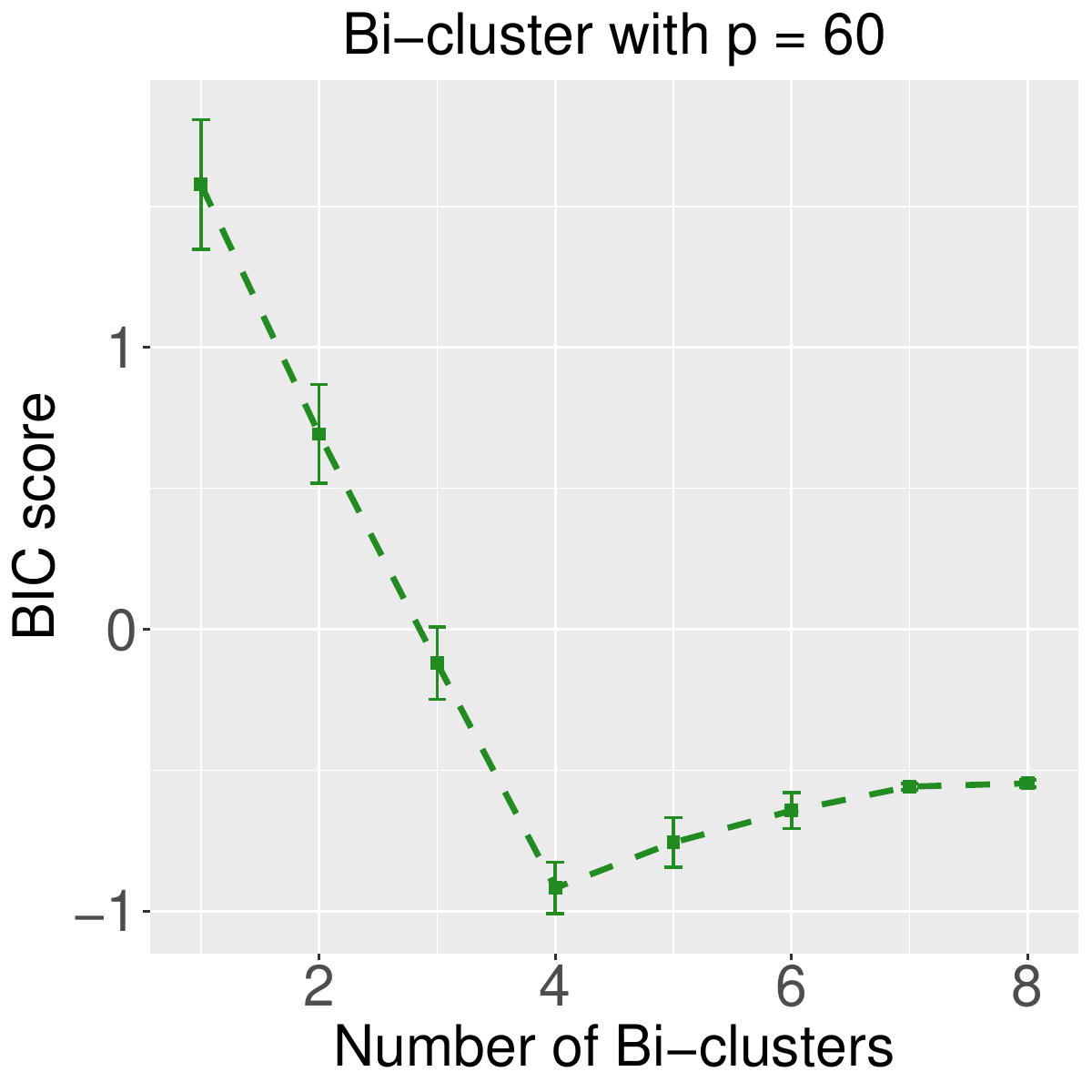}
        \label{fig:bic_panel_1}
    \end{subfigure}
    \hfill
    \begin{subfigure}[b]{0.25\textwidth}
        \includegraphics[width=\textwidth]{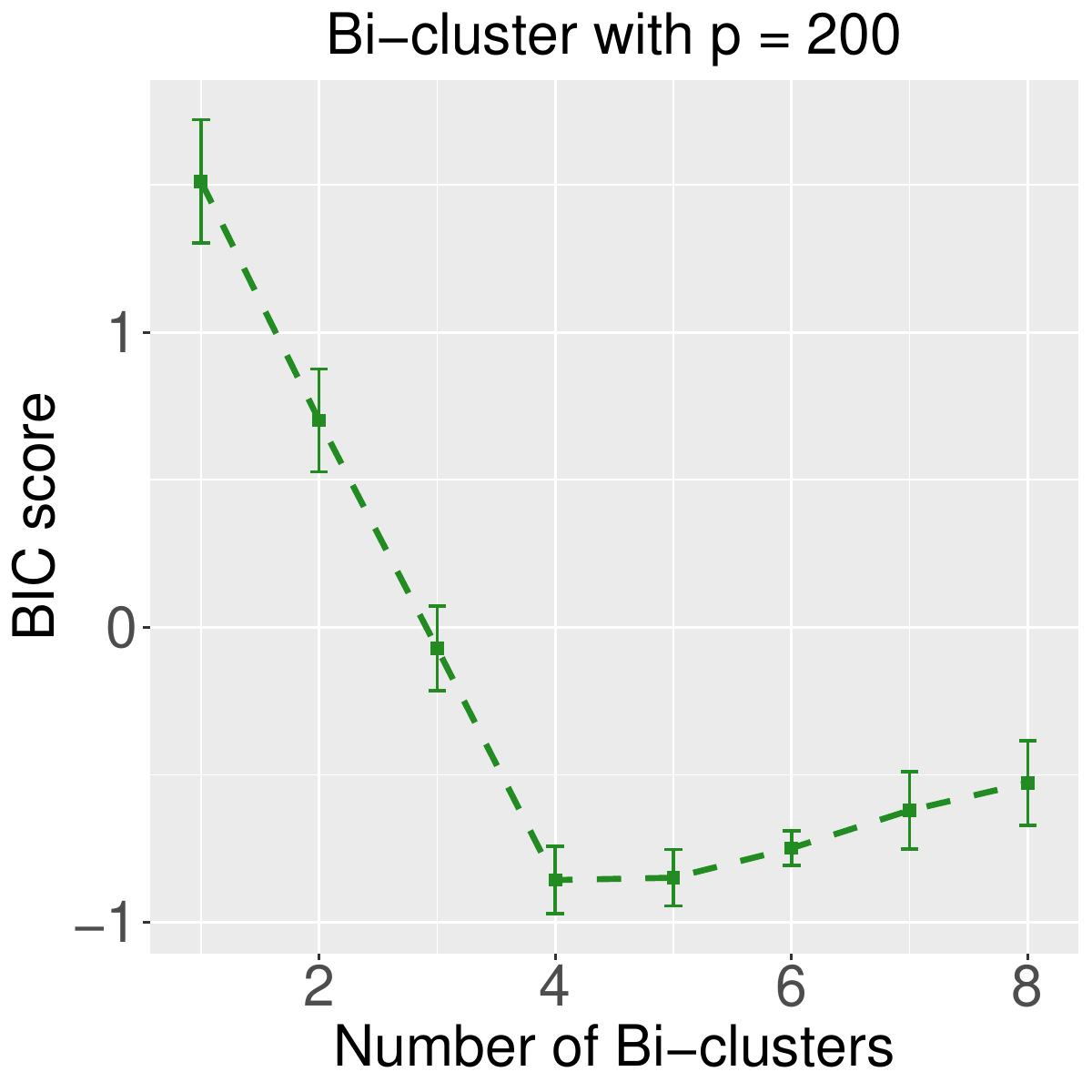}
        \label{fig:bic_panel_2}
    \end{subfigure}
    \hfill
    \begin{subfigure}[b]{0.25\textwidth}
        \includegraphics[width=\textwidth]{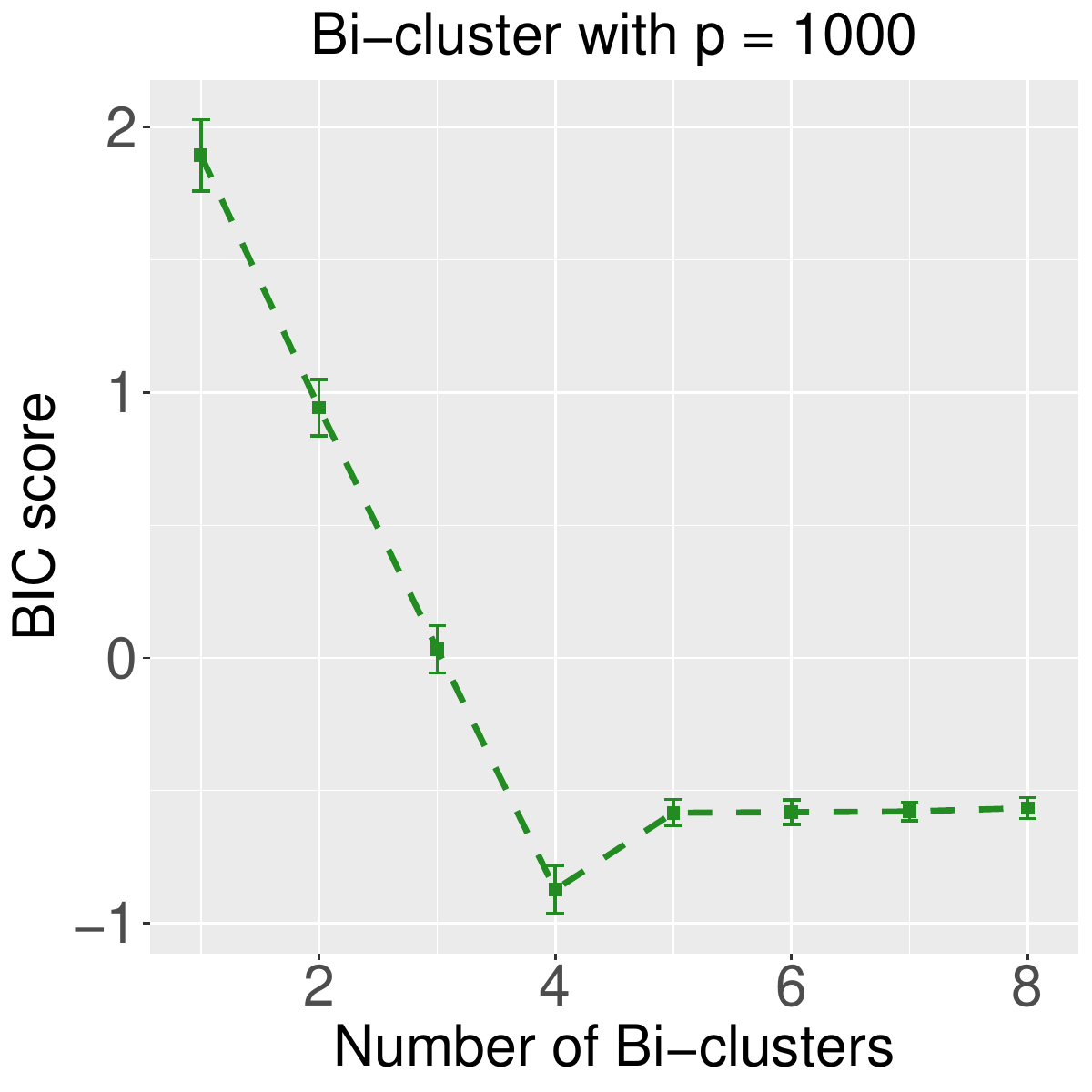}
        \label{fig:bic_panel_3}
    \end{subfigure}

    \vspace{10pt} 

    \begin{subfigure}[b]{0.25\textwidth}
        \includegraphics[width=\textwidth]{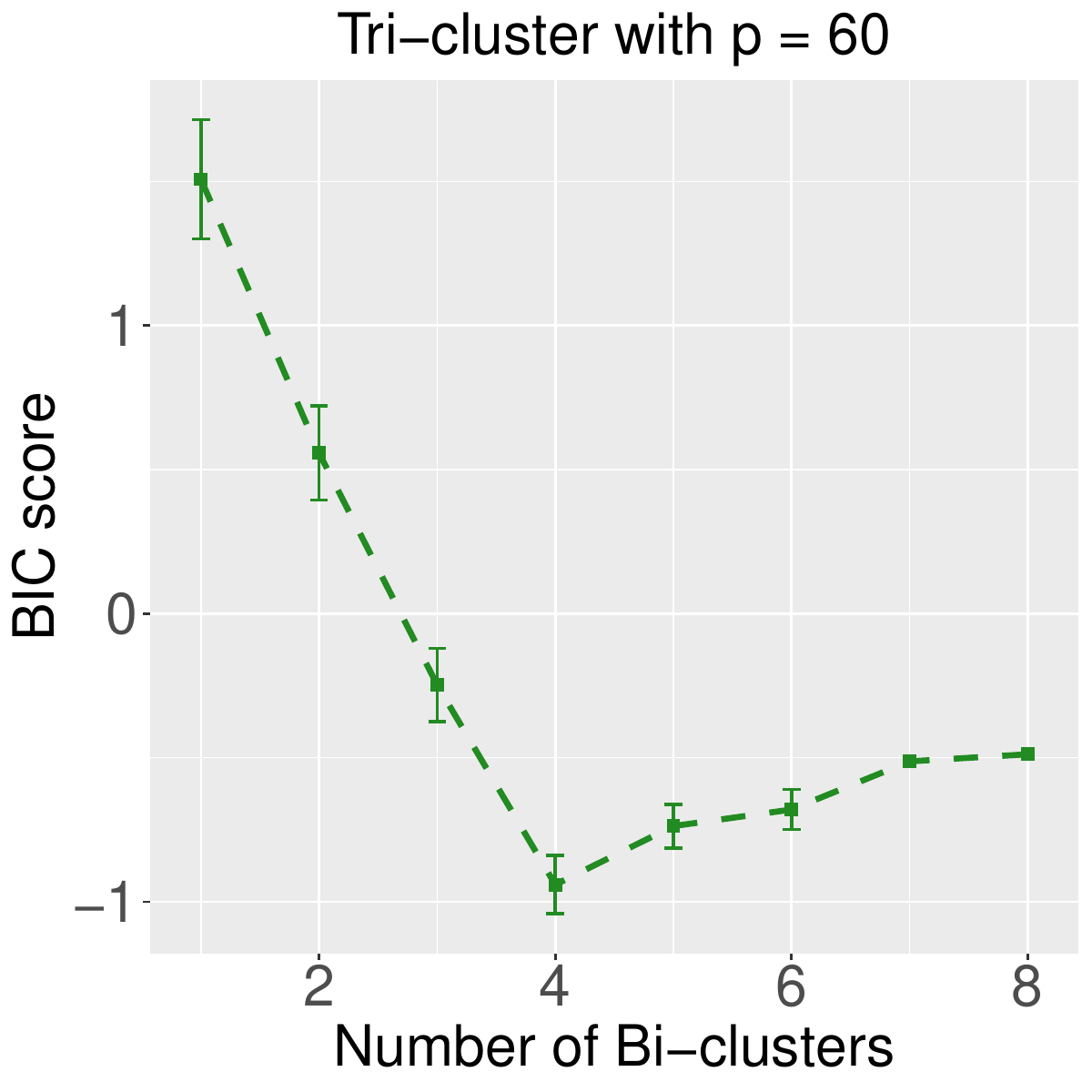}
        \label{fig:bic_panel_4}
    \end{subfigure}
    \hfill
    \begin{subfigure}[b]{0.25\textwidth}
        \includegraphics[width=\textwidth]{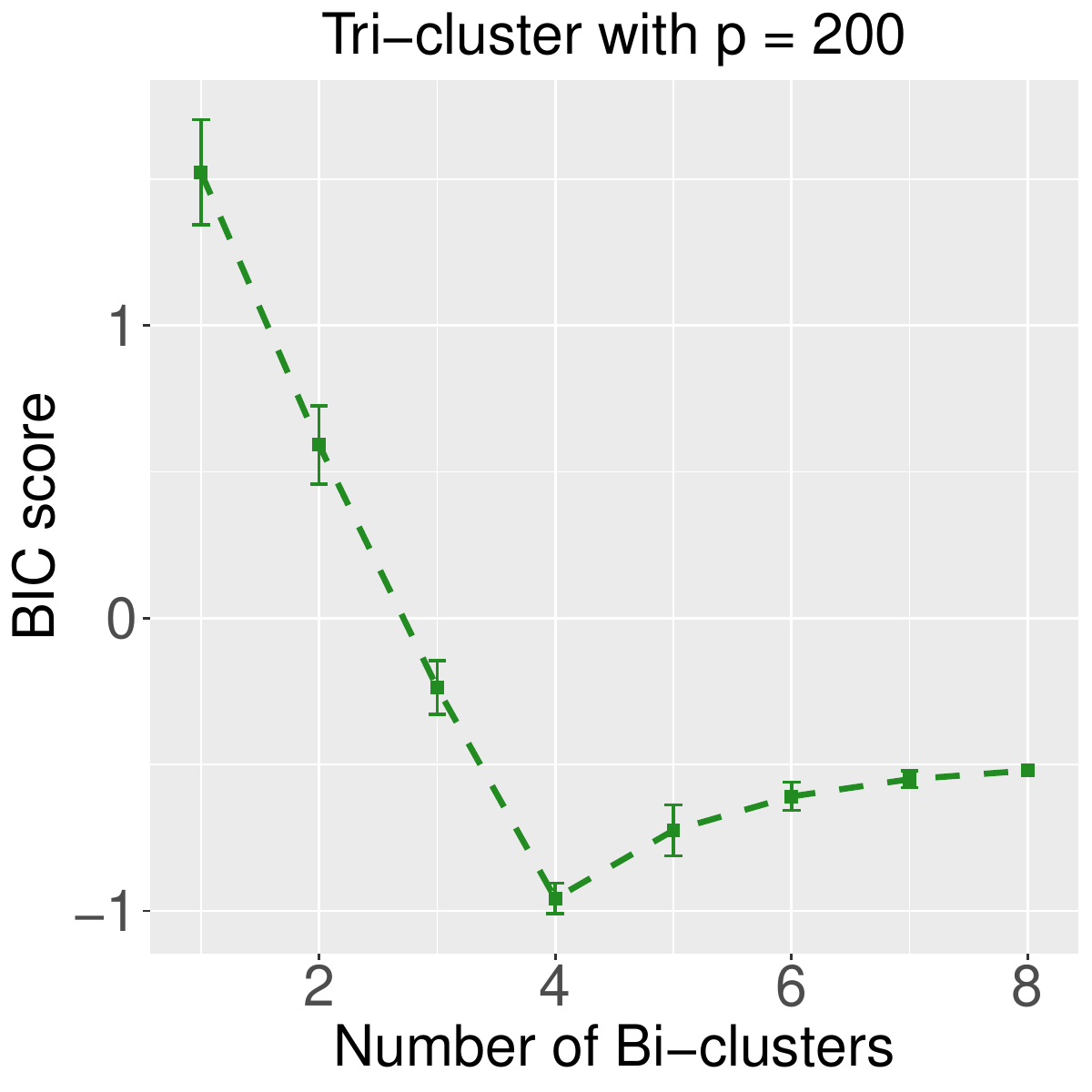}
        \label{fig:bic_panel_5}
    \end{subfigure}
    \hfill
    \begin{subfigure}[b]{0.25\textwidth}
        \includegraphics[width=\textwidth]{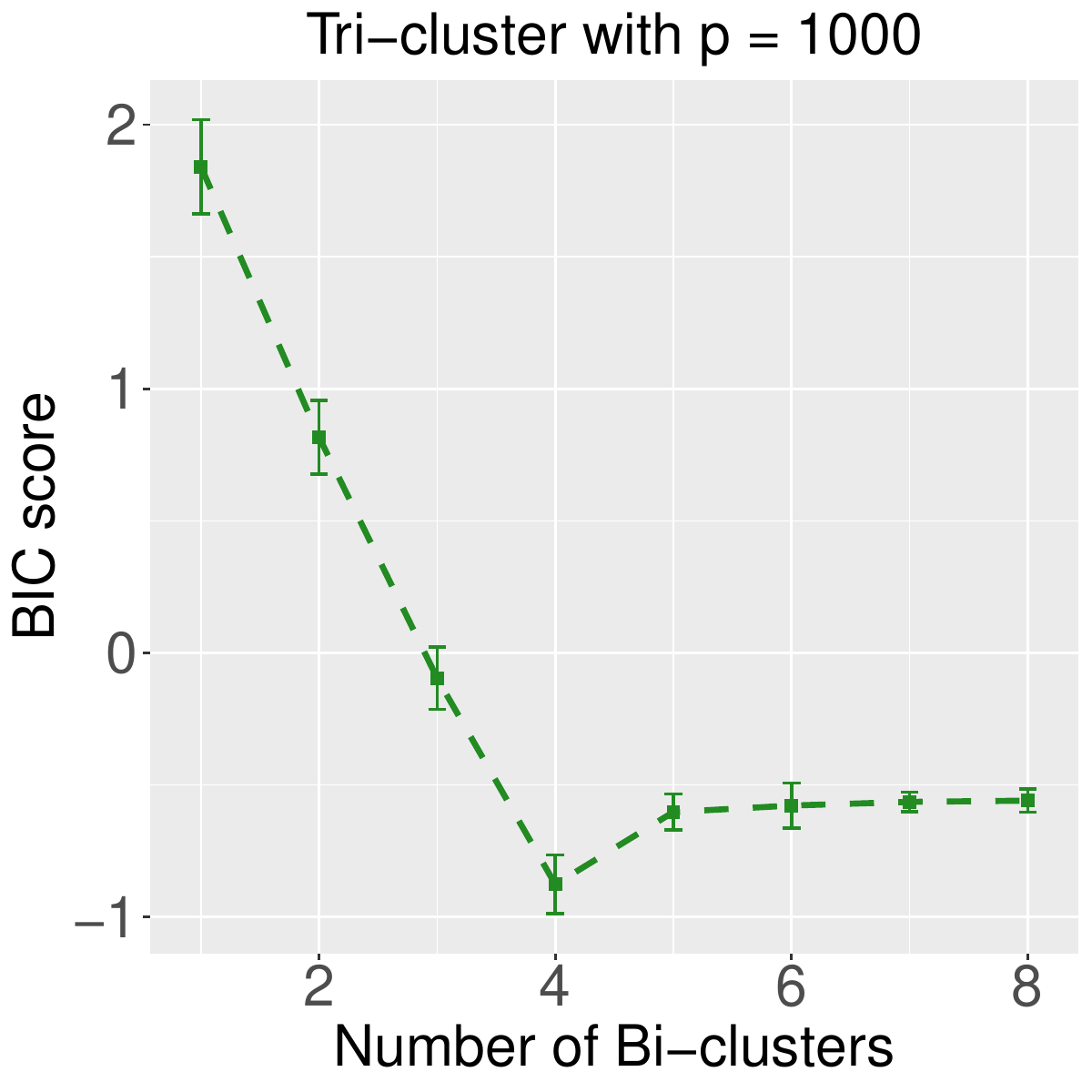}
        \label{fig:bic_panel_6}
    \end{subfigure}

    \caption{Distribution of BIC scores for selecting the number of biclusters/triclusters $K$ across 20 simulation repetitions. Each panel corresponds to one simulation scenario (top row: three biclustering scenarios; bottom row: three triclustering scenarios) and displays the BIC values obtained for different candidate $K$. The true number of biclusters/triclusters is $K=4$ in all scenarios; lower BIC indicates a better fit and is used to select $K$.}

    \label{bic:overall}
\end{figure}

\begin{figure}[H]
    \centering
    \begin{subfigure}[b]{0.25\textwidth}
        \includegraphics[width=\textwidth]{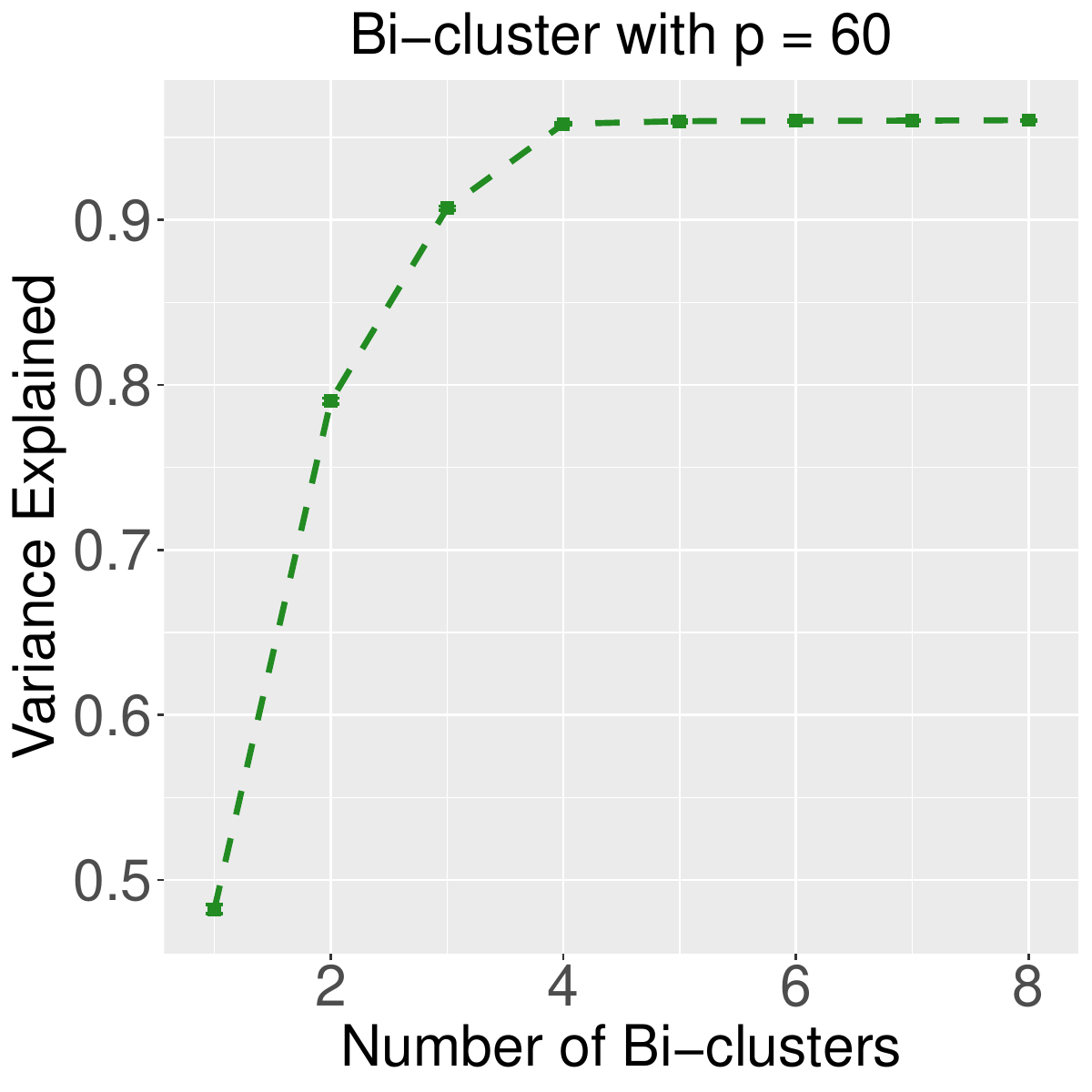}
        \label{fig:ve_panel_1}
    \end{subfigure}
    \hfill
    \begin{subfigure}[b]{0.25\textwidth}
        \includegraphics[width=\textwidth]{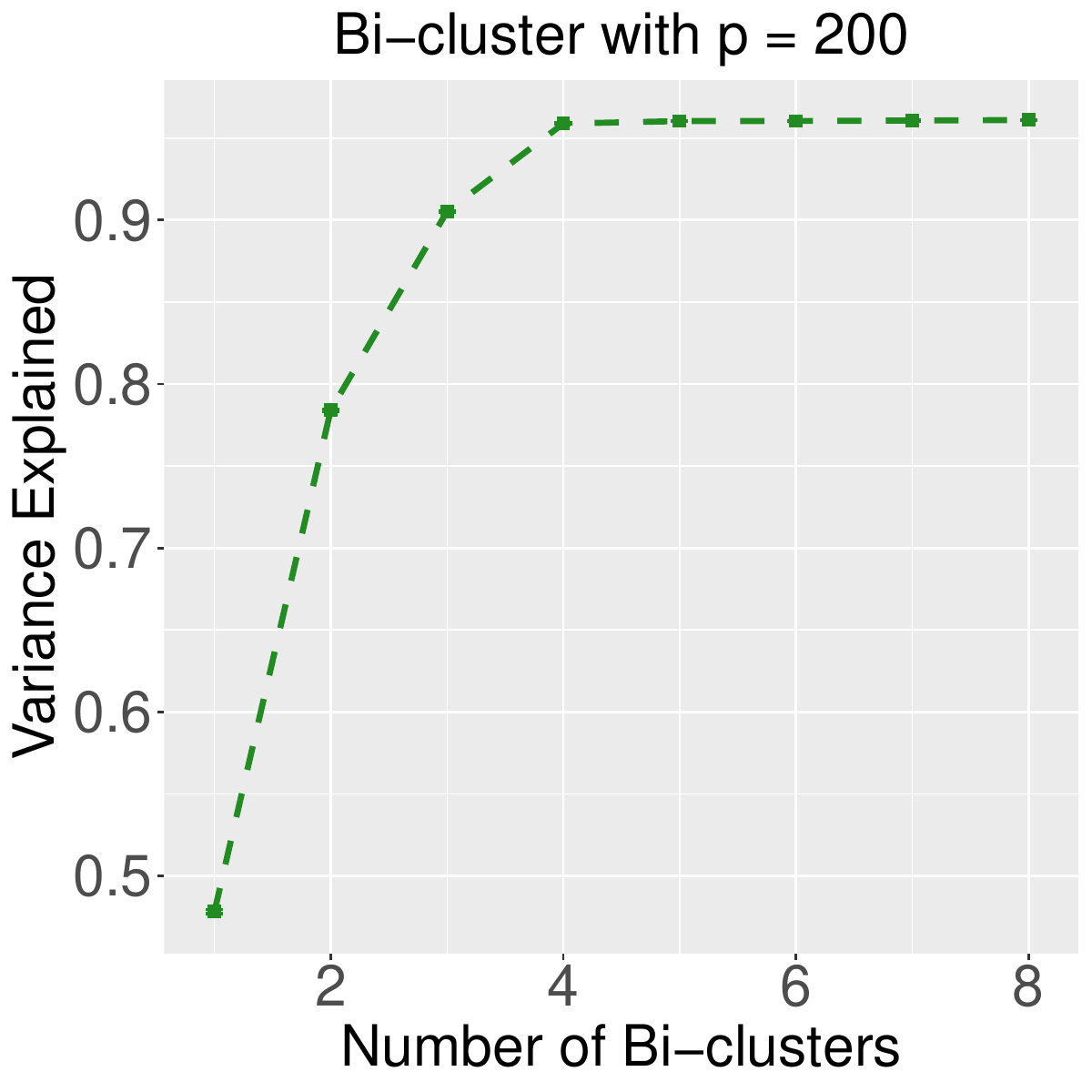}
        \label{fig:ve_panel_2}
    \end{subfigure}
    \hfill
    \begin{subfigure}[b]{0.25\textwidth}
        \includegraphics[width=\textwidth]{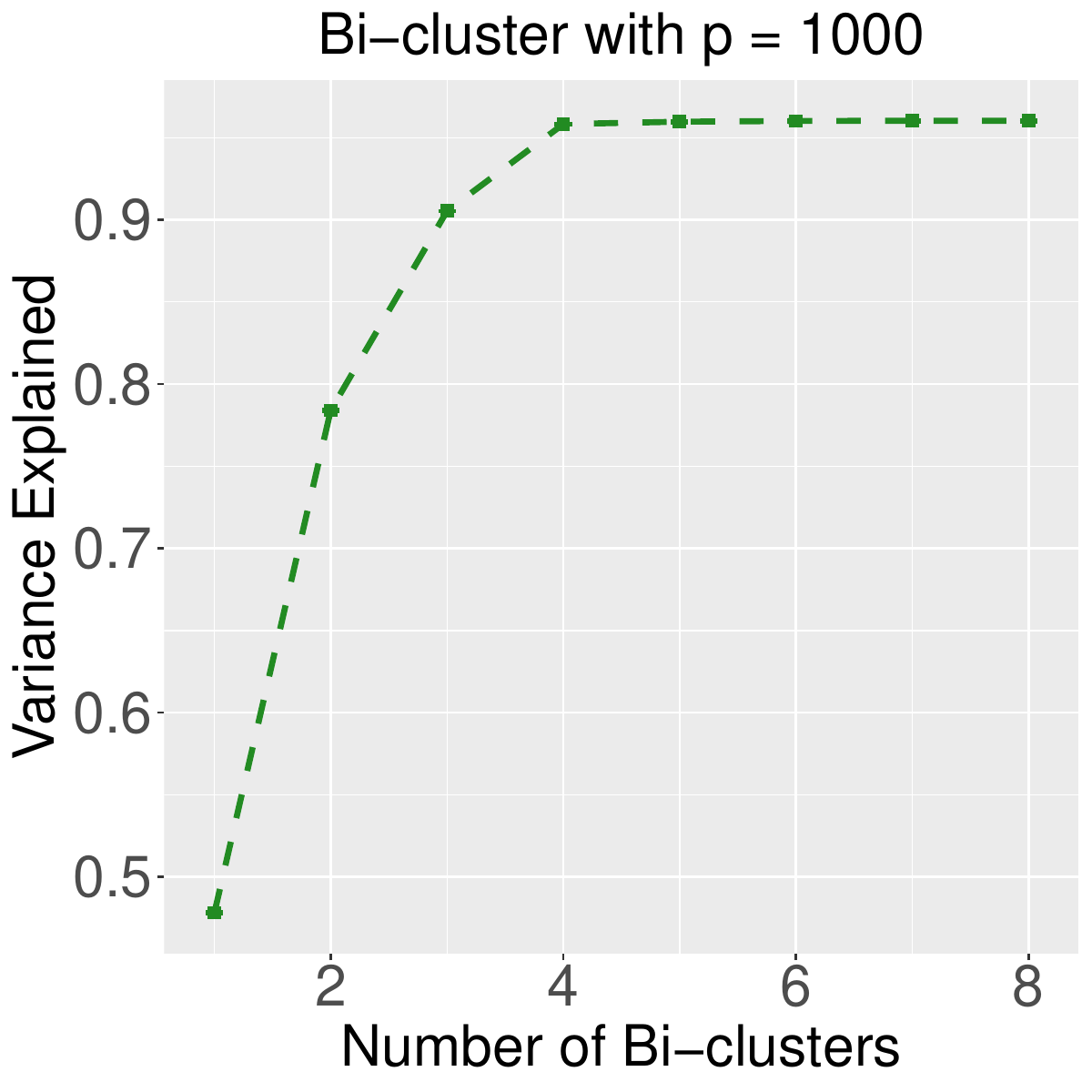}
        \label{fig:ve_panel_3}
    \end{subfigure}

    \vspace{10pt} 

    \begin{subfigure}[b]{0.25\textwidth}
        \includegraphics[width=\textwidth]{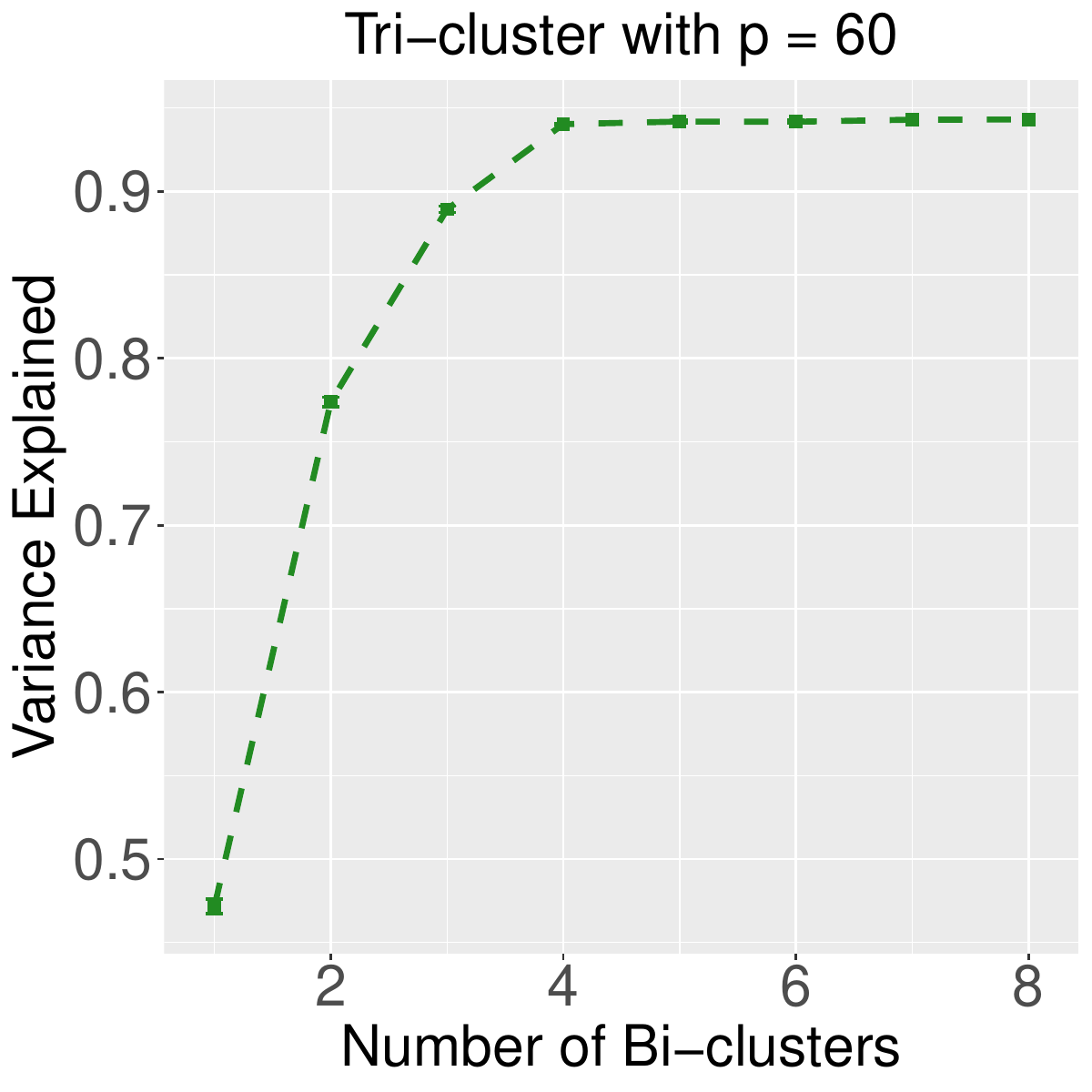}
        \label{fig:ve_panel_4}
    \end{subfigure}
    \hfill
    \begin{subfigure}[b]{0.25\textwidth}
        \includegraphics[width=\textwidth]{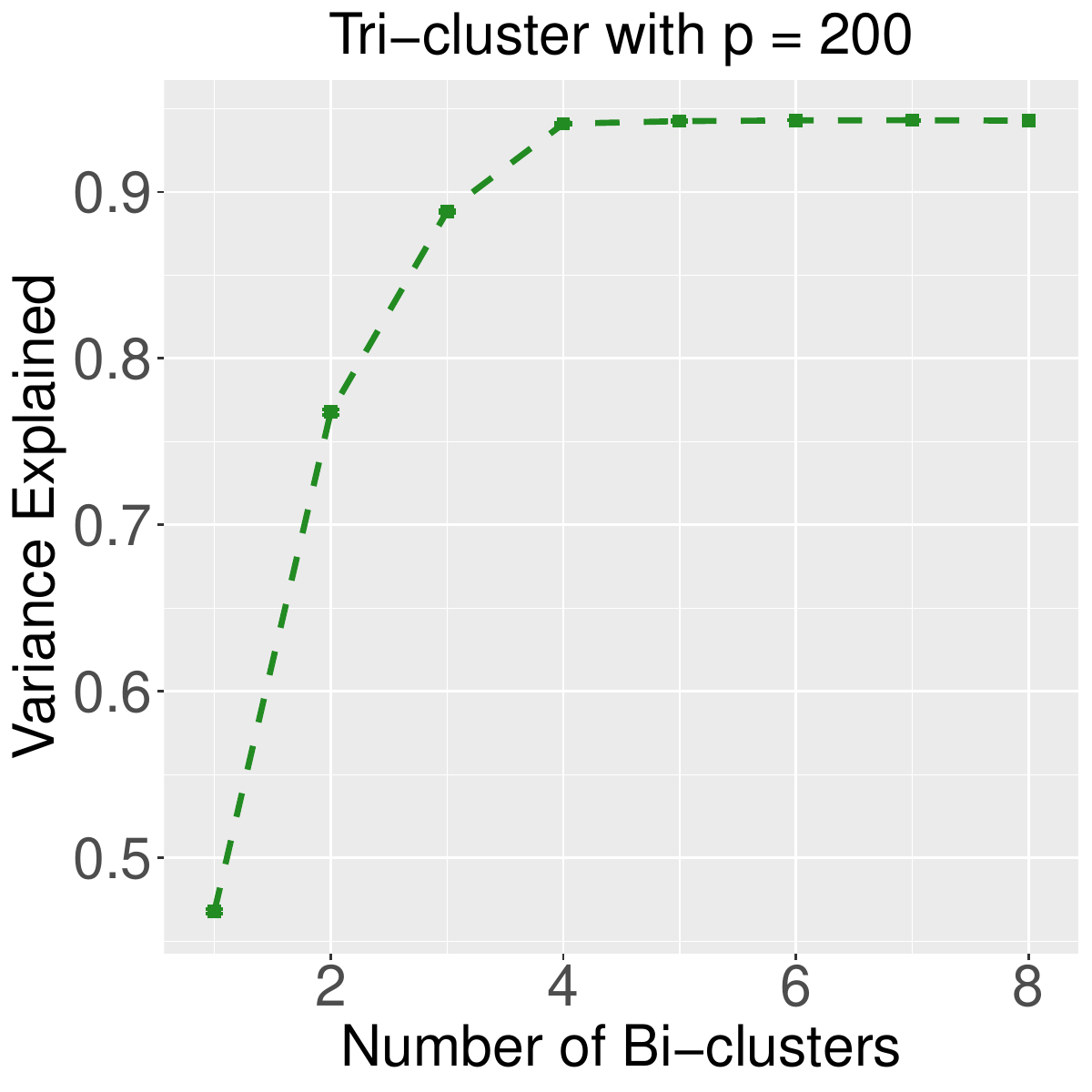}
        \label{fig:ve_panel_5}
    \end{subfigure}
    \hfill
    \begin{subfigure}[b]{0.25\textwidth}
        \includegraphics[width=\textwidth]{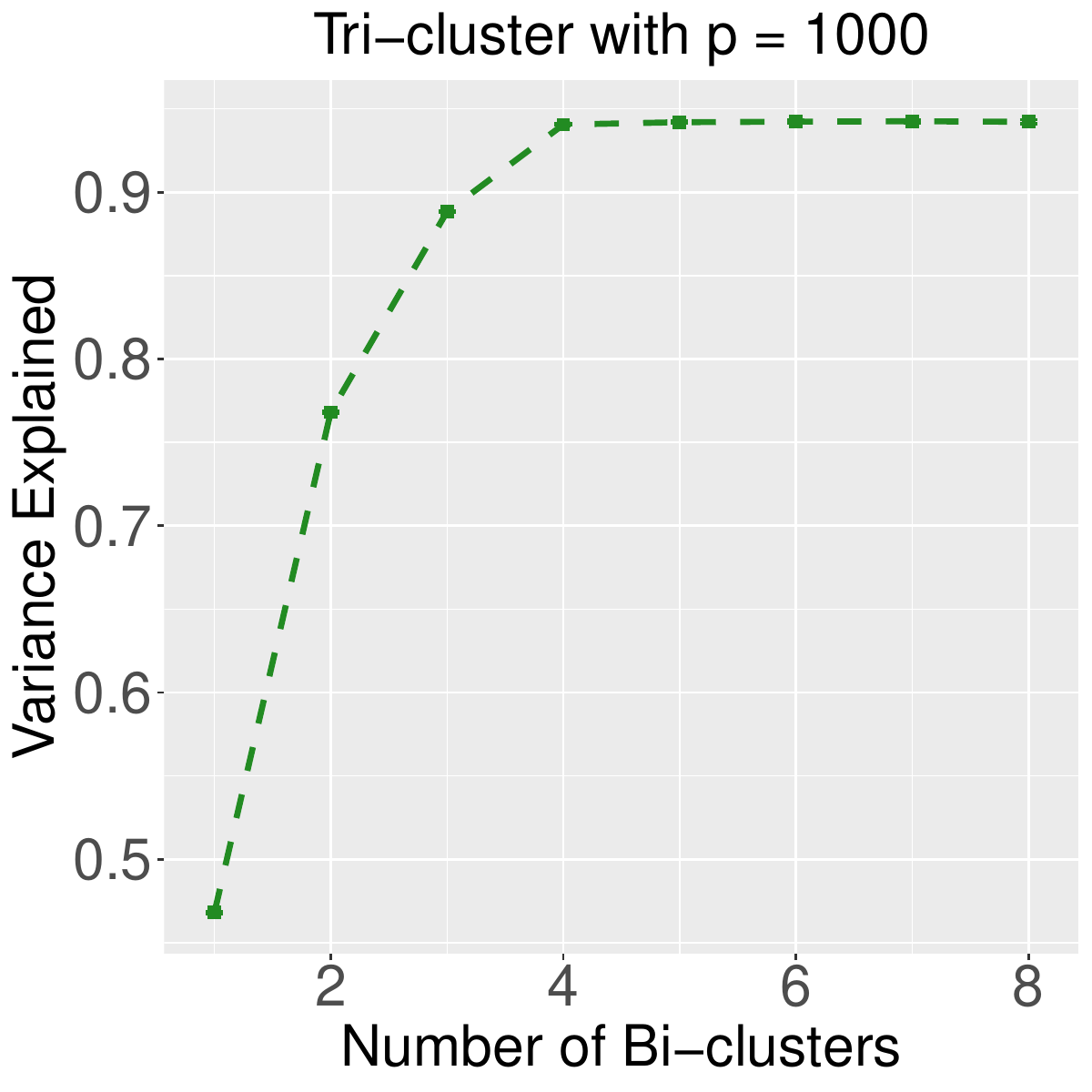}
        \label{fig:ve_panel_6}
    \end{subfigure}

    \caption{Distribution of variance explained for selecting the number of biclusters/triclusters $K$ across 20 simulation repetitions. Each panel corresponds to one simulation scenario (top row: three biclustering scenarios; bottom row: three triclustering scenarios) and displays the variance explained obtained for different candidate $K$. The true number of biclusters/triclusters is $K=4$ in all scenarios; larger variance explained indicates a better fit and is used to select $K$.}

    \label{ve:overall}
\end{figure}

\clearpage
\begin{figure}[p]
    \centering
    \includegraphics[
        width=\textwidth,
        height=0.72\textheight,
        keepaspectratio
    ]{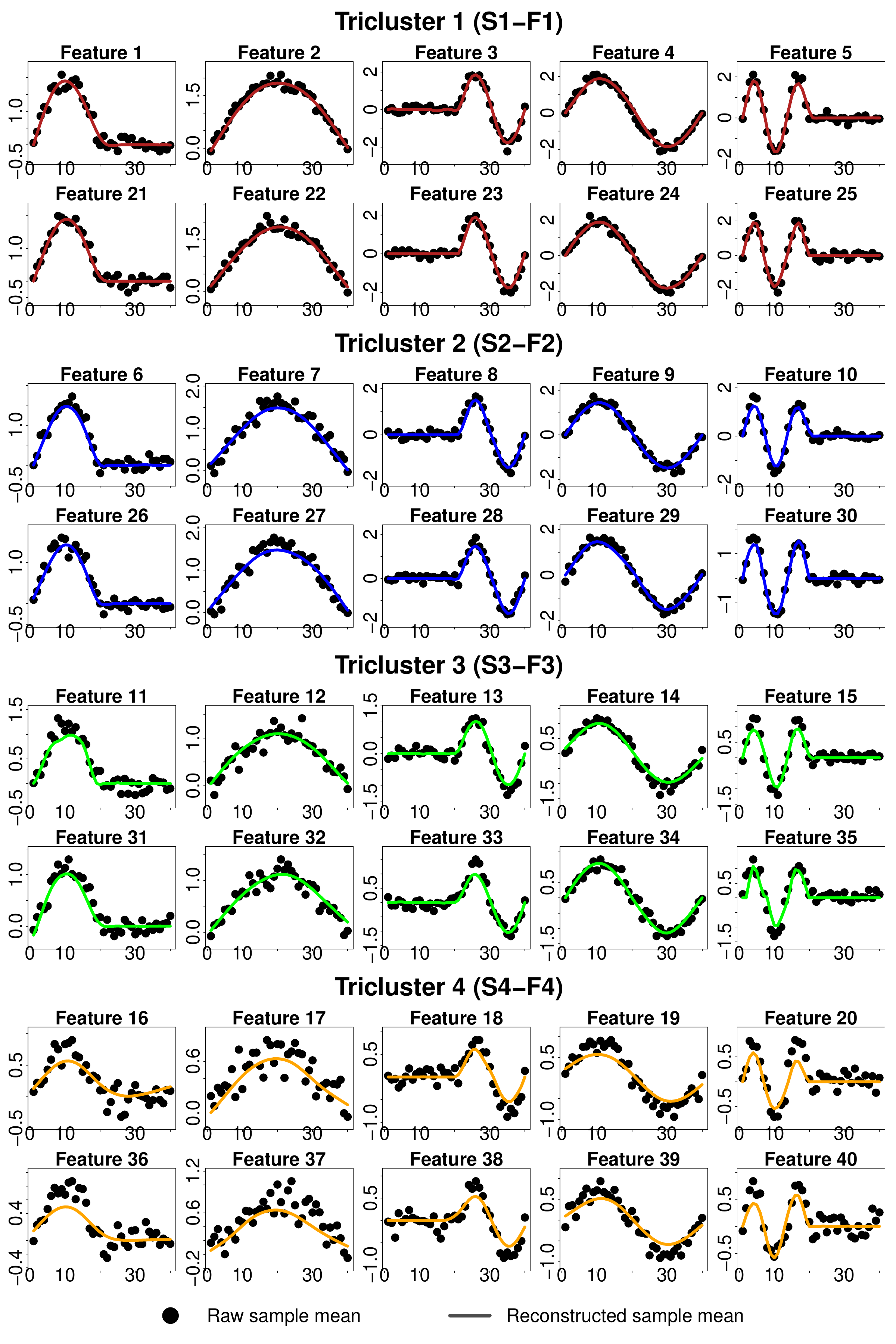}
    \caption{
    Representative sample--feature--time tricluster results from a randomly selected nonoverlapping simulated dataset with \(p=60\) and a missing rate of \(0.4\). The four blocks correspond to the four true triclusters \((\pmb{S}_1,\pmb{F}_1)\), \((\pmb{S}_2,\pmb{F}_2)\), \((\pmb{S}_3,\pmb{F}_3)\), and \((\pmb{S}_4,\pmb{F}_4)\), respectively, where \(\pmb{S}_k\) denotes the sample cluster and \(\pmb{F}_k\) denotes the feature cluster in the \(k\)th tricluster. For each feature in \(\pmb{F}_k\), the identified time-domain component corresponds to the nonzero support of its estimated loading function. Within each block, panels display all ten active features in the corresponding feature group. Black points represent raw sample mean measurements among samples in the corresponding sample group, computed after omitting missing observations at each time point. Colored solid curves represent reconstructed sample mean trajectories obtained from the corresponding sparse functional SVD layer. The displayed features include both full-domain signals and subregion-localized signals, showing the heterogeneous trajectory shapes and time-domain support patterns present within the simulated tricluster structure.
    }
    \label{fig:sim_mean_curve_example}
\end{figure}

\clearpage

\bibliographystyle{apalike}
\bibliography{Mybib}

\end{document}